\documentclass{article} 
\usepackage{iclr2024_conference,times}

\usepackage{amsmath,amsfonts,bm}

\def\eqref#1{equation~\ref{#1}}

\def\1{\bm{1}}

\DeclareMathAlphabet{\mathsfit}{\encodingdefault}{\sfdefault}{m}{sl}
\SetMathAlphabet{\mathsfit}{bold}{\encodingdefault}{\sfdefault}{bx}{n}

\newcommand{\eg}{{\it e.g.}}
\newcommand{\ie}{{\it i.e.}}

\newcommand{\reals}{\mathbb{R}}

\newcommand{\expect}{\mathbb{E}}

\newcommand{\var}{\mathrm{Var}}

\newcommand{\fedavg}{{\sc FedAvg}\xspace}

\newcommand{\scaffold}{{\sc SCAFFOLD}\xspace}

\newcommand{\fedpd}{{\sc FedPD}\xspace}
\newcommand{\fedprox}{{\sc FedProx}\xspace}
\newcommand{\fednova}{{\sc FedNova}\xspace}

\usepackage[colorlinks,linkcolor=red,filecolor=blue,citecolor=blue,urlcolor=blue]{hyperref}
\usepackage{url}
\usepackage{amsthm,amsmath,amssymb, amscd}
\usepackage{subfigure}
\usepackage{graphicx}
\usepackage{enumerate}
\usepackage{enumitem}
\usepackage{algorithm, algorithmic}
\usepackage{dsfont}
\usepackage{booktabs}
\usepackage{multirow}
\usepackage{threeparttable}
\usepackage{makecell}
\usepackage{colortbl}
\definecolor{Ocean}{RGB}{129,194,234}
\definecolor{BPink}{rgb}{0.96, 0.76, 0.76}
\usepackage{soul, color, xcolor}

\usepackage{float}
\restylefloat{algorithm} 
\usepackage{xspace}

\usepackage{mathtools}
\mathtoolsset{showonlyrefs=true}

\newtheorem{thm}{Theorem}

\newtheorem{asp}{Assumption}
\newtheorem{lemma}[thm]{Lemma}

\newcommand{\q}{\frac{N-S}{S(N-1)}}

\newcommand{\one}{\mathds{1}}

\newcommand{\EE}{\mathbb{E}}
\newcommand{\RR}{\mathbb{R}}

\newcommand{\cE}{\mathcal{E}}
\newcommand{\cS}{\mathcal{S}}

\newcommand{\fedcm}{{\sc FedCM}\xspace}
\newcommand{\fedavgm}{{\sc FedAvg-M}\xspace}
\newcommand{\fedavgmvr}{{\sc FedAvg-M-VR}\xspace}
\newcommand{\scaffoldm}{{\sc SCAFFOLD-M}\xspace}
\newcommand{\scaffoldmvr}{{\sc SCAFFOLD-M-VR}\xspace}

\usepackage{float}
\restylefloat{algorithm} 
\title{Momentum Benefits Non-IID Federated Learning 
Simply and Provably}

\author{Ziheng Cheng\thanks{Equal Contribution. This work is supported in part by National Natural Science Foundation of China (NSFC) Grant 12301392, 92370121, and 12288101} \\
Peking University\\
\texttt{alex-czh@stu.pku.edu.cn} 
\And
Xinmeng Huang$^*$ \\
University of Pennsylvania \\
\texttt{xinmengh@sas.upenn.edu}
\AND
Pengfei Wu \\
Peking University\\
\texttt{pengfeiwu1999@stu.pku.edu.cn}
\And 
\hspace{17mm} Kun Yuan\thanks{Corresponding Author. Kun Yuan is also affiliated with National Engineering Labratory for Big Data Analytics and Applications, and AI for Science Institute, Beijing, China.} \\
\hspace{18mm}Peking University\\
\hspace{18mm}\texttt{kunyuan@pku.edu.cn}
}

\iclrfinalcopy 
\begin{document}

\vspace{-20mm}
\maketitle
\vspace{-7mm}
\begin{abstract}
\vspace{-3mm}
Federated learning is a powerful paradigm for large-scale machine learning, but it faces significant challenges due to unreliable network connections, slow communication, and substantial data heterogeneity across clients. \fedavg and \scaffold are two prominent algorithms to address these challenges. In particular,  \fedavg employs multiple local updates before communicating with a 
central server, while \scaffold maintains a control variable on each client to compensate for ``client drift'' in its local updates. Various methods have been proposed to enhance the convergence of these two algorithms, but they either make impractical adjustments to the algorithmic structure or rely on the assumption of bounded data heterogeneity. This paper explores the utilization of momentum to enhance the performance of \fedavg and \scaffold. When all clients participate in the training process, we demonstrate that incorporating momentum allows \fedavg to converge without relying on the assumption of bounded data heterogeneity even using a constant local learning rate. This is novel and fairly surprising as existing analyses for \fedavg require
bounded data heterogeneity even with diminishing local learning rates. In partial client participation, we show that momentum enables \scaffold to converge provably faster without imposing any additional assumptions. Furthermore, we use momentum to develop new variance-reduced extensions of \fedavg and \scaffold, which exhibit state-of-the-art convergence rates. Our experimental results support all theoretical findings.
\end{abstract}

\vspace{-4mm}
\section{Introduction}
\vspace{-2mm}
Federated learning (FL) is a powerful paradigm for large-scale machine learning \citep{konevcny2016federated,mcmahan2017communication}. In situations where data and computational resources are dispersed among a diverse range of clients, including phones, tablets, sensors, hospitals, and other devices and agents, federated learning facilitates local data processing and collaboration among these clients \citep{kairouz2021advances}. Consequently, a centralized model can be trained without transmitting decentralized data from clients directly to servers, thereby ensuring a fundamental level of privacy. 

Federated learning encounters several significant challenges in algorithmic development. Firstly, the reliability and relatively slow nature of network connections between the server and clients pose obstacles to efficient communication during the training process. Secondly, the dynamic availability of only a small subset of clients for training at any given time demands strategies that can adapt to this variable environment. Lastly, the presence of substantial heterogeneity of non-iid data across different clients further complicates the training process.

\fedavg \citep{konevcny2016federated,mcmahan2017communication,stich2019local,yu2019linear,lin2020don,wang2021cooperative} emerges as a prevalent algorithm for FL, leveraging multiple stochastic gradient descent (SGD) steps within each client before communicating with a central server. While \fedavg is readily implementable and succeeds in certain applications, its performance is notably hindered by the presence of data heterogeneity, \ie, non-iid clients, even when {\em all clients} participate in the training process \citep{li2019convergence,yang2021achieving}. To mitigate the influence of data heterogeneity, \scaffold\ \citep{karimireddy2020scaffold} maintains a control variable on each client to compensate for ``client drift'' in its local SGD updates, making convergence more robust to data heterogeneity and client sampling. Due to their practicality and effectiveness, \fedavg\ and \scaffold\ have become foundational algorithms in federated learning, leading to the development of numerous variants that cater to decentralized \citep{koloskova2020unified,rizk2022privatized,nguyen2022performance,alghunaim2023local}, compressed \citep{haddadpour2021federated,reisizadeh2020fedpaq,mitra2021linear}, asynchronous \citep{chen2020vafl,chen2020asynchronous,xu2021asynchronous}, and personalized \citep{fallah2020personalized,pillutla2022federated,tan2022towards,t2020personalized} federated learning scenarios.

Various methods have been proposed to enhance the convergence of \fedavg, \scaffold, and their variance-reduced\footnote{Throughout the paper, variance reduction refers to techniques aiming to mitigate the influence of within-client gradient stochasticity, as opposed to the inter-client data heterogeneity.} extensions.
While exhibiting superior convergence rates, these approaches typically make impractical adjustments to algorithmic structures. For instance, {\sc STEM} \citep{khanduri2021stem} requires increasing either the batch size or the number of local steps with algorithmic iterations. Similarly, {\sc CE-LSGD} \citep{patel2022towards} and {\sc MIME} \citep{karimireddy2020mime} mandate computing a large-batch or even full-batch local gradient per round for each client. Additionally, \fedprox\ \citep{li2020federated}, \fedpd\ \citep{zhang2021fedpd}, and {\sc FedDyn} \citep{durmus2021federated} rely on solving ``local problems'' to an extremely high precision. These adjustments may not align with the practical constraints in federated learning setups. 

Furthermore, many of these algorithms, including \fedavg, {\sc STEM}, \fedprox, {\sc MIME}, and {\sc CE-SGD}, still rely on the assumption of bounded data heterogeneity. When this assumption is violated, their theoretical analyses become invalid. While some algorithms, such as {\sc LED} \citep{alghunaim2023local} and {\sc VRL-SGD} \citep{liang2019variance}, can handle unbounded data heterogeneity, their convergence rates are not state-of-the-art, as stated in Table \ref{tab:full_results}. These limitations motivate us to develop novel strategies that are easy to implement, robust to data heterogeneity, and exhibit superior convergence.

\vspace{-2mm}
\subsection{Main results and contributions}
\vspace{-1mm}
This paper examines the utilization of {\em momentum} to enhance the performance of \fedavg\ and \scaffold.
To ensure simplicity and practicality in implementations, we only introduce momentum to the local SGD steps, avoiding any inclusion of impractical elements, such as gradient computation of multiple minibatches or solving local problems to high precision. Remarkably, this straightforward approach effectively alleviates the necessity for stringent assumptions regarding bounded data heterogeneity, leading to noteworthy improvements in convergence rates. The main findings and contributions of this paper are summarized below.

First, when {\em all} clients participate in the training process: 
\begin{itemize}[leftmargin=0.3in]
\vspace{-2mm}
\item We show that incorporating momentum allows \fedavg\ and its variance-reduced extension to {\em converge under unbounded data heterogeneity}, even using constant local learning rates. 
This is rather surprising as, to our knowledge, all existing analyses for \fedavg, \eg, \cite{karimireddy2020scaffold,yang2021achieving,wang2020slowmo} require bounded data heterogeneity even with diminishing local learning rates.

\item We further establish that, by effectively removing the influence of data heterogeneity on convergence,  momentum empowers \fedavg and its variance-reduced extension with state-of-the-art convergence rates in the context of full client participation. 
\vspace{-2mm}
\end{itemize}

Second, when {\em partial} clients  participate in the training process per iteration: 

\begin{itemize}[leftmargin=0.3in]
\vspace{-2mm}
\item The proposed \scaffoldm that incorporates momentum into \scaffold achieves provably faster convergence. To our knowledge, this is the \emph{first} result that improves \scaffold without imposing additional assumptions beyond those in \cite{karimireddy2020scaffold}. 

\item We further introduce momentum to \scaffold with variance reduction, obtaining the \emph{first} variance-reduced FL algorithm that converges without bounded data heterogeneity. This method attains the state-of-the-art convergence rate in the context of partial client participation and unbounded data heterogeneity. 
\vspace{-2mm}
\end{itemize}

\begin{table}[t]
    \vspace{-15mm}
    \caption{\small 
The comparison of convergence rates of FL algorithms when {\bf all clients} participate in training. Notation $L$ is the smoothness constant of objective functions, $\Delta=f(x^0)-\min_x f(x)$ is the initialization gap, $\sigma^2$ is the variance of gradient noises, $N$ is the number of clients, $K$ is the number of  local steps per round, and $R$ is the number of communication rounds,  $\zeta^2$ and $G$ are uniform bounds of data heterogeneity $(1/N)\sum_{1\leq i\leq N}\|\nabla f_i(x)-\nabla f(x)\|^2$ and gradient norm $G:=\sup_x\max_{1\leq i\leq N}\|\nabla f_i(x)\|$ with $G^2 \gg \zeta^2$ typically. The ``Assumptions'' column lists all assumptions beyond  Assumption \ref{asp:smooth} and \ref{asp:sgd_var}.
}
    \label{tab:full_results}
    \centering{
    \scalebox{0.95}{
    \begin{threeparttable}
    \begin{tabular}{lll}
    \toprule
    {\bf Algorithm}  & {\bf Convergence Rate  $\expect[ \|\nabla f(\hat{x})\|^2]\lesssim$} & {\bf Assumptions} \vspace{1mm} \\ 
    \midrule \vspace{-3mm}\\
    {\sc FedAvg}\\ 
    \quad \citep{yu2019parallel}  & $\left(\frac{L\Delta\sigma^2}{NKR}\right)^{1/2} +\left(\frac{L\Delta G}{R}\right)^{2/3}+ \frac{L\Delta}{R}$ & Bounded grad.\\
        \quad \citep{koloskova2020unified} & $\left(\frac{L\Delta\sigma^2}{NR}\right)^{1/2} + \left(\frac{L\Delta K \zeta}{R}\right)^{2/3}+ \frac{L\Delta K}{R}$ & Bounded hetero. \\
      \quad \citep{karimireddy2020scaffold} & $\left(\frac{L\Delta\sigma^2}{NKR}\right)^{1/2} + \left(\frac{L\Delta\zeta}{R}\right)^{2/3}+ \frac{L\Delta}{R}$ & Bounded hetero.\\
    \quad \citep{yang2021achieving} & $\left(\frac{L\Delta\sigma^2}{NKR}\right)^{1/2} + \frac{L\Delta}{R}$ & Bounded hetero.\tnote{1} \vspace{1mm} \\


        {\makecell[l]{{\sc FedCM\tnote{2}} \\ \citep{xu2021fedcm}}} & ${\left(\frac{L\Delta ({\sigma^2}+NK{G}^2)}{NKR}\right)^{1/2}+\left(\frac{L\Delta ({\sigma}/{\sqrt{K}}+G)}{R}\right)^{2/3}}$ & {\makecell[l]{Bounded grad. \\ Bounded hetero.}} \\

    \makecell[l]{{\sc LED} \\ \citep{alghunaim2023local}} & $\left(\frac{L\Delta \sigma^2}{NKR}\right)^{1/2}+\left(\frac{L\Delta \sigma}{\sqrt{K}R}\right)^{2/3}+\frac{L\Delta}{R}$ & $-$ \vspace{1mm} \\
    
    {\makecell[l]{{\sc VRL-SGD\tnote{2}} \\ \citep{liang2019variance}}} & ${\left(\frac{L\Delta \sigma^2}{NKR}\right)^{1/2}+\left(\frac{L\Delta \sigma}{\sqrt{K}R}\right)^{2/3}+\frac{L\Delta}{R}}$ & ${-}$ \vspace{1mm} \\

    \rowcolor{Ocean}
    {\sc FedAvg-M} (Thm.~\ref{thm:Fedavg-m}) 
    & $\left(\frac{L\Delta \sigma^2}{NKR}\right)^{1/2}+\frac{L\Delta}{R}$ &  $-$ \vspace{0mm}  \\
    \midrule \vspace{-3mm}\\
    
    {\sc Variance-Reduction}  \vspace{1mm}\\ 

    \quad \makecell[l]{{\sc BVR-L-SGD} \\ \citep{murata2021bias}} & $\left(\frac{L\Delta \sigma}{NKR}\right)^{2/3}+\frac{\sigma^2}{NKR}+\frac{L\Delta}{R}$ & \makecell[l]{Sample smooth\\ $\mathcal{O}(K)$ minibatches\tnote{3}} \vspace{1mm} \\
    
    \quad \makecell[l]{{\sc CE-LSGD} \\ \citep{patel2022towards}} & $\left(\frac{L\Delta \sigma}{NKR}\right)^{2/3}+\frac{\sigma^2}{NKR}+\frac{L\Delta}{R}$ & \makecell[l]{Sample smooth\\ $\mathcal{O}(K)$ minibatches\tnote{3}} \vspace{1mm} \\
    
    \quad \makecell[l]{{\sc STEM} \\ \citep{khanduri2021stem}}  & $\frac{L\Delta+\sigma^2+\zeta^2}{(NKR)^{2/3}}+\frac{L\Delta}{R}$ & \makecell[l]{Sample smooth \\ Bounded hetero.} \vspace{1mm} \\

    \rowcolor{BPink}
    {\sc FedAvg-M-VR} (Thm.~\ref{thm:Fedavg-m-vr}) & $\left(\frac{L\Delta \sigma}{NKR}\right)^{2/3}+\frac{L\Delta}{R}$& Sample smooth\tnote{4} \vspace{0.25mm} \\

    \rowcolor{BPink}
    {{\sc FedAvg-M-VR} (Thm.~\ref{app:thm:fedavg_mvr})} & ${\left(\frac{L\Delta \sigma}{NKR}\right)^{2/3}+ \frac{\sigma^2}{NKR}+\frac{L\Delta}{R}}$& {Sample smooth\tnote{4}} \vspace{-0mm} \\
    
    \bottomrule
    \end{tabular}
    \begin{tablenotes}
        \footnotesize
        \item[1] The local learning rate vanishes to zero when data heterogeneity is unbounded, \ie, $\zeta\to \infty$.
        {\item[2] The works have not been published in peer-reviewed venues.}
        \item[3] A large number of minibatches are utilized on each client per communication round.
        {\item[4] The difference roots in the amount of batches used for initialization. See the values of $B$ in Thm.~\ref{app:thm:fedavg_mvr}.}
    \end{tablenotes}
    \end{threeparttable}
    }
    \vspace{-2mm}
    }
\end{table}

Tables \ref{tab:full_results} and \ref{tab:partial_results} present a comprehensive comparison of the convergence rates and associated assumptions of prior algorithms and our proposed methods. 
It is observed that by simply adding momentum to local steps, \fedavg, \scaffold, and their variance-reduced variants all attain state-of-the-art convergence rates without resorting to further assumptions such as bounded data heterogeneity. We support our theoretical findings with extensive numerical experiments. 

\vspace{-1mm}
\subsection{Related work} 
\vspace{-2mm}
\paragraph{FL with homogeneous clients.} 
\fedavg\ is a well-known algorithm introduced by \cite{McMahan2017CommunicationEfficientLO} as a heuristic to enhance communication efficiency and data privacy in federated learning. Numerous subsequent studies have focused on analyzing its convergence under the assumption of homogeneous datasets, where clients are independent and identically distributed (iid) and all clients participate fully \citep{stich2019local,yu2019parallel,wang2021cooperative,lin2020don,zhou2017convergence}. However, when dealing with heterogeneous clients and partial client participation, \fedavg\ is found to be vulnerable to data heterogeneity because of the "client drift" effect \citep{karimireddy2020scaffold,yang2021achieving,wang2020slowmo,li2019convergence}. 

\vspace{-3mm}
\paragraph{FL with heterogeneous clients.} 
Numerous research efforts are devoted to mitigating the impact of data heterogeneity in FL. For example, \cite{li2020federated} propose \fedprox, which introduces a proximal term to the objective function. \cite{yang2021achieving} utilize a two-sided learning rate approach, while \cite{wang2020tackling} propose \fednova, a normalized averaging method. Additionally, \cite{zhang2021fedpd} presents \fedpd, which addresses data heterogeneity from a primal-dual optimization perspective. Notably, \cite{karimireddy2020scaffold} introduces \scaffold, an effective algorithm that employs control variables to mitigate the influence of data heterogeneity and partial client participation. {\sc FedGate} \citep{haddadpour2021federated} and {\sc LED} \citep{alghunaim2023local} are two recent effective algorithms that have alleviated the impact of data heterogeneity, utilizing gradient tracking \citep{xu2015augmented,di2016next,pu2020distributed,xin2020improved,alghunaim2021unified,huang2024stochastic} and exact-diffusion \citep{Yuan2019ExactDF,yuan2020influence,yuan2023removing} techniques, respectively.

\vspace{-3mm}
\paragraph{FL with momentum.} 
The momentum mechanism dates back to Nesterov's acceleration \citep{Nesterov2004Intro}  and  Polyak’s heavy-ball method \citep{polyak1964some}, which later flourishes in the stochastic optimization \citep{Yan2018AUA,yu2019linear,liu2020improved} and other areas \citep{Yuan2021DecentLaMDM,he2023unbiased,he2023lower,chen2023optimal,huang2024stochastic}. Extensive research has explored incorporating momentum into FL \citep{reddi2021adaptive,wang2020slowmo,karimireddy2020mime,khanduri2021stem,patel2022towards,das2022faster,yu2019linear, xu2021fedcm}, and have demonstrated its impact on enhancing the empirical performance of FL methods \citep{wang2020slowmo,xu2021fedcm,reddi2021adaptive,jin2022accelerated,kim2022communication}. However, whether momentum can offer {\em theoretical benefits} to FL, especially in mitigating the impact of data heterogeneity, remains unclear. This work demonstrates that momentum can benefit non-iid federated learning simply and provably. Notably, the utility of momentum is demonstrated in  domains other than FL. For instance, \cite{guo2021novel} proves that momentum can correct the bias experienced by the {\sc Adam} method, while recently \cite{fatkhullin2023momentum} shows that momentum can improve the error feedback technique in communication compression. The analysis presented in this work distinguishes from prior works including \cite{guo2021novel,fatkhullin2023momentum} due to the unique challenges encountered in FL including multiple local updates, data heterogeneity, and partial client participation. 

\begin{table}[t]
\vspace{-15mm}
    \caption{\small The comparison of convergence rates of FL algorithms when {\bf $\boldsymbol{S}$ out of $\boldsymbol{N}$ clients} participate in training per iteration. Notations are the same as those in Table \ref{tab:full_results}.}
    \label{tab:partial_results}
    \centering{\small
    \scalebox{0.95}{
    \begin{threeparttable}
    \begin{tabular}{lll}
    \toprule
    {\bf Algorithm}  & {\bf Convergence Rate  $\expect[ \|\nabla f(\hat{x})\|^2]\lesssim$} & {\bf Assumptions} \vspace{1mm} \\ 
    \midrule 
    \makecell[l]{\scaffold \\ \citep{karimireddy2020scaffold}} & $\left(\frac{L\Delta \sigma^2}{SKR}\right)^{1/2}+\frac{L\Delta}{R}\left(\frac{N}{S}\right)^{2/3}$ & $-$ \vspace{2mm}\\

    \rowcolor{Ocean}
    {\sc SCAFFOLD-M} (Thm. \ref{thm:scaffold-m}) & $\left(\frac{L\Delta \sigma^2}{SKR}\right)^{1/2}+\frac{L\Delta}{R}\left(1+\frac{N^{2/3}}{S}\right)$ & $-$ \vspace{0mm} \\

    \midrule \vspace{-3mm}\\
    
    {\sc Variance-Reduction} \vspace{1mm} \\ 

    \quad \makecell[l]{{\sc MimeLiteMVR}\tnote{1} \\ \citep{karimireddy2020mime}} & $\left(\frac{L\Delta (\sigma+\zeta)}{R}\right)^{2/3}+\frac{L\Delta+\sigma^2+\zeta^2}{R}$ & \makecell[l]{Sample smooth\\ Noiseless grad.} \vspace{0mm} \\
     
    \quad \makecell[l]{MB-STORM\\ \citep{patel2022towards}} & $\left(\frac{L\Delta \sigma}{S\sqrt{K}R}\right)^{2/3}+\left(\frac{L\Delta \zeta}{SR}\right)^{2/3}+\frac{\zeta^2}{SR}+\frac{L\Delta}{R}+\frac{\sigma^2}{SKR}$ & \makecell[l]{Sample smooth \\ Bounded hetero.\\ $\mathcal{O}(K)$ minibatches\tnote{2} } \vspace{0mm} \\

    \quad \makecell[l]{CE-LSGD \tnote{1} \\ \citep{patel2022towards}} & $\left(\frac{L\Delta \sigma}{S\sqrt{K}R}\right)^{2/3}+\left(\frac{L\Delta \zeta}{SR}\right)^{2/3}+\frac{\zeta^2}{SR}+\frac{L\Delta}{R}+\frac{\sigma^2}{SKR}$ & \makecell[l]{Sample smooth \\ Bounded hetero.\\ $\mathcal{O}(K)$ minibatches\tnote{2} } \vspace{0mm} \\
    
    \rowcolor{BPink}
    {\sc SCAFFOLD-M-VR} (Thm.~\ref{thm:scaffold-m-vr}) & $\left(\frac{L\Delta \sigma}{S\sqrt{K}R}\right)^{2/3}+\frac{L\Delta }{R}\left(1+\frac{N^{1/2}}{S}\right)$ & Sample smooth\tnote{3} \vspace{0.25mm} \\

    \rowcolor{BPink}
    {{\sc SCAFFOLD-M-VR} (Thm.~\ref{app:thm:scaffold_mvr})} & ${\left(\frac{L\Delta \sigma}{S\sqrt{K}R}\right)^{2/3}+\frac{L\Delta }{R}\left(1+\frac{N^{1/2}}{S}\right) + \frac{\sigma^2}{SKR}}$ &{Sample smooth\tnote{3}} \vspace{0mm} \\
    
    \bottomrule
    \end{tabular}
    \begin{tablenotes}
        \footnotesize
        \item[1] {\sc MimeLiteMVR} and {\sc CE-LSGD} consider the setting of streaming clients. 
        \item[2] A large number of minibatches are utilized on each client per communication round.
        {\item[3] The difference roots in the amount of batches used for initialization. See the values of $B$ in Thm.~\ref{app:thm:scaffold_mvr}.}
    \end{tablenotes}
    \end{threeparttable}
    }
    }
    \vspace{-1mm}
\end{table}

\vspace{-4mm}
\section{Problem setup}
\vspace{-3mm}
This section formulates the problem of non-iid federated learning. 
Formally, we consider minimizing the following objective with the fewest number of client-server communication rounds:
\setlength{\abovedisplayskip}{1pt}\setlength{\belowdisplayskip}{1pt}
\begin{equation*}
    \min_{x\in\RR^d}\quad f(x):= \frac{1}{N}\sum_{i=1}^N f_i(x)\quad \text{where}\quad  f_i(x):=\expect_{\xi_i\sim \mathcal{D}_i}[F(x;\xi_i)].
\end{equation*}
Here, the random variable $\xi_i$ represents a local data point available at client $i$, while the function $f_i(x)$ denotes the non-convex local loss function associated with client $i$. This function takes expectation concerning the local data distribution $\mathcal{D}_i$. In practice, the local data distributions $\mathcal{D}_i$ among different clients typically differ from each other, resulting in the inequality $f_i(x) \neq f_j(x)$ for any pair of nodes $i$ and $j$. This phenomenon is commonly referred to as {\em data heterogeneity}. If all local clients were homogeneous, meaning that all local data samples follow the same distribution $\mathcal{D}$, we would have $f_i(x) = f_j(x)$ for any $i$ and $j$. In addition, throughout the paper, we assume that the function $f$ is bounded from below and possesses a global minimum $f^*$. To facilitate convergence analysis, we also introduce the following standard assumptions.\setlength{\abovedisplayskip}{2pt}\setlength{\belowdisplayskip}{2pt}

\begin{asp}[\sc Standard smoothness]\label{asp:smooth}
Each local objective $f_{i}$ is $L$-smooth, \ie, $\left\|\nabla f_{i}(x)-\nabla f_{i}(y)\right\| \leq L\|x-y\|$, for any $x,y\in\RR^d$ and $1\leq i\leq N$.
\end{asp}
\begin{asp}[\sc Sample-wise smoothness]\label{asp:sample_smooth}
Each sample-wise objective $F(x;\xi)$ is $L$-smooth, \ie, $\|\nabla F(x;\xi_i)-\nabla F(y;\xi_i)\|\leq L\|x-y\|$ for any $x,y\in\RR^d$, $1\leq i\leq N$, and $\xi_i\overset{iid}{\sim} \mathcal{D}_i$.
\end{asp}

\vspace{-2mm}
It is worth noting that Assumption \ref{asp:sample_smooth} implies Assumption \ref{asp:smooth}, which is typically used in variance-reduced algorithms, \eg, \cite{karimireddy2020mime,khanduri2021stem,Fang2018SPIDERNN,cutkosky2019momentum}. We will utilize either Assumption \ref{asp:smooth} or \ref{asp:sample_smooth} in different algorithms. It is worth highlighting that, these are the {\em only} assumptions required for all our theoretical analyses. 

\begin{asp}[\sc Stochastic Gradient]\label{asp:sgd_var} 
	There exists $\sigma\geq 0$ such that for any $x\in\RR^d$ and $1\leq i\leq N$, $\EE_{\xi_i}[\nabla F(x;\xi_i)]=\nabla f_i(x)$ and $\EE_{\xi_i}[\|\nabla F(x;\xi_i)-\nabla f_i(x)\|^2]\leq \sigma^2$
 where $\xi_i\overset{iid}{\sim} \mathcal{D}_i$.
\end{asp}

\vspace{-4mm}
\section{Accelerating FedAvg with momentum}
\vspace{-2mm}
This section focuses on full client participation. We will introduce momentum to both \fedavg and its variance-reduced extension. Furthermore, we will justify that the incorporation of momentum effectively mitigates the impact of data heterogeneity, leading to improved convergence rates.

\vspace{-2mm}
\subsection{\fedavg with momentum} \label{sec:fedavgm}
\vspace{-1mm}
\paragraph{Algorithm.} We introduce momentum to enhance the estimation of the stochastic gradient, resulting in the algorithm \fedavgm, as presented in Algorithm \ref{alg:fedavg_mom}. In \fedavgm, the subscript $i$ represents the client index, while the superscripts $r$ and $k$ denote the outer loop index and inner local update index, respectively. The structure of \fedavgm remains identical to the vanilla \fedavg, except for the inclusion of momentum in  gradient computation (see highlight in Algorithm \ref{alg:fedavg_mom}):
\begin{align}\label{eq:fedavg_m}
    g_i^{r,k}=\beta \nabla F(x^{r,k}_i;\xi^{r,k}_i) + (1-\beta )g^r,
\end{align}
where $\beta \in [0, 1]$ is the momentum coefficient, and $g^r$ represents a global gradient estimate updated in the outer loop $r$. It is important to note that \fedavgm will reduce to the vanilla \fedavg when $\beta = 1$. Furthermore, \fedavgm is easy to implement, as it maintains the same algorithmic structure and incurs no additional uplink communication overhead compared to \fedavg. 
{Notably, no extra downlink commmunication cost is needed if clients store the last iterate model $x^r$ so that momentum $g^{r+1}$ can recovered through $(x^{r+1}-x^r)/\gamma$.}

\begin{algorithm}[!t]
    \caption{\fedavgm: \fedavg with momentum}
    \label{alg:fedavg_mom}
    \begin{algorithmic}
        \REQUIRE{initial model $x^0$ and gradient estimate $g^0$, local learning rate $\eta$, global learning rate $\gamma$, momentum $\beta$}
        \FOR{$r=0, \cdots, R-1$}
            \FOR{each client $i\in \{1,\dots,N\}$ in parallel}
                \STATE{Initialize local model $x^{r,0}_i=x^r$}
                \FOR{$k=0, \cdots, K-1$} 
                    \STATE{
                            \colorbox{Ocean}{Compute $g^{r,k}_i= \beta \nabla F(x^{r,k}_i;\xi^{r,k}_i) + (1-\beta )g^r$} \hfill $\triangleright\,\mbox{\footnotesize{ $\beta=1$ implies \fedavg}}$  \\
                            \;Update local model $x^{r,k+1}_i=x^{r,k}_i-\eta g^{r,k}_i$
                    }
                \ENDFOR
            \ENDFOR
            \STATE{
                 Aggregate local updates $g^{r+1}= \frac{1}{\eta N K}\sum_{i=1}^N\left(x^r - x^{r,K}_i\right)$\\
                Update global model $x^{r+1}= x^r-\gamma g^{r+1}$
            }
        \ENDFOR
    \end{algorithmic}
\end{algorithm}

\vspace{-2mm}
\paragraph{Convergence property.} 
The inclusion of momentum in \fedavg yields notable theoretical improvements. Firstly, it eliminates the need for the data heterogeneity assumption, also known as the gradient similarity assumption. The assumption can be expressed as
\setlength{\abovedisplayskip}{1pt}\setlength{\belowdisplayskip}{1pt}
\begin{equation}
    \frac{1}{N}\sum_{i=1}^N \|\nabla f_i(x) - \nabla f(x)\|^2 \le \zeta^2,\quad \forall\,x \in \RR^d\tag*{(Bounded data heterogeneity)}
\end{equation}
where $\zeta^2$ measures the magnitude of data heterogeneity. By incorporating momentum, the above assumption is {\em no longer required} for the convergence analysis of \fedavg. Secondly, momentum enables \fedavg to converge at a state-of-the-art rate. These improvements are justified as follows: 
\setlength{\abovedisplayskip}{2pt}\setlength{\belowdisplayskip}{2pt}

\vspace{-1mm}
\begin{thm}\label{thm:Fedavg-m}
    Under Assumption \ref{asp:smooth} and \ref{asp:sgd_var}, if we set $g^0=0$, $\beta$, $\gamma$, and $\eta$ as in \eqref{eqn:fedavg-m-para},
    \fedavgm  enjoys
    \setlength{\abovedisplayskip}{1pt}\setlength{\belowdisplayskip}{1pt}
    \begin{equation*}
        \frac{1}{R}\sum_{r=0}^{R-1}\expect[\|\nabla f(x^r)\|^2]
        \lesssim \sqrt{\frac{L\Delta  \sigma^2}{NKR}}+\frac{L\Delta }{R},
    \end{equation*}
    \setlength{\abovedisplayskip}{2pt}\setlength{\belowdisplayskip}{2pt}
    where  $\Delta \triangleq f(x^0)-\min_x f(x)$ and  $\lesssim$ absorbs numeric numbers. See proof in Appendix \ref{app:fedavgm}.
\end{thm}

\vspace{-4mm}
\paragraph{Comparison with \fedavg.} Table \ref{tab:full_results} compares \fedavgm with prior algorithms when all clients participate in the training process. The results demonstrate that \fedavgm attains the most favorable convergence rate without relying on any assumption of data heterogeneity. Moreover, this rate matches the lower bound provided by \cite{Arjevani2019LowerBF}. Notably, a recent work \citep{huang2023distributed} establishes the convergence of \fedavg by relaxing the bounded data heterogeneity to a bound on $f^\star - \frac{1}{N}\sum_{i=1}^Nf_i^\star$ where $f^\star \triangleq\min_x f(x)$ and $f_i^\star\triangleq f_i(x)$. However, their convergence does not benefit from local updates. Moreover, it still suffers from data heterogeneity $f^\star - \frac{1}{N}\sum_{i=1}^Nf_i^\star$ and gets slow as the number of clients $N$ increases, resulting in a suboptimal rate.

\vspace{-3mm}
\paragraph{{Comparison with \fedcm.}} {\fedavgm  coincides with the \fedcm algorithm proposed by \cite{xu2021fedcm}. However,  our result outperforms that of \cite{xu2021fedcm} in several aspects. First, our convergence only utilizes the standard smoothness of objectives and gradient stochasticity while \cite{xu2021fedcm} additionally require bounded data heterogeneity and bounded gradients which are rarely valid in practice, suggesting the limitation of their result. Second, the convergence established by \cite{xu2021fedcm} is significantly weaker than ours and cannot even asymptotically approach the customary rate $O(1/\sqrt{NKR})$ in non-convex FL, as demonstrated by the results stated in Table~\ref{tab:full_results}.}

\vspace{-3mm}
\paragraph{Constant local learning rate.} 
Based on Theorem \ref{thm:Fedavg-m}, it can be inferred that when $R\gtrsim {NKL\Delta}/{\sigma^2}$, \fedavgm allows the utilization of {\em constant} local learning rate $\eta$ which does not necessarily decay as the number of communication rounds $R$ increases. This characteristic eases the tuning of the local learning rate and improves empirical performance. In contrast, many existing convergence results of \fedavg necessitate the adoption of local learning rates that diminish as $R$ increases, as exemplified by \eg, \cite{yang2021achieving,li2019convergence,karimireddy2020scaffold,koloskova2020unified}. 

\vspace{-3mm}
\paragraph{Intuition on the effectiveness of momentum.} The momentum mechanism relies on an accumulated gradient estimate $g^r$, which is updated through $g^{r+1}=\frac{\beta}{NK}\sum_{i=1}^N\sum_{k=0}^{K-1}\nabla F(x_i^{r,k};\xi_i^{r,k})+(1-\beta)g^r$. While $g^{r}$ is a biased gradient estimate, it exhibits reduced variance due to its accumulation nature compared to a stochastic gradient $\nabla F(x_i^{r,k};\xi_i^{r,k})$ computed with a single data minibatch. Importantly, by utilizing directions $\beta \nabla F(x_i^{r,k};\xi_i^{r,k})+(1-\beta) g^r$ for local updates, an ``anchoring'' effect is achieved, effectively mitigating the ``client-drift'' phenomenon. In the extreme case where $\beta=0$, all clients remain synchronized in their local updates, eliminating the drift incurred by data heterogeneity in the vanilla \fedavg. By appropriately tuning the coefficient $\beta$, \fedavgm maintains the same convergence rate as \citep{yang2021achieving} while removing the requirement of data heterogeneity assumption utilized in their analysis.

\vspace{-2mm}
\subsection{Variance-reduced \fedavg with momentum} 
\label{sec:fedavg-vrm}
\vspace{-1mm}
When each local loss function is further assumed to be sample-wise smooth (\ie, Assumption \ref{asp:sample_smooth}), we can replace the local descent direction in Algorithm \ref{alg:fedavg_mom} with a variance-reduced momentum direction
\begin{align}\label{eqn:vrm-d}
    g_i^{r,k}=  \nabla F(x^{r,k}_i;\xi^{r,k}_i) + (1-\beta )(g^r-\nabla F(x^{r-1};\xi^{r,k}_i))
\end{align}
to further enhance convergence, leading to variance-reduced \fedavg with momentum, or  \fedavgmvr for short, see the detailed algorithm in Appendix \ref{app:fedavgmvr}.
The variable $x^{r-1}$ is the last-iterate global model maintained in the server. The construction of the variance-reduced direction \eqref{eqn:vrm-d} effectively mitigates the influence of within-client gradient noise and can be traced back to SARAH \citep{nguyen2017sarah} and STORM \citep{cutkosky2019momentum} in stochastic optimization; more discussion can be found in \cite{tan2022towards}.
Same as \fedavgm, turning off the variance-reduced momentum of \fedavgmvr, \ie, setting $\beta=1$, recovers \fedavg. \fedavgmvr shares the same algorithmic structure and uplink communication workload as \fedavg.  

\begin{thm}\label{thm:Fedavg-m-vr}\setlength{\abovedisplayskip}{1pt}\setlength{\belowdisplayskip}{1pt}
    Under Assumption \ref{asp:sample_smooth} and \ref{asp:sgd_var}, if we take $g^0=\frac{1}{NB}\sum_{i=1}^N\sum_{b=1}^{B}\nabla F(x^0;\xi^{b}_i)$\footnote{We use $B$ data minibatches per client to initialize the gradient estimate $g^0$ with small variance $\EE[\|g^0-\nabla f(x^0)\|^2]$, after which only one minibatch is utilized per local gradient computation. The same applies below.} with $\{\xi^{b}_i\}_{b=1}^B\overset{iid}{\sim}\mathcal{D}_i$ and set  $\beta$, $\gamma$, $\eta$, and $B$ as in \eqref{eqn:fedavg-mvr-para}, 
    \fedavgmvr enjoys
    \begin{equation*}
        \frac{1}{R}\sum_{r=0}^{R-1}\expect[\|\nabla f(x^r)\|^2]\lesssim \left(\frac{L\Delta  \sigma}{NKR}\right)^{2/3}+\frac{L\Delta }{R}.
        \vspace{-3mm}
    \end{equation*}
    \vspace{-2mm}
\end{thm}


\vspace{-3mm}
\paragraph{Comparison with prior works.} 
\fedavgmvr outperforms existing variance-reduced FL methods in convergence rate, as justified by the results listed in Table \ref{tab:full_results}. Additionally, compared to {\sc BVR-L-SGD} \citep{murata2021bias} and {\sc CE-LSGD} \citep{patel2022towards}, \fedavgmvr conducts each local update using $1+1/K=O(1)$ minibatches on average, contrasting with the $O(K)$ minibatches in {\sc BVR-L-SGD} and {\sc CE-LSGD}. Furthermore, in comparison to {\sc STEM} \citep{khanduri2021stem}, \fedavgmvr does not rely on the assumption of bounded data heterogeneity.

Based on discussions in Sections \ref{sec:fedavgm} and \ref{sec:fedavg-vrm}, we demonstrate that \fedavgm and \fedavgmvr, in the context of full client participation,  can achieve the state-of-the-art convergence rate without resorting to any stronger assumption, \eg, bounded data heterogeneity or impractical algorithmic structures such as a large number of minibatches in local gradient computation.

\vspace{-3mm}
\section{Accelerating SCAFFOLD with momentum}
\vspace{-2mm}
This section addresses the scenario where a random subset of clients participates in each training round. To tackle the challenges of partial participation, \scaffold employs a control variable in each client to counteract the  ``client drift'' effect during local updates. We will introduce momentum to both \scaffold and its variance-reduced extension to gain better convergence results.

\vspace{-2mm}
\subsection{\scaffold with momentum}\label{sec:scaffoldm}
\vspace{-1mm}
\begin{algorithm}[!t]
    \caption{\scaffoldm: {\sc SCAFFOLD} with momentum}
    \label{alg:scaffold_mom}
    \begin{algorithmic}
        \REQUIRE{initial model $x_0$, gradient estimator $g^0$, control variables $\{c_i^0\}_{i=1}^N$ and $c^0$, local learning rate $\eta$, global learning rate $\gamma $, momentum $\beta$}
        \FOR{$r=0, \cdots, R-1$}
            \STATE{Uniformly sample clients $\mathcal{S}_r\subseteq \{1, \cdots, N\}$ with $|\cS_r|=S$}
            \FOR{each client $i\in \cS_r$ in parallel}
                \STATE{Initialize local model $x^{r,0}_i=x^r$}
                \FOR{$k=0, \cdots, K-1$}
                    \STATE{
                            \colorbox{Ocean}{Compute $g^{r,k}_i=  \beta (\nabla F(x^{r,k}_i;\xi^{r,k}_i)-c_i^r+c^r) + (1-\beta )g^r$}  $\triangleright\mbox{\footnotesize{$\beta=1$ implies  \scaffold}}$  \\
                            \;Update local model $x^{r,k+1}_i=x^{r,k}_i-\eta g^{r,k}_i$ 
                    }
                \ENDFOR 
                \STATE{Update control variable $c^{r+1}_i:= \frac{1}{K}\sum_{k=0}^{K-1} \nabla F(x^{r,k}_i;\xi^{r,k}_i)$ (for $i\notin \cS_r$, $c_i^{r+1}=c_i^r$)}
            \ENDFOR
            \STATE Aggregate local updates $g^{r+1}= \frac{1}{\eta S K}\sum_{i\in\cS_r}\left(x^r - x^{r,K}_i\right)$\\
            \STATE Update global model $x^{r+1}= x^r-\gamma g^{r+1}$\\
            \STATE Update control variable $c^{r+1}= c^{r}+ \frac{1}{N}\sum_{i\in\mathcal{S}_r} (c^{r+1}_i-c^{r}_i)$
        \ENDFOR
    \end{algorithmic}
\end{algorithm}

\vspace{-3mm}
\paragraph{Algorithm.} We introduce momentum to enhance the estimation of the stochastic gradient, resulting in the newly proposed algorithm \scaffoldm, outlined in Algorithm \ref{alg:scaffold_mom}. In \scaffoldm, $S$ clients are randomly selected from a pool of $N$ clients for each training iteration. The control variables $c_i$ and $c$ are maintained by the client and server, respectively. In \scaffold, the local descent direction is given by $\nabla F(x_i^{r,k};\xi_i^{r.k})-c_i^r+c^r$. In contrast, \scaffoldm incorporates momentum directions for local updates:
\begin{align}\label{eq:scaffold_m}
g^{r,k}_i= \beta (\nabla F(x^{r,k}_i;\xi^{r,k}_i)-c_i^r+c^r) + (1-\beta )g^r,
\end{align}
where $g^r$ represents the global stochastic gradient vector maintained by the server. It is worth noting that \scaffoldm can reduce to \scaffold by setting $\beta=1$.

\vspace{-4mm}
\paragraph{Convergence property.} \hspace{-3mm} Our momentum yields notable theoretical improvements to \scaffold: 
\begin{thm}\label{thm:scaffold-m}
    Under Assumption \ref{asp:smooth} and \ref{asp:sgd_var}, 
    if we take $g^0=0$, $c_i^0=\frac{1}{B}\sum_{b=1}^{B}\nabla F(x^0;\xi^{b}_i)$ with $\{\xi^{b}_i\}_{b=1}^B\overset{iid}{\sim}\mathcal{D}_i$, $c^0=\frac{1}{N}\sum_{i=1}^Nc_i^0$ and set  $\beta$,  $\gamma$, $\eta$, and $B$  as in \eqref{eqn:scaffold-m-para},  
    \scaffoldm enjoys
    \begin{equation*}
        \frac{1}{R}\sum_{r=0}^{R-1}\expect[\|\nabla f(x^r)\|^2]\lesssim \sqrt{\frac{L\Delta  \sigma^2}{SKR}}+\frac{L\Delta  }{R}\left(1+\frac{N^{2/3}}{S}\right).
        \vspace{-2mm}
    \end{equation*}
\end{thm}

\vspace{-3mm}
\paragraph{Comparison with \scaffold.}
Compared to \scaffold, \scaffoldm exhibits provably faster convergence under partial participation, as justified in the comparison in Table \ref{tab:partial_results}. Specifically, when the gradients are noiseless (\ie, $\sigma^2=0$), achieving the same level of stationarity $\EE[\|\nabla f(\hat x)\|^2]$ requires a  ratio, between \scaffoldm and \scaffold, of communication rounds:
\setlength{\abovedisplayskip}{1pt}\setlength{\belowdisplayskip}{1pt}
\begin{equation}
\frac{1+N^{2/3}/S}{(N/S)^{2/3}}=\left(\frac{S}{N}\right)^{2/3}+\frac{1}{S^{1/3}}.
\end{equation}
Thus, if $S\asymp N^{2/3}$, \scaffoldm achieves up to  $N^{2/9}$ times improvement in comparison to the vanilla \scaffold, when aiming for the same level of stationarity. This improvement is significant as $N$, the number of clients in FL, is typically large. It is also worth highlighting that prior to our \scaffoldm, \scaffold was the only known non-iid FL method, to the best of our knowledge, that is robust to both unbounded data heterogeneity and partial client sampling, and capable of attaining linear speedup without relying on impractical algorithmic structures. The development of \scaffoldm provides an alternative and superior choice.   \setlength{\abovedisplayskip}{2pt}\setlength{\belowdisplayskip}{2pt}

\vspace{-4mm}
\subsection{Variance-reduced \scaffold with momentum}\label{sec:vr-scaffold-m}
\vspace{-2mm}
Similar to \fedavgmvr, when the loss functions further enjoy the sample-wise smoothness property, we can obtain \scaffoldmvr by 
replacing momentum directions in Algorithm \ref{alg:scaffold_mom} with  variance-reduced momentum directions
\begin{equation}\label{eqn:sca-vrm-d}
g_i^{r,k}=\nabla F(x^{r,k}_i;\xi^{r,k}_i) -\beta (c_i^r-c^r)+ (1-\beta )(g^r-\nabla F(x^{r-1};\xi^{r,k}_i)).
\end{equation}
The detailed algorithm is in Appendix \ref{app:scaffoldmvr}, and the convergence is shown below.
\vspace{-1mm}
\begin{thm}\label{thm:scaffold-m-vr}
    Under Assumption \ref{asp:sample_smooth} and \ref{asp:sgd_var}, if we take $c_i^0=\frac{1}{B}\sum_{b=1}^{B}\nabla F(x^0;\xi^{b}_i)$ with $\{\xi_i^b\}_{b=1}^B\overset{iid}{\sim}\mathcal{D}_i$, $g^0=c^0=\frac{1}{N}\sum_{i=1}^Nc_i^0$  and set  $\beta$,  $\gamma $, $\eta$, and   $B$ as in \eqref{eqn:scaffold-mvr-para},
   \scaffoldmvr enjoys
    \begin{equation*}
        \frac{1}{R}\sum_{r=0}^{R-1}\expect[\|\nabla f(x^r)\|^2]\lesssim \left(\frac{L\Delta  \sigma}{S\sqrt{K}R}\right)^{2/3}+\frac{L\Delta  }{R}\left(1+\frac{N^{1/2}}{S}\right).
        \vspace{-3mm}
    \end{equation*}
\end{thm}

\vspace{-3mm}
\paragraph{Comparison with variance-reduced methods.} 
    \scaffoldmvr outperforms all existing variance-reduced federated learning methods under partial participation in terms of convergence rate when data heterogeneity is severe (\ie, $\zeta^2$ is large), see results listed in Table \ref{tab:partial_results}. Moreover, \scaffoldmvr has the following additional advantages.  Compared to {\sc MimeLiteMVR} \citep{karimireddy2020mime}, \scaffoldmvr does not need access to noiseless (full-batch) local gradients per iteration.
    Compared to {\sc MB-STORM} \citep{patel2022towards} and {\sc CE-LSGD} \citep{patel2022towards}, \scaffoldmvr does not require bounded data heterogeneity and conducts each local update using $1+1/K=O(1)$ minibatches on average, instead of $\mathcal{O}(K)$.

\vspace{-1mm}
Based on  Sections \ref{sec:scaffoldm} and \ref{sec:vr-scaffold-m}, we demonstrate that \scaffoldm and \scaffoldmvr, in the context of partial client participation, can achieve state-of-the-art convergence rates without resorting to any stronger assumption, \eg, bounded data heterogeneity or impractical algorithmic structures such as a large number of minibatches in local gradient computation.

\begin{figure}[t]
\vspace{-12mm}
    \centering
    \hspace{-3mm}
    \subfigure[Acceleration by momentum]{\includegraphics[clip,width=0.34\textwidth]{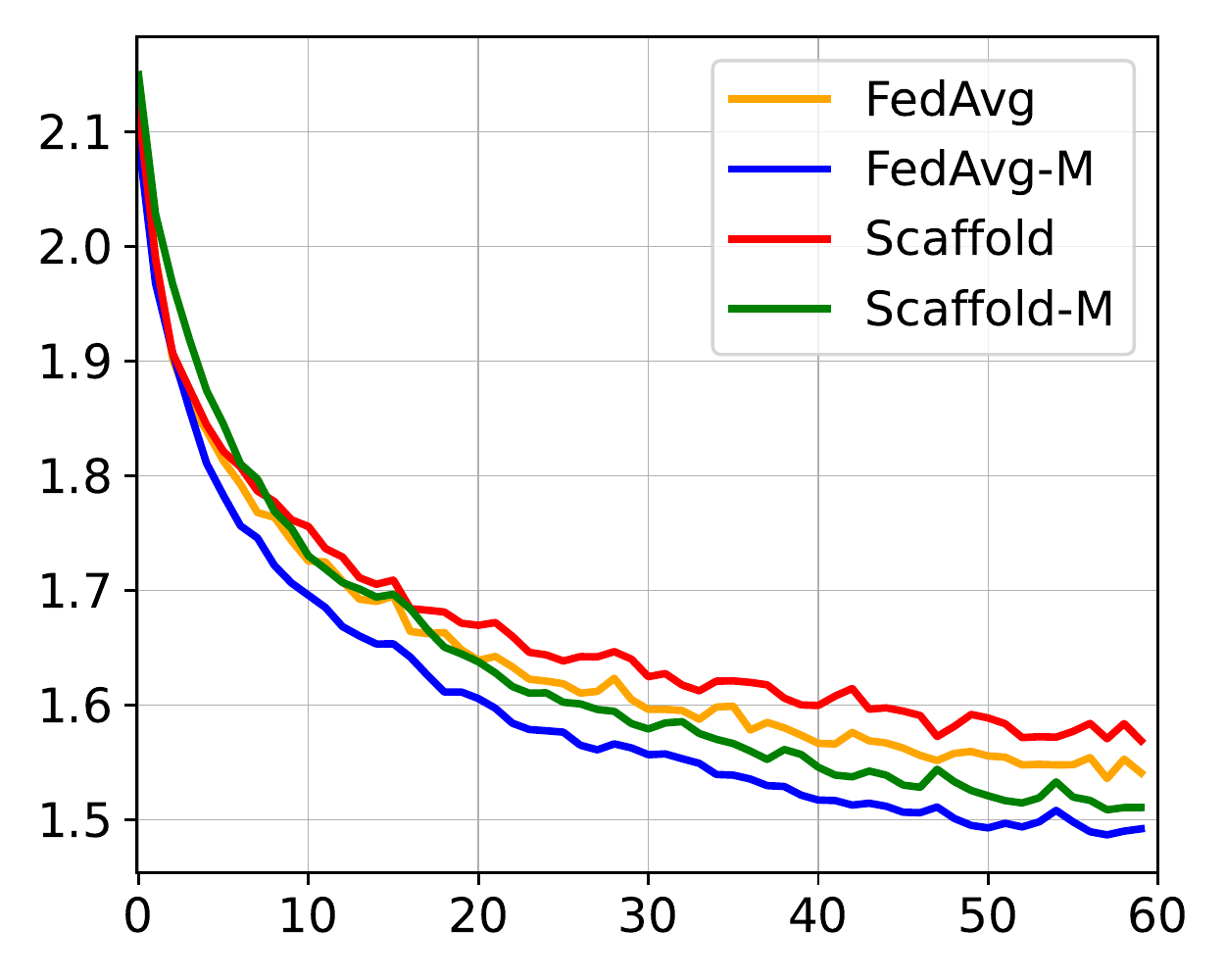} \label{fig:full_m}} \hspace{-3mm}
    \subfigure[Comparison of VR methods]{\includegraphics[clip,width=0.34\textwidth]{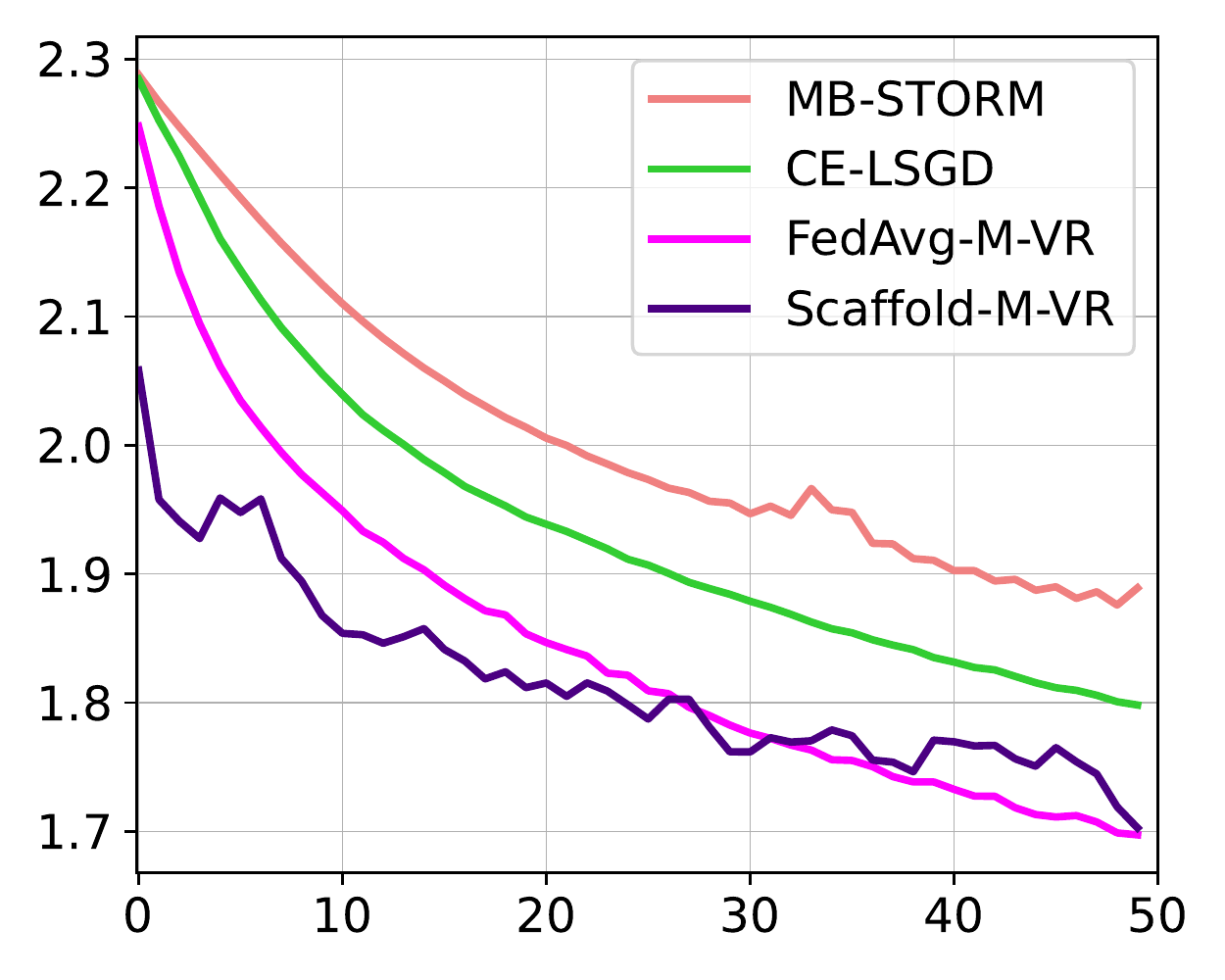} \label{fig:full_vr}} \hspace{-3mm}
    \subfigure[Partial participation]{\includegraphics[clip,width=0.34\textwidth]{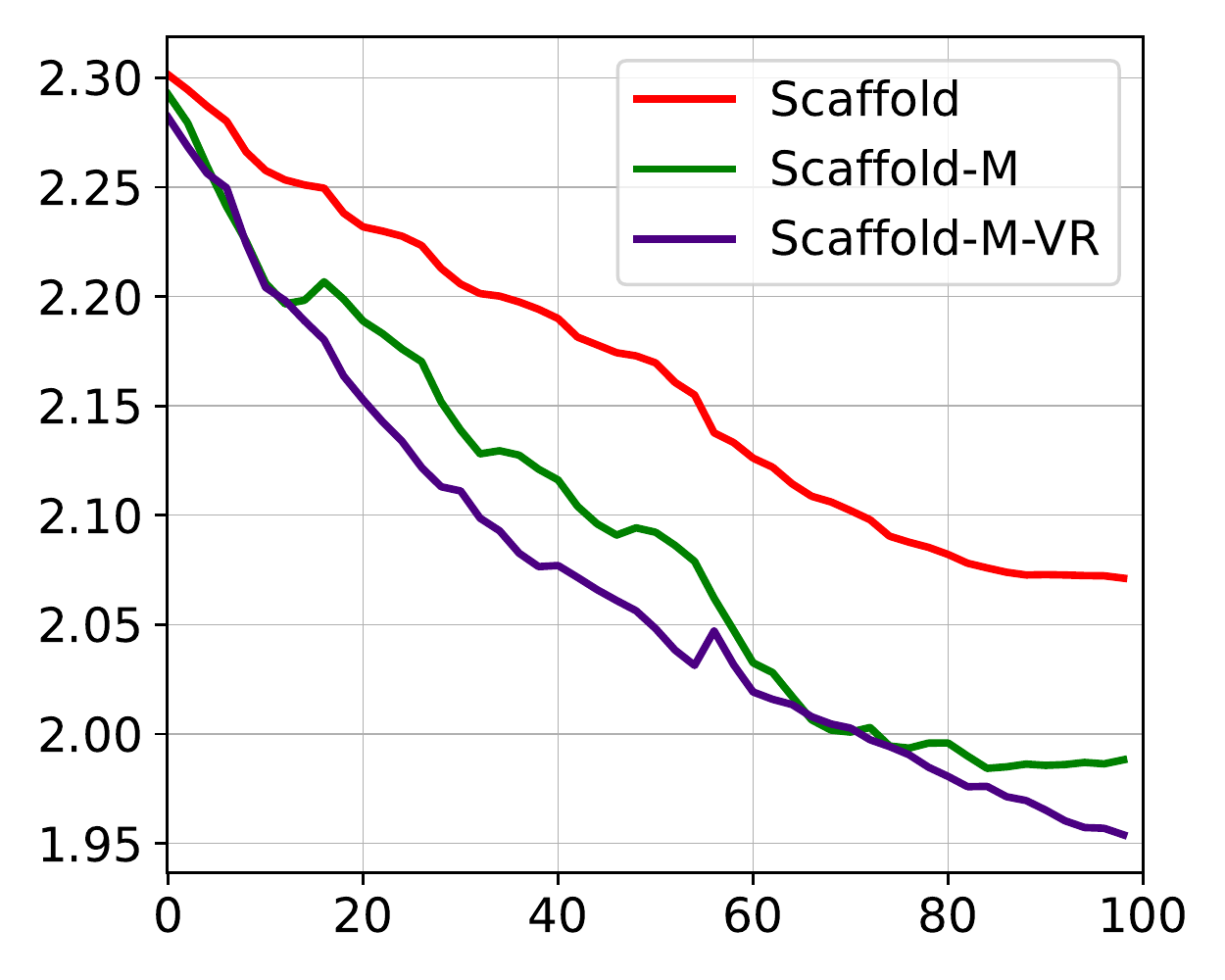} \label{fig:par}}
    \hspace{-3mm}
    \vspace{-3mm}
    \caption{Test loss of three-layer MLP versus the number of communication rounds}
    \vspace{-1mm}
\end{figure}
   
\vspace{-4mm}
\section{Experiments}
\vspace{-3mm}
We present experiments on the CIFAR-10 dataset \citep{krizhevsky2009learning} with two neural networks (three-layer MLP,  ResNet-18) to justify the efficacy of our proposed algorithms. We evaluate them along with baselines including \fedavg \citep{konevcny2016federated}, \scaffold \citep{karimireddy2020scaffold}, {\sc MB-STORM},  {\sc CE-LSGD} \citep{patel2022towards}.  Parameters (such as learning rates) in our implementation are set by grid search. {We defer more experimental details  and results (\eg, investigating the impact of momentum value $\beta$, setups with large $N$) to Appendix \ref{app_sec:exp}.}

\vspace{-2mm}
\subsection{MLP experiments}
\vspace{-2mm}
The MLP experiments involve $K=32$ local updates and  $N=10$ clients with data generated via the Dirichlet distribution \citep{hsu2019measuring} with a parameter of $0.5$ and $0.2$ for full and partial client participation, respectively (small parameter value implies severe heterogeneity). 

Firstly, we compare the performance of \fedavgm and \scaffoldm with their momentum-free counterparts, namely the vanilla \fedavg and \scaffold, under full client participation. The results are presented in Figure \ref{fig:full_m}, in which it can be observed that incorporating momentum significantly accelerates the convergence of both \fedavg and \scaffold. 

\vspace{-1mm}
Secondly, we compare four momentum-based variance-reduced methods: {\sc Minibatch-STORM}, {\sc CE-LSGD}, \fedavgmvr (our Algorithm \ref{alg:fedavg_mom_vr}), and \scaffoldmvr (our Algorithm \ref{alg:scaffold_mom_vr}), under full client participation. The comparison is illustrated in Figure \ref{fig:full_vr}. Our proposed methods outperform {\sc Minibatch-STORM} and {\sc CE-LSGD} with substantial margins.

\vspace{-1mm}
Lastly, we investigate the case of partial client participation with $S=1$ and compare the performance of \scaffoldm and \scaffoldmvr with vanilla \scaffold. The results are presented in Figure \ref{fig:par}. Once again, we observe that the introduction of momentum leads to significant improvements even when only a few clients participate in each round of training.


\vspace{-2mm}
\subsection{ResNet18 experiments}\label{subsec:resnet}
\vspace{-2mm}
We further compare the above algorithms with a larger model: ResNet18 \citep{he2016deep} under varying data heterogeneity by setting the parameter of Dirichlet distribution as $0.5$ and $0.1$, respectively, where a small parameter value suggests severe data heterogeneity.  The experiment involves $N=10$ clients and $K=16$ local updates. We set $S=2$ in partial client participation. 

Figure \ref{fig:resnet_low} reports the test accuracy of full and partial client participation under mild data heterogeneity while Figure \ref{fig:resnet_high} presents the counterparts under severe data heterogeneity, where the bottom right one is smoothed by plotting the best-so-far result. Again, we observe that \fedavgm and \scaffoldm significantly outperform the vanilla \fedavg and \scaffold. Moreover,
for ResNet18 and severe data heterogeneity, \fedavgm and \scaffoldm exhibit notably greater advantages over their momentumless counterparts than for MLP scenarios under milder data heterogeneity. The observation demonstrates amplified advantages of the introduced momentum in larger models and severely heterogeneous data, which is aligned with our theoretical predictions and suggests the promising utility of our proposed methods in real-world applications. 



\begin{figure}[t]
    \centering
    \vspace{-12mm}
    \hspace{-3mm}
    \subfigure[Full (left) \& partial (right) participation under {\em mild} data heterogeneity]{
    \includegraphics[clip,width=0.34\textwidth]{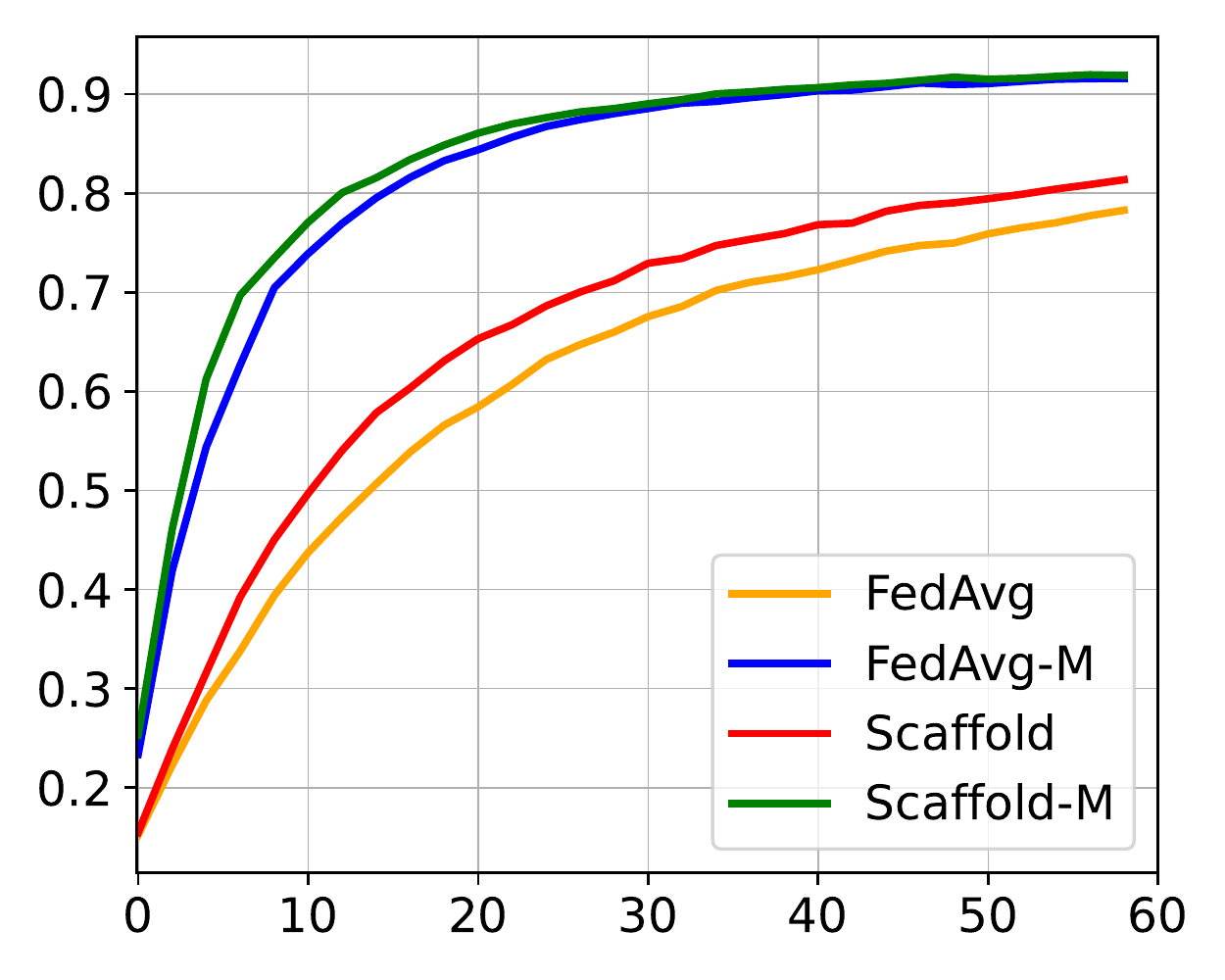} 
    \includegraphics[clip,width=0.34\textwidth]{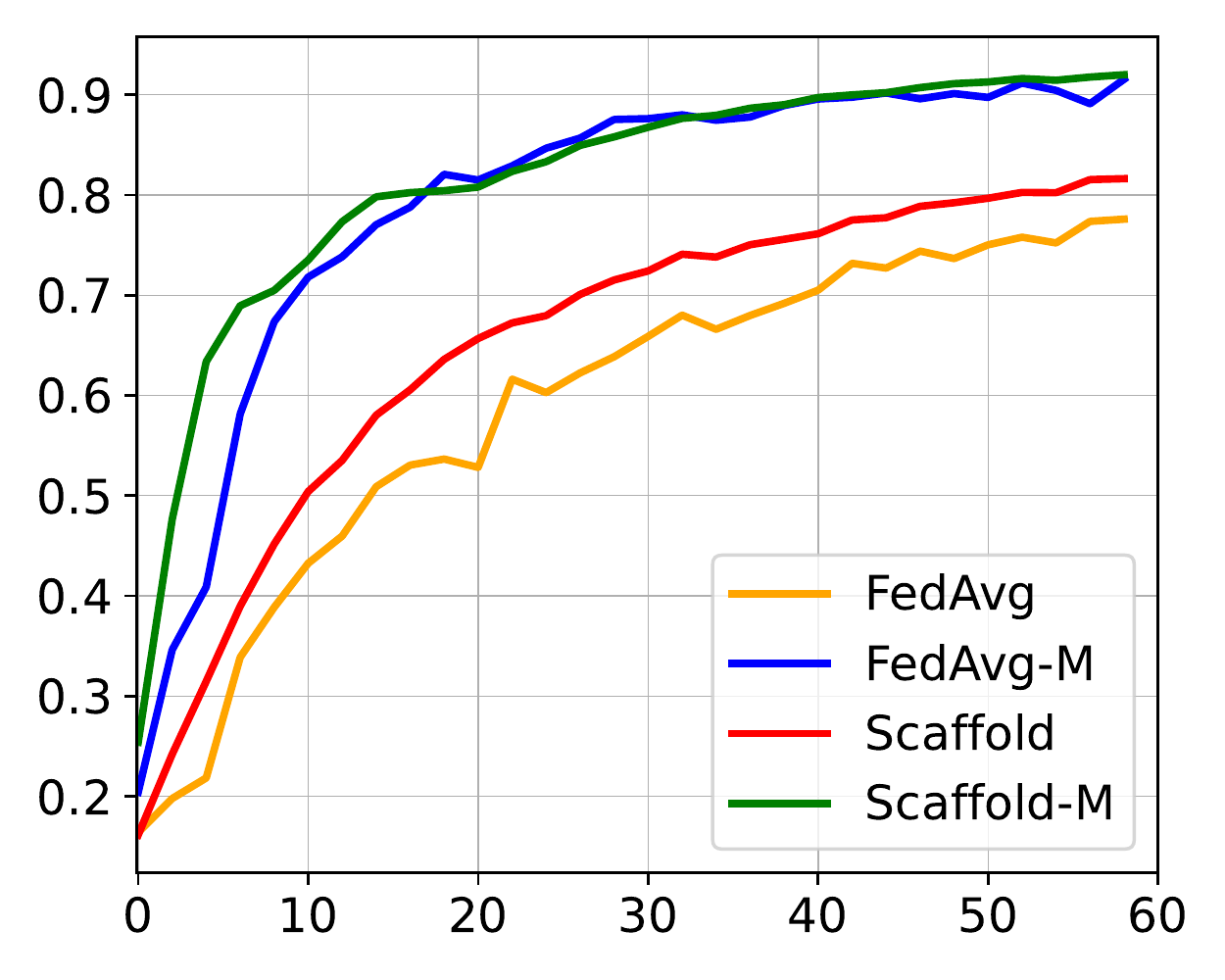}
    \label{fig:resnet_low}} 
    \hspace{-3mm}
    \vspace{-3mm}\\
    \hspace{-3mm}
     \subfigure[\mbox{Full (left) \& partial (right, best-so-far) participation under {\em severe}  heterogeneity}]{
     \includegraphics[clip,width=0.34\textwidth]{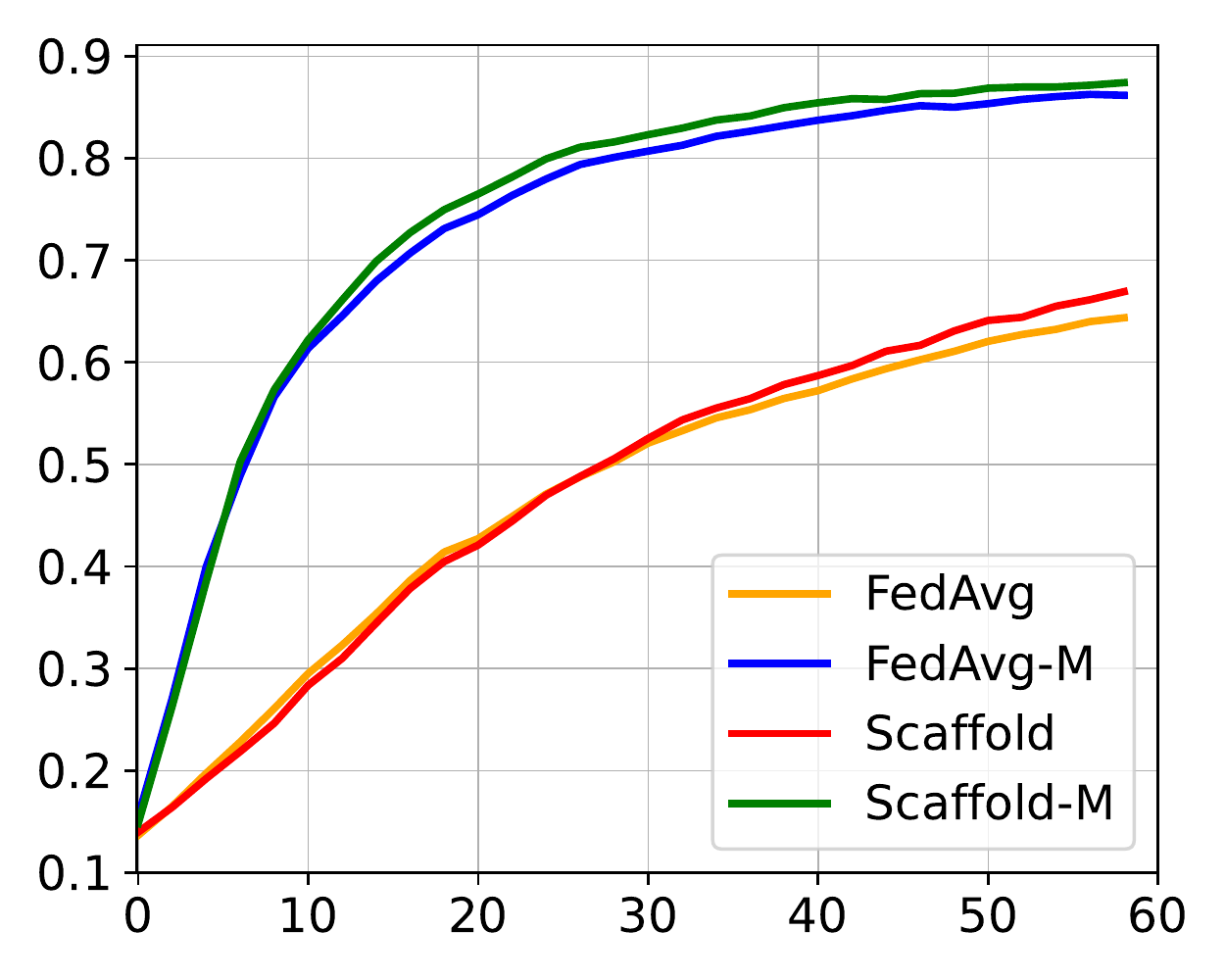} 
     \includegraphics[clip,width=0.34\textwidth]{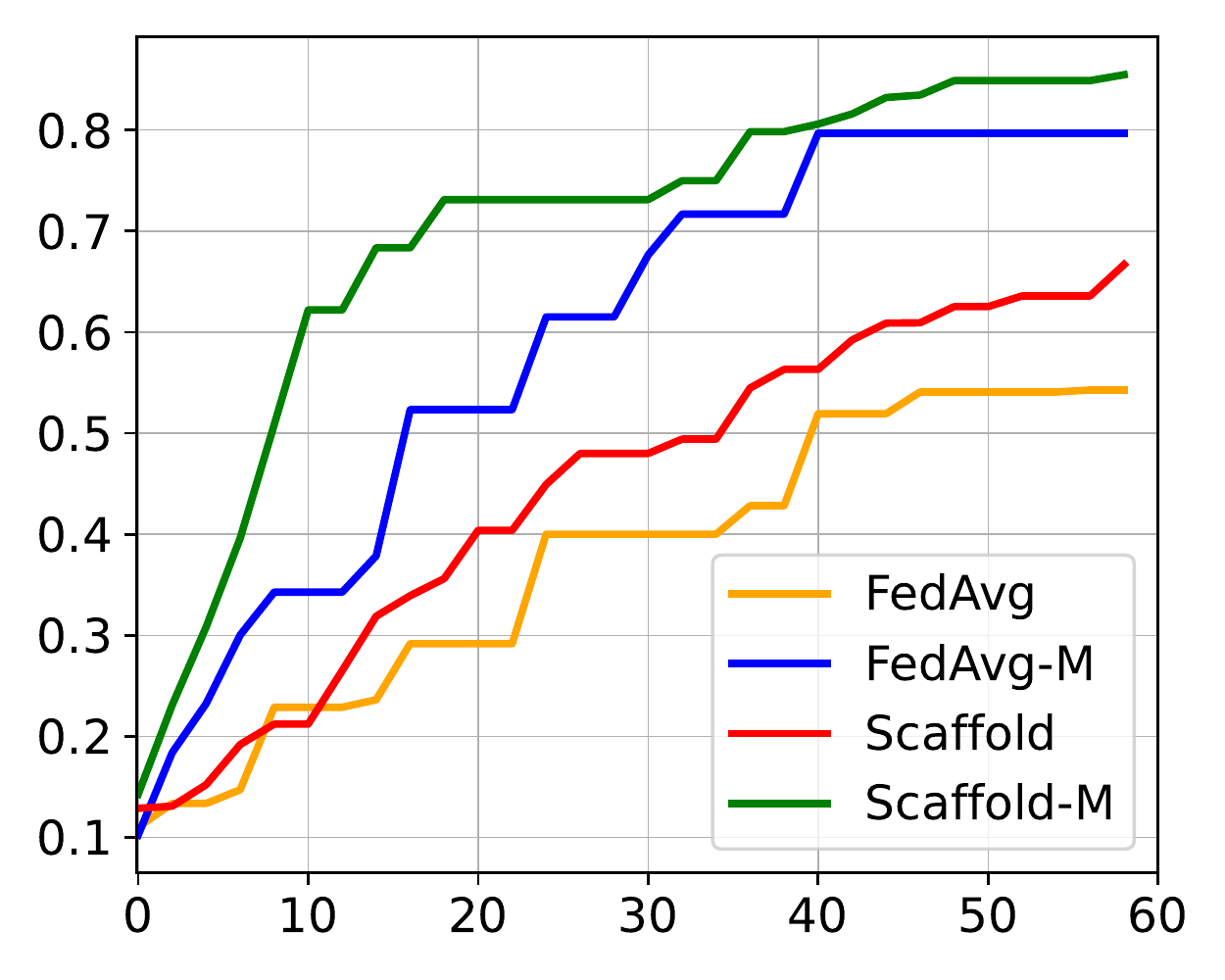} \label{fig:resnet_high}}
    \hspace{-3mm}
    \vspace{-3mm}
    \caption{Test accuracy of ResNet18 versus the number of communication rounds}
    \vspace{-1mm}
\end{figure}

\vspace{-3mm}
\section{Conclusion}
\vspace{-2mm}
We propose momentum variants of \fedavg and \scaffold under various client participation situations and objectives' smoothness. All the momentum variants make simple and practical adjustments to \fedavg and \scaffold yet obtain state-of-the-art performance among their peers, especially under severe data heterogeneity or small gradient variance. In particular, \fedavgm converges under unbounded data heterogeneity and admits constant local learning rates, giving the {\em first} neat convergence for \fedavg-type methods; \scaffoldm is the {\em first} FL method that outperforms \scaffold unconditionally. Experiments conducted support our theoretical findings.

\newpage

\bibliography{iclr2024_conference}
\bibliographystyle{iclr2024_conference}

\newpage
\appendix
\section{Preliminaries of proofs}
Let $\mathcal{F}^0=\emptyset$ and 
$\mathcal{F}^{r,k}_i:= \sigma(\{x^{r,j}_i\}_{0\leq j\leq k}\cup \mathcal{F}^r)$ 
and  $\mathcal{F}^{r+1}:=\sigma(\cup_i \mathcal{F}^{r,K}_i)$ for all $r\geq 0$ where $\sigma(\cdot)$ indicates the $\sigma$-algebra. Let $\expect_r [\cdot]:=\expect [\cdot|\mathcal{F}^r]$ be the expectation, conditioned on the filtration $\mathcal{F}^{r}$, with respect to the random variables $\{\cS^r,\{\xi_i^{r,k}\}_{1\leq i\leq N,0\leq k<K}\}$ in the $r$-th iteration. We also use $\EE[\cdot]$ to denote the global expectation over all randomness in algorithms. Through out the proofs, we use $\sum_i$ to represent the sum over $i\in \{1,\dots,N\}$, while $\sum_{i\in \cS^r}$ denotes the sum over $i\in \cS^r$. 
Similarly, we use $\sum_k$ to represent the sum of $k\in\{0,\dots,K-1\}$.
For all $r\geq 0$, we define the following auxiliary variables to facilitate proofs:
\begin{align*}
        \cE_r&:= \expect [\|\nabla f(x^r)-g^{r+1}\|^2], \\
        U_r &:= \frac{1}{NK}\sum_{i}\sum_{k}\expect [\|x_i^{r,k}-x^r\|]^2, \\
        \zeta_i^{r,k}&:= \expect [x_i^{r,k+1}-x_i^{r,k} | \mathcal{F}^{r,k}_i], \\
        \Xi_r&:= \frac{1}{N}\sum_{i=1}^N\expect [\|\zeta^{r,0}_i\|^2],\\
        V_r&:= \frac{1}{N}\sum_{i=1}^N\expect[\|c^r_i-\nabla f_i(x^{r-1})\|^2].
\end{align*}
We remark that quantity $V_r$ is only used in the analysis of \scaffold-based algorithms.
Throughout the appendix, we let $\Delta  :=f(x^0)-f^*, G_0 :=\frac{1}{N}\sum_{i}\|\nabla f_i(x^0)\|^2$, $x^{-1}:=x^0$ and $\mathcal{E}_{-1}:=\expect[\|\nabla f(x^0)-g^0\|^2]$. We will use the following foundational lemma for all our algorithms.

\begin{lemma}\label{lem:onestep}
    Under Assumption \ref{asp:smooth}, if $\gamma L\leq \frac{1}{24}$, the following holds all $r\geq 0$:
    \begin{equation*}
        \expect [f(x^{r+1})] \leq \expect [f(x^r)] - \frac{11\gamma }{24}\EE[\|\nabla f(x^r)\|^2] + \frac{13\gamma }{24} \mathcal{E}_{r}.
    \end{equation*}
\end{lemma}

\begin{proof}
    Since $f$ is $L$-smooth, we have
    \begin{align}
        f(x^{r+1})\leq &f(x^r)+\langle \nabla f(x^r),x^{r+1}-x^r\rangle +\frac{L}{2}\|x^{r+1}-x^r\|^2\\
        =& {f(x^r)-\gamma  \|\nabla f(x^r)\|^2+\gamma  \langle \nabla f(x^r),\nabla f(x^r)-g^{r+1}\rangle+\frac{L\gamma ^2}{2}\|g^{r+1}\|^2}.
    \end{align}
    Since $x^{r+1}=x^r-\gamma  g^{r+1}$, using Young's inequality, we further have
    \begin{equation*}
        \begin{aligned}
            &f(x^{r+1})\\
            \leq& f(x^r)-\frac{\gamma }{2}\|\nabla f(x^r)\|^2+\frac{\gamma }{2} \|\nabla f(x^r)-g^{r+1}\|^2+L\gamma ^2(\|\nabla f(x^r)\|^2+\|\nabla f(x^r)-g^{r+1}\|^2) \\
            \leq& f(x^r) - \frac{11\gamma }{24}\|\nabla f(x^r)\|^2 + \frac{13\gamma }{24} \|\nabla f(x^r)-g^{r+1}\|^2,
        \end{aligned}
    \end{equation*}
    where the last inequality is due to $\gamma L\leq \frac{1}{24}$. Taking the global expectation completes the proof.
\end{proof}

To handle local updates and client sampling, we will also use the following technical lemmas.

\begin{lemma}[\cite{karimireddy2020scaffold}]\label{lem:bias-var}
Suppose $\{X_1,\cdots,X_{\tau}\}\subset \reals^d$ be  random variables that are potentially dependent. If their marginal means and variances satisfy $\expect[X_i] = \mu_i$ and  $\expect[\|X_i - \mu_i\|^2]\leq \sigma^2$, then it holds that
\[
    \expect\left[\left\|\sum_{i=1}^\tau X_i\right\|^2\right] \leq \left\|\sum_{i=1}^\tau \mu_i\right\|^2+ \tau^2 \sigma^2.
\]
If they are correlated in the Markov way such that  $\expect [X_i | X_{i-1}, \cdots X_{1}] = \mu_i$ and  $\expect[\|X_i - \mu_i\|^2\mid \mu_i]\leq \sigma^2$, \ie, the variables $\{X_i - \mu_i\}$ form a martingale. Then the following tighter bound holds:
\[
    \expect\left[\left\|\sum_{i=1}^\tau X_i\right\|^2\right] \leq 2\EE\left[\left\|\sum_{i=1}^\tau \mu_i\right\|^2\right]+ 2\tau\sigma^2.
\]
\end{lemma}

\begin{lemma}\label{lem:par_sample}
    Given vectors $v_1, \cdots, v_N\in \reals^d$ and $\bar v=\frac{1}{N}\sum_{i=1}^N v_i$, if we sample $\mathcal{S}\subset \{1,\cdots,N\}$ uniformly randomly such that $|\mathcal{S}|=S$, then it holds that
    \begin{equation*}
            \expect \left[\left\|\frac{1}{S}\sum_{i\in \mathcal{S}}v_i\right\|^2\right] 
            = \|\bar v\|^2 + \frac{N-S}{S(N-1)}\frac{1}{N}\sum_{i=1}^N\|v_i-\bar v\|^2.
    \end{equation*}
\end{lemma}
\begin{proof}
Letting $\one\{i\in\mathcal{S}\}$ be the indicator for the event $i\in\cS_r$, we prove this lemma by direct calculation as follows:
    \begin{equation*}
        \begin{aligned}
            \expect \left[\left\|\frac{1}{S}\sum_{i\in \mathcal{S}}v_i\right\|^2\right] 
            &= \expect \left[\left\|\frac{1}{S}\sum_{i=1}^Nv_i\one\{i\in\mathcal{S}\}\right\|^2\right] \\
            &= \frac{1}{S^2}\expect \left[\left (\sum_{i}\|v_i\|^2\one\{i\in\mathcal{S}\}+2\sum_{i<j}v_i^\top v_j\one\{i,j\in\mathcal{S}\}\right)\right] \\
            &= \frac{1}{SN}\sum_{i=1}^N\|v_i\|^2+\frac{1}{S^2}\frac{S(S-1)}{N(N-1)}2\sum_{i<j}v_i^\top v_j \\
            &= \frac{1}{SN}\sum_{i=1}^N\|v_i\|^2 + \frac{1}{S^2}\frac{S(S-1)}{N(N-1)}\left(\left\|\sum_{i=1}^N v_i\right\|^2-\sum_{i=1}^N \|v_i\|^2\right) \\
            &= \frac{N-S}{S(N-1)}\frac{1}{N}\sum_{i=1}^N\|v_i\|^2+\frac{N(S-1)}{S(N-1)}\|\overline{v}\|^2 \\
            &= \frac{N-S}{S(N-1)}\frac{1}{N}\sum_{i=1}^N\|v_i-\overline{v}\|^2 + \|\overline{v}\|^2.
        \end{aligned}
    \end{equation*}
\end{proof}

In the following subsections, we present complete proofs of our main results. For \fedavgm and \scaffoldm, our proofs only rely on  Assumption \ref{asp:smooth} and \ref{asp:sgd_var}, while for \fedavgmvr and \scaffoldmvr, our proofs rely on Assumption \ref{asp:sample_smooth} and \ref{asp:sgd_var}.

\newpage
\section{\fedavg with momentm}
\subsection{\fedavgm}\label{app:fedavgm}
In this subsection, we present the proofs for the \fedavgm algorithm.
\begin{lemma}\label{lem:Fedavg_grad_err}
    If $\gamma L\leq \frac{\beta }{6}$, the following holds for $r\geq 1$:
    \begin{equation*}
       \cE_r\leq \left(1-\frac{8\beta }{9}\right)\mathcal{E}_{r-1}+\frac{4\gamma^2 L^2}{\beta }\expect [\|\nabla f(x^{r-1})\|^2]+\frac{2\beta ^2\sigma^2}{NK}+4\beta  L^2U_r .
    \end{equation*}
    Additionally, it holds for $r=0$ that
    \begin{equation*}
        \mathcal{E}_0\leq (1-\beta )\mathcal{E}_{-1}+\frac{2\beta ^2\sigma^2}{NK}+4\beta  L^2U_0.
    \end{equation*}
\end{lemma}

\begin{proof}
    For $r\geq 1$,
    \begin{equation*}
        \begin{aligned}
            \cE_r
            &= \expect [\|\nabla f(x^r)-g^{r+1}\|^2] \\
            &= \expect \left[\left\| (1-\beta )(\nabla f(x^r)-g^r)+\beta  \left(\nabla f(x^r)-\frac{1}{NK}\sum_{i}\sum_{k}\nabla F(x_i^{r,k};\xi_i^{r,k})\right)\right\|^2\right] \\
            &= \expect\left[\|(1-\beta )(\nabla f(x^r)-g^r)\|^2\right] + \beta ^2\expect\left[\left\|\nabla f(x^r)-\frac{1}{NK}\sum_{i,\,k}\nabla F(x_i^{r,k};\xi_i^{r,k})\right\|^2\right]\\
            &\quad\quad + 2\beta \expect\left[\left\langle (1-\beta )(\nabla f(x^r)-g^r),\nabla f(x^r)-\frac{1}{NK}\sum_{i,\,k}\nabla f(x_i^{r,k})\right\rangle\right].
        \end{aligned}
    \end{equation*}
    Note that $\{\nabla F(x_i^{r,k};\xi_i^{r,k})\}_{0\leq k<K}$ are sequentially correlated. Applying the AM-GM inequality and Lemma \ref{lem:bias-var}, we have
    \begin{align}
                    \cE_r&\leq \left(1+\frac{\beta }{2}\right)\expect [\|(1-\beta )(\nabla f(x^r)-g^r)\|^2] + 2\beta  L^2U_r  + 2\beta ^2\left(\frac{\sigma^2}{NK}+L^2U_r \right).
    \end{align}
    Using the AM-GM inequality again and Assumption \ref{asp:smooth}, we have 
    \begin{align}
                    \cE_r
            &\leq (1-\beta )^2\left(1+\frac{\beta }{2}\right)\left[\left(1+\frac{\beta}{2}\right)\mathcal{E}_{r-1}+\left(1+\frac{2}{\beta }\right)L^2\expect[\|x^r-x^{r-1}\|^2]\right]+\frac{2\beta ^2\sigma^2}{NK}+4\beta  L^2U_r \\
            &\leq (1-\beta )\mathcal{E}_{r-1}+\frac{2}{\beta }L^2\expect[\|x^r-x^{r-1}\|^2]+\frac{2\beta ^2\sigma^2}{NK}+4\beta  L^2U_r  \\
            &\leq \left(1-\frac{8\beta }{9}\right)\mathcal{E}_{r-1}+4\frac{\gamma^2 L^2}{\beta }\expect [\|\nabla f(x^{r-1})\|^2]+\frac{2\beta ^2\sigma^2}{NK}+4\beta  L^2U_r,
    \end{align}
    where we plug in $\|x^r-x^{r-1}\|^2\leq 2\gamma ^2(\|\nabla f(x^{r-1})\|^2+\|g^{r}-\nabla f(x^{r-1})\|^2)$ and use $\gamma L\leq \frac{\beta }{6}$ in the last inequality. Similarly for $r=0$,
    \begin{equation*}
        \begin{aligned}
            \mathcal{E}_0
            &\leq \left(1+\frac{\beta }{2}\right)\expect [\|(1-\beta )(\nabla f(x^0)-g^0)\|^2] + 2\beta  L^2U_0  + 2\beta ^2\left(\frac{\sigma^2}{NK}+L^2U_0 \right) \\
            &\leq (1-\beta )\mathcal{E}_{-1}+\frac{2\beta ^2\sigma^2}{NK}+4\beta  L^2U_0 .
        \end{aligned}
    \end{equation*}
\end{proof}

\begin{lemma}\label{lem:Fedavg_client_drift}
    If $\eta LK\leq\frac{1}{\beta }$, the following holds for $r\geq 0$:
    \begin{equation*}
        U_r \leq 2eK^2\Xi_r + K\eta^2\beta ^2\sigma^2(1+2K^3L^2\eta^2\beta ^2).
    \end{equation*}
\end{lemma}

\begin{proof}
    Recall that $\zeta^{r,k}_i:= \expect[x^{r,k+1}_i-x^{r,k}_i|\mathcal{F}^{r,k}_i] = -\eta\left((1-\beta )g^r+\beta \nabla f_i(x^{r,k}_i)\right)$. Then we have
    \begin{equation*}
        \begin{aligned}
        \expect[ \|\zeta^{r,j}_i-\zeta^{r,j-1}_i\|^2]
        &\leq \eta^2L^2\beta ^2\expect [\|x^{r,j}_i-x^{r,j-1}_i\|^2] \\
        &\leq \eta^2L^2\beta ^2(\eta^2\beta ^2\sigma^2+\expect[\|\zeta^{r,j-1}_i\|^2). 
        \end{aligned}
    \end{equation*}
    For any $1\leq j\leq k-1\leq K-2$, using $\eta L\leq \frac{1}{\beta  K}\leq \frac{1}{\beta (k+1)}$, we have
    \begin{equation*}
        \begin{aligned}
            \expect [\|\zeta^{r,j}_i\|^2]
            &\leq \left(1+\frac{1}{k}\right) \expect [\|\zeta^{r,j-1}_i\|^2] + (1+k)\expect [\|\zeta^{r,j}_i-\zeta^{r,j-1}_i\|^2] \\
            &\leq \left(1+\frac{2}{k}\right) \expect [\|\zeta^{r,j-1}_i\|^2] + (k+1)L^2\eta^4\beta ^4\sigma^2 \\
            &\leq e^2\expect [\|\zeta^{r,0}_i\|^2] + 4k^2L^2\eta^4\beta ^4\sigma^2,
        \end{aligned}
    \end{equation*}
    where the last inequality is by unrolling the recursive bound and using $\left(1+\frac{2}{k}\right)^k\leq e^2$.
    By Lemma \ref{lem:bias-var}, it holds that for $k\geq 2$,
    \begin{equation*}
        \begin{aligned}
            \expect [\|x^{r,k}_i-x^r\|^2]
            &\leq 2\expect \left[\left\|\sum_{j=0}^{k-1}\zeta^{r,j}_i\right\|^2\right] + 2k\eta^2\beta ^2\sigma^2 \\
            &\leq 2k\sum_{j=0}^{k-1}\expect[\|\zeta^{r,k}_i\|^2]+2k\eta^2\beta ^2\sigma^2 \\
            &\leq 2e^2k^2\expect[\|\zeta^{r,0}_i\|^2]+2k\eta^2\beta ^2\sigma^2(1+4k^3L^2\eta^2\beta ^2).
        \end{aligned}
    \end{equation*}
    This is also valid for $k=0,1$. Summing up over $i$ and $k$ finishes the proof.
\end{proof}

\begin{lemma}\label{lem:Fedavg_grad_norm}
    If $288e (\eta KL)^2((1-\beta )^2+e(\beta \gamma LR)^2)\leq 1$,
    then it holds for $r\geq 0$ that
    \begin{equation*}
        \sum_{r=0}^{R-1}\Xi_r\leq \frac{1}{72eK^2L^2}\sum_{r=-1}^{R-2}(\mathcal{E}_{r}+\expect[\|\nabla f(x^{r})\|^2])+2\eta^2\beta ^2eRG_0 .
    \end{equation*}
\end{lemma}
\begin{proof}
    Note that $\zeta_i^{r,0}=-\eta((1-\beta )g^r+\beta  \nabla f_i(x^r))$,
    \begin{equation*}
        \frac{1}{N}\sum_{i=1}^N\|\zeta_i^{r,0}\|^2 \leq 2\eta^2\left((1-\beta )^2\|g^r\|^2+\beta ^2\frac{1}{N}\sum_{i=1}^N\|\nabla f_i(x^r)\|^2\right).
    \end{equation*}
    Using Young's inequality, we have for any $q>0$ that
    \begin{equation*}
        \begin{aligned}
            \expect [\|\nabla f_i(x^r)\|^2]
            &\leq (1+q)\expect[\|\nabla f_i(x^{r-1})\|^2]+(1+q^{-1})L^2\expect[\|x^r-x^{r-1}\|^2]\\
            &\leq (1+q)\expect[\|\nabla f_i(x^{r-1})\|^2]+2(1+q^{-1})\gamma^2 L^2(\mathcal{E}_{r-1}+\expect[\|\nabla f(x^{r-1})\|^2])\\
            &\leq (1+q)^r\expect[\|\nabla f_i(x^0)\|^2]+\frac{2}{q}\gamma^2 L^2\sum_{j=0}^{r-1}(\mathcal{E}_{j}+\expect[\|\nabla f(x^{j})\|^2)(1+q)^{r-j}.
        \end{aligned}
    \end{equation*}
    Take $q=\frac{1}{r}$ and we have
    \begin{equation}\label{eqn:vninbvsdvds}
        \expect [\|\nabla f_i(x^r)\|^2]\leq e\expect [\|\nabla f_i(x^0)\|^2]+2e(r+1)\gamma^2 L^2\sum_{j=0}^{r-1}(\mathcal{E}_{j}+\expect[\|\nabla f(x^{j})\|^2).
    \end{equation}
    Note that this inequality is valid for $r=0$. Therefore, using \eqref{eqn:vninbvsdvds}, we have
    \begin{align}
        \sum_{r=0}^{R-1}\Xi_r
            &\leq \sum_{r=0}^{R-1} 2\eta^2\expect\left[(1-\beta )^2\|g^r\|^2+\beta ^2\frac{1}{N}\sum_{i=1}^N\|\nabla f_i(x^r)\|^2\right] \\
            &\leq \sum_{r=0}^{R-1} 2\eta^2\left(2(1-\beta )^2(\mathcal{E}_{r-1}+\expect[\|\nabla f(x^{r-1})\|^2])+\beta ^2\frac{1}{N}\sum_{i=1}^N\expect[\|\nabla f_i(x^r)\|^2]\right) \\
            &\leq \sum_{r=0}^{R-1} 4\eta^2(1-\beta )^2(\mathcal{E}_{r-1}+\expect[\|\nabla f(x^{r-1})\|^2]) \\ 
            &\quad\quad +2\eta^2\beta ^2\sum_{r=0}^{R-1} \left(\frac{e}{N}\sum_{i=1}^N\expect [\|\nabla f_i(x^0)\|^2]+2e(r+1)(\gamma L)^2\sum_{j=0}^{r-1}(\mathcal{E}_{j}+\expect[\|\nabla f(x^{j})\|^2])\right)\\
            &\leq 4\eta^2(1-\beta )^2\sum_{r=0}^{R-1} (\mathcal{E}_{r-1}+\expect[\|\nabla f(x^{r-1})\|^2]) \\
            &\quad\quad +2\eta^2\beta ^2\left( eRG_0 +2e(\gamma LR)^2\sum_{r=0}^{R-2} (\mathcal{E}_{r}+\expect[\|\nabla f(x^{r})\|^2])\right).
    \end{align}
    Rearranging the equation and applying the upper bound of $\eta$ completes the proof.
\end{proof}

\begin{thm}
    Under Assumption \ref{asp:smooth} and \ref{asp:sgd_var}, if we take $g^0=0$, 
    \begin{equation}\label{eqn:fedavg-m-para}
        \begin{aligned}
            &\beta  = \min\left\{, \sqrt{\frac{NKL\Delta  }{\sigma^2R}}\right\}\text{ for any constant }c\in (0,1],\quad \gamma =\min\left\{\frac{1}{24L}, \frac{\beta }{6L}\right\},\\
            &\eta KL \lesssim \min\left\{1,\frac{1}{\beta \gamma LR},\left(\frac{L\Delta }{G_0\beta^3R}\right)^{1/2}, \frac{1}{(\beta  N)^{1/2}}, \frac{1}{(\beta ^3NK)^{1/4}}\right\}
        \end{aligned}
    \end{equation}
    then \fedavgm converges as
    \begin{equation*}
        \frac{1}{R}\sum_{r=0}^{R-1}\expect[\|\nabla f(x^r)\|^2]
        \lesssim \sqrt{\frac{L\Delta  \sigma^2}{NKR}}+\frac{L\Delta }{R}.
    \end{equation*}
    Here $G_0 :=\frac{1}{N}\sum_{i=1}^N\|\nabla f_i(x^0)\|^2$.
\end{thm}

\begin{proof}
    Combining Lemma \ref{lem:Fedavg_grad_err} and \ref{lem:Fedavg_client_drift}, we have
    \begin{align}
             \cE_r\leq & \left(1-\frac{8\beta }{9}\right)\mathcal{E}_{r-1}+4\frac{(\gamma L)^2}{\beta }\expect[\|\nabla f(x^{r-1})\|^2]+\frac{2\beta ^2\sigma^2}{NK}\\
             &\quad +4\beta  L^2\left( 2eK^2\Xi_r + K\eta^2\beta ^2\sigma^2(1+2K^3L^2\eta^2\beta ^2\right).
    \end{align}
    and 
    \begin{equation}
        \mathcal{E}_0\leq (1-\beta )\mathcal{E}_{-1}+\frac{2\beta ^2\sigma^2}{NK}+4\beta  L^2\left( 2eK^2\Xi_0 + K\eta^2\beta ^2\sigma^2(1+2K^3L^2\eta^2\beta ^2)\right).
    \end{equation}
    Summing over $r$ from $0$ to $R-1$ and applying Lemma \ref{lem:Fedavg_grad_norm},
    \begin{equation*}
        \begin{aligned}
            \sum_{r=0}^{R-1}\cE_r
            &\leq \left(1-\frac{8\beta }{9}\right)\sum_{r=-1}^{R-2}\cE_r + 4\frac{(\gamma L)^2}{\beta }\sum_{r=0}^{R-2}\expect [\|\nabla f(x^{r})\|^2] + 2\frac{\beta ^2\sigma^2}{NK}R \\
            &\quad\quad +4\beta  L^2\left(2eK^2\sum_{r=0}^{R-1}\Xi_r + RK\eta^2\beta ^2\sigma^2(1+2K^3L^2\eta^2\beta ^2)\right) \\
            &\leq \left(1-\frac{7\beta }{9}\right)\sum_{r=-1}^{R-2}\cE_r + \left(4\frac{(\gamma L)^2}{\beta }+\frac{\beta }{9}\right) \sum_{r=-1}^{R-2}\expect[\|\nabla f(x^{r})\|^2]+16\beta ^3(e\eta KL)^2RG_0  \\
            &\quad\quad +\frac{2\beta ^2\sigma^2}{NK}R+4\beta ^3(\eta KL)^2\left(\frac{1}{K}+2(\eta KL\beta )^2\right)\sigma^2R \\
            &\leq \left(1-\frac{7\beta }{9}\right)\sum_{r=-1}^{R-2}\cE_r + \frac{2\beta }{9}\sum_{r=-1}^{R-2}\expect [\|\nabla f(x^{r})\|^2]+16\beta ^3(e\eta KL)^2RG_0 +\frac{4\beta ^2\sigma^2}{NK}R.
        \end{aligned}
    \end{equation*}
    Here in the last inequality we apply
    \begin{equation*}
        4\beta (\eta KL)^2\left(\frac{1}{K}+2(\eta KL\beta )^2\right) \leq \frac{2}{NK}\quad\text{and}\quad \gamma {L}\leq\frac{\beta }{6}.
    \end{equation*}

    Therefore, 
    \begin{equation*}
        \sum_{r=0}^{R-1}\cE_r\leq \frac{9}{7\beta }\mathcal{E}_{-1} + \frac{2}{7}\expect [\sum_{r=-1}^{R-2}\|\nabla f(x^{r})\|^2] + \frac{144}{7}(e\beta \eta KL)^2G_0 R+ \frac{36\beta \sigma^2}{7NK}R.
    \end{equation*}
    Combine this inequality with Lemma \ref{lem:onestep} and we get
    \begin{equation*}
        \frac{1}{\gamma }\expect [f(x^R)-f(x^0)]\leq -\frac{1}{7}\sum_{r=0}^{R-1}\expect[\|\nabla f(x^r)\|^2] + \frac{39}{56\beta }\mathcal{E}_{-1}+ \frac{78}{7}(e\beta \eta KL)^2G_0 R+ \frac{39\beta \sigma^2}{14NK}R.
    \end{equation*}
    Finally, noticing that $g^0=0$ implies $\mathcal{E}_{-1}\leq 2L(f(x^0)-f^*)=2L\Delta  $, we obtain
    \begin{equation*}
        \begin{aligned}
            \frac{1}{R}\sum_{r=0}^{R-1}\expect[\|\nabla f(x^r)\|^2]
            &\lesssim \frac{L\Delta  }{\gamma LR}+\frac{\mathcal{E}_{-1}}{\beta  R} + (\beta \eta KL)^2G_0 + \frac{\beta \sigma^2}{NK} \\
            &\lesssim \frac{L\Delta  }{R}+ \frac{L\Delta  }{\beta  R}+\frac{\beta \sigma^2}{NK}+(\beta \eta KL)^2G_0  \\
            &\lesssim \frac{L\Delta }{R}+\sqrt{\frac{L\Delta  \sigma^2}{NKR}}.
        \end{aligned}
    \end{equation*}
\end{proof}

\subsection{\fedavgmvr}\label{app:fedavgmvr}
In this subsection, we present the proofs for the \fedavgmvr algorithm, shown as in Algorithm \ref{alg:fedavg_mom_vr}.
\begin{algorithm}[ht]
    \caption{\fedavgmvr: \fedavg with variance-reduced momentum}
    \label{alg:fedavg_mom_vr}
    \begin{algorithmic}
        \REQUIRE{initial model $x^{-1}=x^0$ and gradient estimate $g^0$,  local learning rate $\eta$, global learning rate $\gamma$, momentum $\beta$}
        \FOR{$r=0, \cdots, R-1$}
            \FOR{each client $i\in \{1,\dots,N\}$ in parallel}
                \STATE{Initial local model $x^{r,0}_i=x^r$}
                \FOR{$k=0, \cdots, K-1$}
                    \STATE{
                            \colorbox{BPink}{Compute direction $g^{r,k}_i=
                                \nabla F(x^{r,k}_i;\xi^{r,k}_i) + (1-\beta )(g^r-\nabla F(x^{r-1};\xi^{r,k}_i))$}\\
                            \;Update local model $x^{r,k+1}_i=x^{r,k}_i-\eta g^{r,k}_i$
                    }
                \ENDFOR
            \ENDFOR
            \STATE{
                Aggregate local updates $g^{r+1}= \frac{1}{\eta N K}\sum_{i=1}^N\left(x^r - x^{r,K}_i\right)$\\
                Update global model global $x^{r+1}= x^r-\gamma g^{r+1}$
            }
        \ENDFOR
    \end{algorithmic}
\end{algorithm}

\begin{lemma}\label{lem:Fedavg_VR_grad_err}
    If $\gamma L\leq \sqrt{\frac{\beta  NK}{54}}$, the following holds for $r\geq 1$:
    \begin{equation*}
        \cE_r\leq (1-\frac{8\beta }{9})\mathcal{E}_{r-1}+\frac{4}{\beta }L^2U_r +\frac{3\beta ^2\sigma^2}{NK}+\frac{6(\gamma L)^2}{NK}\expect[\|\nabla f(x^{r-1})\|^2.
    \end{equation*}
    Also for $r=0$, it holds that
    \begin{equation*}
        \mathcal{E}_0\leq (1-\beta )\mathcal{E}_{-1}+\frac{4}{\beta }L^2U_r +\frac{3\beta ^2\sigma^2}{NK}.
    \end{equation*}
\end{lemma}
\begin{proof}
    \begin{equation*}
        \begin{aligned}
            &\cE_r
            = \expect \left[\left\|\frac{1}{NK}\sum_{i,\,k}\nabla F(x^{r,k}_i;\xi^{r,k}_i)+(1-\beta )\left(g^r-\frac{1}{NK}\sum_{i,\,k}\nabla F(x^{r-1};\xi^{r,k}_i)\right)-\nabla f(x^r)\right\|^2\right] \\
            =& \expect\Bigg[\Bigg\|(1-\beta )(g^r-\nabla f(x^{r-1}))+\frac{1}{NK}\sum_{i,\,k}\nabla F(x^{r,k}_i;\xi^{r,k}_i)-\nabla f(x^r) \\
            &\quad+(1-\beta )\left(\nabla f(x^{r-1})-\frac{1}{NK}\sum_{i,\,k}\nabla F(x^{r-1};\xi^{r,k}_i)\right)\Bigg\|^2\Bigg] \\
            =& (1-\beta )^2\mathcal{E}_{r-1} + \underbrace{2\expect \left[\left\langle (1-\beta )(g^r-\nabla f(x^{r-1})), \frac{1}{NK}\sum_{i,\,k}\nabla f_i(x^{r,k}_i)-\nabla f(x^r)\right\rangle\right]}_{\Lambda_1} \\
            &\hspace{-3mm}+\underbrace{\expect \left\|\frac{1}{NK}\sum_{i,\,k}\nabla F(x^{r,k}_i;\xi^{r,k}_i)-\nabla f(x^r) +(1-\beta )\left(\nabla f(x^{r-1})-\frac{1}{NK}\sum_{i,\,k}\nabla F(x^{r-1};\xi^{r,k}_i)\right) \right\|^2}_{\Lambda_2}.
        \end{aligned}
    \end{equation*}
    By the AM-GM inequality and Assumption \ref{asp:sample_smooth},
    \begin{equation*}
        \Lambda_1\leq \beta (1-\beta )^2\mathcal{E}_{r-1}+\frac{1}{\beta }L^2U_r .
    \end{equation*}
    By Assumption \ref{asp:sample_smooth},
    \begin{equation*}
        \begin{aligned}
            \Lambda_2&= \expect \Bigg[\Bigg\| \frac{1}{NK}\sum_{i,\,k}(\nabla F(x^{r,k}_i;\xi^{r,k}_i)-\nabla F(x^{r};\xi^{r,k}_i))+\beta \left(\frac{1}{NK}\sum_{i,\,k}\nabla F(x^{r,k}_i;\xi^{r,k}_i)-\nabla f(x^r)\right) \\
            &\quad\quad +(1-\beta )\left(\frac{1}{NK}\sum_{i,\,k}(\nabla F(x^{r};\xi^{r,k}_i)-\nabla F(x^{r-1};\xi^{r,k}_i))-\nabla f(x^r)+\nabla f(x^{r-1})\right) \Bigg\|^2\Bigg] \\
            &\leq 3L^2U_r  + 3\frac{\beta ^2\sigma^2}{NK}+3(1-\beta )^2\frac{L^2}{NK}\expect[\|x^r-x^{r-1}\|^2.
        \end{aligned}
    \end{equation*}
    Therefore, for $r\geq 1$,
    \begin{equation*}
        \begin{aligned}
            \cE_r
            &\leq (1-\beta )\mathcal{E}_{r-1}+\frac{4}{\beta }L^2U_r +\frac{3\beta ^2\sigma^2}{NK}+3(1-\beta )^2\frac{L^2}{NK}\expect[\|x^r-x^{r-1}\|^2] \\
            &\leq (1-\frac{8\beta }{9})\mathcal{E}_{r-1}+\frac{4}{\beta }L^2U_r +\frac{3\beta ^2\sigma^2}{NK}+\frac{6(\gamma L)^2}{NK}\expect[\|\nabla f(x^{r-1})\|^2].
        \end{aligned}
    \end{equation*}
    The last inequality is derived by $\|x^r-x^{r-1}\|^2\leq 2\gamma ^2(\|\nabla f(x^{r-1})\|^2+\|g^{r}-\nabla f(x^{r-1})\|^2)$ and $\gamma L\leq \sqrt{\frac{\beta  NK}{54}}$. 
    Similarly, for $r=0$, we can obtain
    \begin{equation*}
        \mathcal{E}_0\leq (1-\beta )\mathcal{E}_{-1}+\frac{4}{\beta }L^2U_0 +\frac{3\beta ^2\sigma^2}{NK}.
    \end{equation*}
\end{proof}

\begin{lemma}\label{lem:Fedavg_VR_client_drift}
    If $\eta KL\leq \frac{1}{4e}$, the following holds:
    \begin{equation*}
        U_r \leq 4eK^2\Xi_r + 8(\eta K)^2(2(\eta KL)^2+K^{-1})\left(\beta ^2\sigma^2+2L^2\expect[\|x^r-x^{r-1}\|^2]\right).
    \end{equation*}
\end{lemma}
\begin{proof}
    Note that $\zeta_{i}^{r,k}=-\eta (\nabla f_i(x_i^{r,k})+(1-\beta )(g^r-\nabla f_i(x^{r-1}))$. Then we have
    \begin{equation*}
        \begin{aligned}
            \expect [\|\zeta_i^{r,j}-\zeta_i^{r,j-1}\|^2]
            &\leq \eta^2L^2\expect [\|x^{r,j}_i-x^{r,j-1}_i\|^2] \\
            &= \eta^2L^2\left(\expect [\|\zeta^{r,j-1}_i\|^2]+\expect[\var[x^{r,j}_i-x^{r,j-1}_i|\mathcal{F}_i^{r,j-1}]]\right).
        \end{aligned}
    \end{equation*}
    Here we use bias-variance decomposition and $\var [\cdot|\cdot]$ stands for the conditional variance. Since
    \begin{equation*}
        \begin{aligned}
            &\expect[\var[x^{r,j}_i-x^{r,j-1}_i|\mathcal{F}_i^{r,j-1}]] \\
            =&\eta^2\expect\left[\left\|\nabla F(x^{r,j-1}_i;\xi^{r,j-1}_i)-\nabla f_i(x_i^{r,j-1}) - (1-\beta )\left(\nabla F(x^{r-1};\xi^{r,j-1}_i)-\nabla f_i(x^{r-1})\right)\right\|^2\right] \\
            \leq &\eta^2\left(2\beta^2\sigma^2+ 2(1-\beta )^2L^2\expect[\|x^{r-1}-x_i^{r,j-1}\|^2\right),
        \end{aligned}
    \end{equation*}
    then 
    \begin{equation*}
        \begin{aligned}
            &\expect [\|\zeta_i^{r,j}-\zeta_i^{r,j-1}\|^2] \\
            &\qquad \leq \eta^2L^2\left(\expect [\|\zeta^{r,j-1}_i\|^2+2\beta ^2\eta^2\sigma^2+ 2\eta^2(1-\beta )^2L^2\expect[\|x^{r-1}-x_i^{r,j-1}\|^2]\right) \\
            &\qquad \leq \eta^2L^2\left(\expect [\|\zeta^{r,j-1}_i\|^2]+2\beta ^2\eta^2\sigma^2+ 4\eta^2L^2\expect[\|x^{r-1}-x^r\|^2+\|x^{r}-x_i^{r,j-1}\|^2]\right).
        \end{aligned}
    \end{equation*}
    Therefore for any $1\leq j\leq k-1\leq K-2$,
    \begin{equation}\label{eq:Fedavg_VR_client_drift_bound}
        \begin{aligned}
            &\expect \|\zeta^{r,j}_i\|^2\leq (1+\frac{1}{k}) \expect [\|\zeta^{r,j-1}_i\|^2 + (1+k)\expect [\|\zeta^{r,j}_i-\zeta^{r,j-1}_i\|^2] \\
            \leq &\left(1+\frac{2}{k}\right) \expect \|\zeta^{r,j-1}_i\|^2  + (k+1)\eta^2L^2\left(2\beta ^2\eta^2\sigma^2+ 4\eta^2L^2\expect[\|x^{r-1}-x^r\|^2+\|x^{r}-x_i^{r,j-1}\|^2]\right) \\
            \leq &e^2\expect \|\zeta^{r,0}_i\|^2 + 8k^2L^2\eta^4(2\beta ^2\sigma^2+4L^2\expect[\|x^r-x^{r-1}\|^2]) +4e^2k(\eta L)^4\sum_{j'=0}^{j-1}\expect[\|x_i^{r,j'}-x^r\|^2].
        \end{aligned}
    \end{equation}
    Here the second inequality is by $\eta L\leq \frac{1}{K}\leq \frac{1}{k+1}$. The last inequality is by unrolling the recursive bound and using $\left(1+\frac{2}{k}\right)^k\leq e^2$.
    By Lemma \ref{lem:bias-var}, it holds that
    \begin{align}
            &\expect [\|x^{r,k}_i-x^r\|^2]\\
            \leq&2\expect \left[\left\|\sum_{j=0}^{k-1}\zeta^{r,j}_i\right\|^2\right] + 2\sum_{j=0}^{k-1}\expect[\var[x^{r,j+1}_i-x^{r,j}_i|\mathcal{F}_i^{r,j}] ]\\
            \leq& 2k\sum_{j=0}^{k-1}\expect[\|\zeta^{r,j}_i\|^2] + 2\sum_{j=0}^{k-1}\left(2\beta ^2\eta^2\sigma^2+ 4\eta^2L^2\expect[\|x^{r-1}-x^r\|^2+\|x^{r}-x_i^{r,j}\|^2]\right).\label{eqn:nvidnvgsdx}
    \end{align}
    Summing up \eqref{eqn:nvidnvgsdx} over $k=0,\dots,K-1$, using \eqref{eq:Fedavg_VR_client_drift_bound} and $8(\eta L)^2 + 8e^2(\eta KL)^4\leq \frac{1}{2}$ due to the condition on $\eta$,
    we have
    \begin{equation*}
        \frac{1}{2K}\sum_{k=0}^{K-1}\expect [\|x^{r,k}_i-x^r\|^2
        \leq 2eK^2\expect[\|\zeta_i^{r,0}\|^2]+(8(\eta K)^4L^2+4\eta^2K)\left(\beta ^2\sigma^2+2L^2\expect[\|x^r-x^{r-1}\|^2]\right).
    \end{equation*}
    This implies
    \begin{equation*}
        U_r \leq 4eK^2\Xi_r + 8(\eta K)^2(2(\eta KL)^2+K^{-1})\left(\beta ^2\sigma^2+2L^2\expect[\|x^r-x^{r-1}\|^2]\right).
    \end{equation*}
\end{proof}

\begin{lemma}\label{lem:Fedavg_VR_grad_norm}
    If $\gamma L\leq \frac{1}{24}$ and $288e(\eta KL)^2\left(\frac{289}{72}(1-\beta )^2+8e(\gamma \beta  LR)^2\right)\leq \beta ^2$,
    then the following holds:
    \begin{equation*}
        \sum_{r=0}^{R-1}\Xi_r \leq \frac{\beta ^2}{288eK^2L^2}\sum_{r=-1}^{R-2}(\mathcal{E}_{r}+\expect[\|\nabla f(x^{r})\|^2])+4\eta^2\beta ^2eRG_0 .
    \end{equation*}
\end{lemma}
\begin{proof}
    Recall that $\zeta_i^{r,0}=-\eta((1-\beta )(g^r-\nabla f_i(x^{r-1}))+\nabla f_i(x^r))$. Consequently, we have
    \begin{equation*}
        \begin{aligned}
            \|\zeta_i^{r,0}\|^2 
            &\leq 2\eta^2\left((1-\beta )^2\|g^r\|^2+\|\nabla f_i(x^r)-(1-\beta )\nabla f_i(x^{r-1})\|^2\right) \\
            &\leq 2\eta^2(1-\beta )^2(1+2(\gamma L)^2)\|g^r\|^2+4\eta^2\beta ^2\|\nabla f_i(x^r)\|^2 \\
            &\leq \frac{289}{144}\eta^2(1-\beta )^2\|g^r\|^2+4\eta^2\beta ^2\|\nabla f_i(x^r)\|^2.
        \end{aligned}
    \end{equation*}
    Using Young's inequality, we can obtain that for any $q>0$,
    \begin{equation*}
        \begin{aligned}
            \expect [\|\nabla f_i(x^r)\|^2]
            &\leq (1+q)\expect[\|\nabla f_i(x^{r-1})\|^2]+(1+q^{-1})L^2\expect\|x^r-x^{r-1}\|^2\\
            &\leq (1+q)\expect[\|\nabla f_i(x^{r-1})\|^2]+2(1+q^{-1})(\gamma L)^2(\mathcal{E}_{r-1}+\expect[\|\nabla f(x^{r-1})\|^2])\\
            &\leq (1+q)^r\expect[\|\nabla f_i(x^0)\|^2]+\frac{2}{q}(\gamma L)^2\sum_{j=0}^{r-1}(\mathcal{E}_{j}+\expect[\|\nabla f(x^{j})\|^2])(1+q)^{r-j}.
        \end{aligned}
    \end{equation*}
    Taking $q=\frac{1}{r}$ in the above, we have
    \begin{equation*}
        \expect [\|\nabla f_i(x^r)\|^2]\leq e\expect [\|\nabla f_i(x^0)\|^2]+2e(r+1)(\gamma L)^2\sum_{j=0}^{r-1}(\mathcal{E}_{j}+\expect[\|\nabla f(x^{j})\|^2]).
    \end{equation*}
    This inequality holds as well trivially for $r=0$. Therefore, we have
    \begin{equation*}
        \begin{aligned}
            \sum_{r=0}^{R-1}\Xi_r
            &\leq \sum_{r=0}^{R-1} \expect\left[\frac{289}{144}\eta^2(1-\beta )^2 \|g^r\|^2 + 4\eta^2\beta ^2\frac{1}{N}\sum_{i=1}^N\|\nabla f_i(x^r)\|^2\right] \\
            &\leq \sum_{r=0}^{R-1} \frac{289}{72}\eta^2(1-\beta )^2(\mathcal{E}_{r-1}+\expect[\|\nabla f(x^{r-1})\|^2]) \\ 
            &\quad\quad +4\eta^2\beta ^2\sum_{r=0}^{R-1} \left(\frac{e}{N}\sum_i\expect [\|\nabla f_i(x^0)\|^2]+2e(r+1)(\gamma L)^2\sum_{j=0}^{r-1}(\mathcal{E}_{j}+\expect[\|\nabla f(x^{j})\|^2])\right) \\
            &\leq \frac{289}{72}\eta^2(1-\beta )^2\sum_{r=0}^{R-1} (\mathcal{E}_{r-1}+\expect[\|\nabla f(x^{r-1})\|^2]) \\
            &\quad\quad 4\eta^2\beta ^2\left( eRG_0 +2e(\gamma LR)^2\sum_{r=0}^{R-2} (\mathcal{E}_{r}+\expect[\|\nabla f(x^{r})\|^2])\right) \\
            &\leq \frac{\beta ^2}{288eK^2L^2}\sum_{r=-1}^{R-2}(\mathcal{E}_{r}+\expect[\|\nabla f(x^{r})\|^2])+4\eta^2\beta ^2eRG_0 .
        \end{aligned}
    \end{equation*}
    Here the last inequality is due to the upper bound of $\eta$. 
\end{proof}

\begin{thm}\label{app:thm:fedavg_mvr}
    Under Assumption \ref{asp:sample_smooth} and \ref{asp:sgd_var}, if we take $g^0=\frac{1}{NB}\sum_{i=1}^N\sum_{b=1}^{B}\nabla F(x^0;\xi^{b}_i)$ with $\{\xi^{b}_i\}_{b=1}^B\overset{iid}{\sim}\mathcal{D}_i$ and set 
    \begin{equation}\label{eqn:fedavg-mvr-para}
        \begin{aligned}
            &\beta  = \min\left\{c, \left(\frac{NKL^2\Delta^2  }{\sigma^4R^2}\right)^{1/3}\right\} \text{ for any constant }c\in (0,1], \quad \gamma  = \min\left\{\frac{1}{24L}, \sqrt{\frac{\beta  NK}{54L^2}}\right\},\\
            &\eta KL\lesssim \min\left\{\left(\frac{L\Delta }{G_0\gamma LR}\right)^{1/2}, \left(\frac{\beta }{N}\right)^{1/2}, \left(\frac{\beta }{NK}\right)^{1/4}\right\},\quad B=\left\lceil\frac{K}{R\beta ^2}\right\rceil,
        \end{aligned}
    \end{equation}
    \fedavgmvr converges as 
    \begin{equation*}
        \frac{1}{R}\sum_{r=0}^{R-1}\expect[\|\nabla f(x^r)\|^2]\lesssim \left(\frac{L\Delta  \sigma}{NKR}\right)^{2/3}+\frac{L\Delta }{R}.
    \end{equation*}
    {Alternatively, if $B= \Theta(KR)$ and $\beta  = \min\left\{\frac{1}{R}, \left(\frac{NKL^2\Delta^2 }{\sigma^4R^2}\right)^{1/3}\right\}$, then \fedavgmvr converges as 
    \begin{equation*}
        \frac{1}{R}\sum_{r=0}^{R-1}\expect[\|\nabla f(x^r)\|^2]\lesssim \left(\frac{L\Delta  \sigma}{NKR}\right)^{2/3}+ \frac{\sigma^2}{NKR}+\frac{L\Delta }{R}.
    \end{equation*}}
\end{thm}

\begin{proof}
    Combine Lemma \ref{lem:Fedavg_VR_grad_err}, \ref{lem:Fedavg_VR_client_drift} and we have
    \begin{equation*}
        \begin{array}{c}
            \begin{aligned}
                \cE_r
                &\leq  (1-\frac{8\beta }{9})\mathcal{E}_{r-1}+\frac{(6\gamma L)^2}{NK}\expect[\|\nabla f(x^{r-1})\|^2]+\frac{3\beta ^2\sigma^2}{NK} \\
                &\quad\quad+\frac{4}{\beta }L^2\left( 4eK^2\Xi_r + 8(\eta K)^2(2(\eta KL)^2+K^{-1})(\beta ^2\sigma^2+2L^2\expect[\|x^r-x^{r-1}\|^2])\right)
            \end{aligned} \\
            \mathcal{E}_0\leq  (1-\beta )\mathcal{E}_{-1}+\frac{3\beta ^2\sigma^2}{NK}+\frac{4}{\beta }L^2\left(4eK^2\Xi_0 + 8(\eta K)^2(2(\eta KL)^2+K^{-1}))\beta ^2\sigma^2\right)
        \end{array}
    \end{equation*}
    Summing over $r$ from $0$ to $R-1$ and applying Lemma \ref{lem:Fedavg_VR_grad_norm},
    \begin{equation*}
        \begin{aligned}
            &\sum_{r=0}^{R-1}\cE_r\\
            \leq &(1-\frac{8\beta }{9})\sum_{r=-1}^{R-2}\cE_r + \frac{6(\gamma L)^2}{NK}\expect \left[\sum_{r=0}^{R-2}\|\nabla f(x^{r})\|^2
        \right] + \frac{3\beta ^2\sigma^2}{NK}R \\
            &\;\; +\frac{4}{\beta }L^2\left(4eK^2\sum_{r=0}^{R-1}\Xi_r + 8(\eta K)^2(2(\eta KL)^2+\frac{1}{K})\left(R\beta ^2\sigma^2+2L^2\sum_{r=0}^{R-1}\expect[\|x^r-x^{r-1}\|^2]\right)\right) \\
            \leq& (1-\frac{7\beta }{9})\sum_{r=-1}^{R-2}\cE_r + \left(\frac{6(\gamma L)^2}{NK}+\frac{\beta }{9}\right)\expect [\sum_{r=-1}^{R-2}\|\nabla f(x^{r})\|^2]+64\beta (e\eta KL)^2RG_0  \\
            &\quad +\frac{3\beta ^2\sigma^2}{NK}R+32\beta (\eta KL)^2\left(\frac{1}{K}+2(\eta KL)^2\right)\sigma^2R \\
            \leq& (1-\frac{7\beta }{9})\sum_{r=-1}^{R-2}\cE_r + \frac{2\beta }{9}\expect \left[\sum_{r=-1}^{R-2}\|\nabla f(x^{r})\|^2\right]+64\beta (e\eta KL)^2RG_0 +\frac{4\beta ^2\sigma^2}{NK}R.
        \end{aligned}
    \end{equation*}
    Here in the second inequality, we apply
    \begin{equation*}
        \left\{
        \begin{array}{c}
            32\beta (\eta KL)^2(\frac{1}{K}+2(\eta KL)^2) \leq \frac{\beta ^2}{NK}, \\
            \frac{128(\eta KL)^2}{\beta }(\frac{1}{K}+2(\eta KL)^2)(\gamma L)^2 \leq \frac{\beta }{18}, \\
            \gamma {L}\leq\sqrt{\frac{\beta  NK}{54}}.
        \end{array}\right.
    \end{equation*}
    Therefore, we obtain
    \begin{equation*}
        \sum_{r=0}^{R-1}\cE_r\leq \frac{9}{7\beta }\mathcal{E}_{-1} + \frac{2}{7}\expect \left[\sum_{r=-1}^{R-2}\|\nabla f(x^{r})\|^2\right] + \frac{576}{7}(e\eta KL)^2G_0 R+ \frac{36\beta \sigma^2}{7NK}R.
    \end{equation*}
    Combine this inequality with Lemma \ref{lem:onestep} and we get
    \begin{equation*}
        \frac{1}{\gamma }\expect [f(x^R)-f(x^0)]\leq -\frac{1}{7}\sum_{r=0}^{R-1}\expect[\|\nabla f(x^r)\|^2] + \frac{39}{56\beta }\mathcal{E}_{-1}+ \frac{312}{7}(e\eta KL)^2G_0 R+ \frac{39\beta \sigma^2}{14NK}R.
    \end{equation*}
    Finally, for $B=\left\lceil\frac{K}{R\beta ^2}\right\rceil$, noticing that $g^0=\frac{1}{NB}\sum_{i}\sum_{b=1}^B\nabla F(x^0;\xi_i^b)$ implies $\mathcal{E}_{-1}\leq \frac{\sigma^2}{NB}\leq \frac{\beta ^2\sigma^2R}{NK}$ and thus
    \begin{equation*}
        \begin{aligned}
            \frac{1}{R}\sum_{r=0}^{R-1}\expect[\|\nabla f(x^r)\|^2]
            &\lesssim \frac{L\Delta  }{\gamma LR}+\frac{\mathcal{E}_{-1}}{\beta  R} + (\eta KL)^2G_0 + \frac{\beta \sigma^2}{NK} \\
             &\lesssim \frac{L\Delta  }{\gamma LR} + \frac{\beta \sigma^2}{NK} \\
            &\lesssim \frac{L\Delta  }{R}+ \frac{L\Delta  }{\sqrt{\beta  NK} R}+\frac{\beta \sigma^2}{NK} \\
            &\lesssim \frac{L\Delta }{R}+\left(\frac{L\Delta  \sigma}{NKR}\right)^{2/3}
        \end{aligned}
    \end{equation*}
    Similarly, for $B=KR$, $\mathcal{E}_{-1}\leq \frac{\sigma^2}{NB}\leq \frac{\sigma^2}{NKR}$, and we have
    \begin{equation*}
        \begin{aligned}
            \frac{1}{R}\sum_{r=0}^{R-1}\expect[\|\nabla f(x^r)\|^2]
            &\lesssim \frac{L\Delta  }{\gamma LR}+\frac{\mathcal{E}_{-1}}{\beta  R} + (\eta KL)^2G_0 + \frac{\beta \sigma^2}{NK} \\
            &\lesssim \frac{L\Delta  }{\gamma LR} +\frac{\sigma^2}{\beta NKR^2} + \frac{\beta \sigma^2}{NK} \\
            &\lesssim \frac{L\Delta  }{R}+ \frac{L\Delta  }{\sqrt{\beta  NK} R}+\frac{\sigma^2}{\beta NKR^2}+\frac{\beta \sigma^2}{NK} \\
            &\lesssim \frac{L\Delta }{R}+\left(\frac{L\Delta  \sigma}{NKR}\right)^{2/3} + \frac{\sigma^2}{NKR}.
        \end{aligned}
    \end{equation*}
\end{proof}

\newpage
\section{\scaffold with momentum}
\subsection{\scaffoldm}\label{app:scaffoldm}
In this subsection, we present the proofs for the \scaffoldm algorithm.
\begin{lemma}\label{lem:scaffold_grad_err}
    If $\gamma L\leq \frac{\beta }{12}$, the following holds for $r\geq 1$:
    \begin{equation*}
        \cE_r\leq \left(1-\frac{8\beta }{9}\right)\mathcal{E}_{r-1}+\frac{16}{\beta }(\gamma L)^2\expect[\|\nabla f(x^{r-1})\|^2]+\frac{4\beta ^2\sigma^2}{SK}+10\beta  L^2U_r +6\beta ^2\q V_r.
    \end{equation*}
    In addition,
    \begin{equation*}
        \mathcal{E}_0\leq (1-\beta )\mathcal{E}_{-1}+\frac{4\beta ^2\sigma^2}{SK}+8\beta  L^2U_0 +4\beta ^2\q V_0.
    \end{equation*}
\end{lemma}

\begin{proof}
    Note that $\frac{1}{N}\sum_{i=1}^Nc_i^r=c^r$ holds for any $r\geq 0$. Using Lemma \ref{lem:par_sample}, we have
    \begin{equation*}
        \begin{aligned}
            \cE_r
            &= \expect \left[\left\|\nabla f(x^r)-\frac{1}{NK}\sum_{i,\,k}g_i^{r,k}\right\|^2\right] + \frac{N-S}{S(N-1)}\frac{1}{N}\sum_{i=1}^N \expect\left[\left\|\frac{1}{K}\sum_{k}g_i^{r,k}-\frac{1}{NK}\sum_{j,k}g_j^{r,k}\right\|^2\right] \\
            &= \underbrace{\expect \left[\left (1-\beta )(\nabla f(x^r)-g^r)+\beta \left(\frac{1}{NK}\sum_{i,\,k}\nabla F(x^{r,k}_i;\xi^{r,k}_i)-\nabla f(x^r)\right)\right\|^2\right]}_{\Lambda_1} \\
            &\quad +\frac{\beta^2(N-S)}{S(N-1)}\underbrace{\frac{1}{N}\sum_{i=1}^N \expect\left[\left\|\frac{1}{K}\sum_{k}\nabla F(x^{r,k}_i;\xi^{r,k}_i)-\frac{1}{NK}\sum_{j,k}\nabla F(x^{r,k}_j;\xi^{r,k}_j)-(c^r_i-c^r)\right\|^2\right]}_{\Lambda_2}.
        \end{aligned}
    \end{equation*}
    For $r\geq 1$, similar to the proof of Lemma \ref{lem:Fedavg_grad_err}, we have
    \begin{equation*}
        \Lambda_1\leq (1-\beta )\mathcal{E}_{r-1}+\frac{2}{\beta }L^2\expect[\|x^r-x^{r-1}\|^2]+\frac{2\beta ^2\sigma^2}{NK}+4\beta  L^2U_r .
    \end{equation*}
    Besides, by AM-GM inequality and Lemma \ref{lem:bias-var},
    \begin{equation*}
        \begin{aligned}
            \Lambda_2
            &\leq \frac{1}{N}\sum_{i=1}^N \expect\left[\left\|\frac{1}{K}\sum_{k}\nabla F(x^{r,k}_i;\xi^{r,k}_i)-c_i^r\right\|^2\right] \\
            &\leq \frac{2\sigma^2}{K}+\frac{2}{N}\sum_i\expect \left[\left\|\frac{1}{K}\sum_k\nabla f_i(x^{r,k}_i)-c_i^r\right\|^2\right] \\
            &\leq \frac{2\sigma^2}{K}+6(L^2U_r +L^2\expect[\|x^r-x^{r-1}\|^2]+V_r).
        \end{aligned}
    \end{equation*}
    Since $\expect [\|x^r-x^{r-1}\|^2]\leq 2\gamma ^2(\mathcal{E}_{r-1}+\expect [\|\nabla f(x^{r-1})\|^2])$ and $\left(\frac{2}{\beta }+6\beta ^2\q\right)2(\gamma L)^2\leq \frac{16}{\beta }(\gamma L)^2\leq\frac{\beta }{9}$, we have 
    \begin{equation*}
        \cE_r
        \leq \left(1-\frac{8\beta }{9}\right)\mathcal{E}_{r-1}+\frac{16}{\beta }(\gamma L)^2\expect[\|\nabla f(x^{r-1})\|^2]+\frac{4\beta ^2\sigma^2}{SK}+10\beta  L^2U_r +6\beta ^2\q V_r.
    \end{equation*}
    The case for $r=0$ is similar.
\end{proof}

\begin{lemma}\label{lem:scaffold_client_drift}
    If $\gamma L\leq \frac{1}{\sqrt{2\beta }}$ and $ \eta KL\leq \frac{1}{\beta }$, it holds for all $r\geq 1$ that
    \begin{equation*}
        U_r \leq \eta^2K^2 \left(8e(\mathcal{E}_{r-1}+2\expect [\|\nabla f(x^{r-1})\|^2]+\beta ^2V_r)+ \beta ^2\sigma^2(K^{-1}+2(\beta \eta KL)^2))\right).
    \end{equation*}
\end{lemma}

\begin{proof}
    Since $\zeta_i^{r,k}=\expect [x^{r,k+1}_i-x^{r,k}_i|\mathcal{F}_i^{r,k}]=-\eta(\beta \nabla f_i(x_i^{r,k})+(1-\beta )g_r-\beta (c_i^r-c^r))$ and $\var[x^{r,k+1}_i-x^{r,k}_i|\mathcal{F}_i^{r,k}]\leq \beta ^2\eta^2\sigma^2$, with exactly the same procedures of Lemma \ref{lem:Fedavg_client_drift}, we have
    \begin{equation*}
        U_r \leq 2eK^2\Xi_r + K\eta^2\beta ^2\sigma^2(1+2K^3L^2\eta^2\beta ^2).
    \end{equation*}
    Additionally, by AM-GM inequality,
    \begin{equation*}
        \begin{aligned}
            \Xi_r
            &= \frac{\eta^2}{N}\sum_i \expect[\|\beta \nabla f_i(x^r)+(1-\beta )g^r-\beta (c_i^r-c^r)\|^2] \\
            &= \frac{\eta^2}{N}\sum_i \expect\left[\|\beta (\nabla f_i(x^r)-\nabla f_i(x^{r-1}))+(1-\beta )(g^r-\nabla f(x^{r-1})) \right. \\
            &\left. \quad\quad -\beta \left(c_i^r-c^r-\nabla f_i(x^{r-1})+\nabla f(x^{r-1})\right)+\nabla f(x^{r-1})\|^2\right] \\
            &\leq 4\eta^2 \left(\beta ^2L^2\EE[\|x^r-x^{r-1}\|^2]+(1-\beta )^2\mathcal{E}_{r-1}+\beta ^2 V_r+\EE[\|\nabla f(x^{r-1})\|^2]\right) \\
            &\leq 4\eta^2(\mathcal{E}_{r-1}+2\expect [\|\nabla f(x^{r-1})\|^2]+\beta ^2V_r).
        \end{aligned}
    \end{equation*}
    Plug this inequality into the above bound completes the proof.
\end{proof}

\begin{lemma}\label{lem:scaffold_control}
    Under the same conditions of Lemma \ref{lem:scaffold_client_drift}, if $\beta \eta KL\leq \frac{1}{24K^{1/4}}$ and $\eta K\leq \frac{N}{5S}\gamma $, then we have
    \begin{equation*}
        \sum_{r=0}^{R-1} V_r\leq \frac{3N}{S}\left(V_0+\frac{4SR}{NK}\sigma^2+\frac{8N}{S}(\gamma L)^2\sum_{r=-1}^{R-2}(\mathcal{E}_{r}+\expect[\|\nabla f(x^{r})\|^2])\right).
    \end{equation*}
\end{lemma}

\begin{proof}
    Since
    \begin{equation*}
        c_i^{r+1}=\left\{\begin{array}{ll}
             c_i^r & \text{with probability } 1-\frac{S}{N}\\
             \frac{1}{K}\sum_k \nabla F(x^{r,k}_i;\xi^{r,k}_i) & \text{with probability } \frac{S}{N},
        \end{array}\right.
    \end{equation*}
    using Young's inequality repeatedly, we have
    \begin{equation*}
        \begin{aligned}
            V_{r+1}
            &= \left(1-\frac{S}{N}\right)\frac{1}{N}\sum_{i=1}^N\expect[\|c_i^r-\nabla f_i(x^r)\|^2]+\frac{S}{N}\frac{1}{N}\sum_{i=1}^N\expect\left[\left\|\frac{1}{K}\sum_k \nabla F(x^{r,k}_i;\xi^{r,k}_i)-\nabla f_i(x^r)\right\|^2\right] \\
            &\leq \left(1-\frac{S}{N}\right)\frac{1}{N}\sum_{i=1}^N\expect[\|c_i^r-\nabla f_i(x^r)\|^2]+\frac{S}{N}\left(\frac{2\sigma^2}{K}+2L^2U_r \right) \\
            &\leq \left(1-\frac{S}{N}\right)\frac{1}{N}\sum_{i=1}^N\expect\left[\left(1+\frac{S}{2N}\right)\|c_i^r-\nabla f_i(x^{r-1})\|^2+\left(1+\frac{2N}{S}\right)L^2\|x^r-x^{r-1}\|^2\right] \\
            &\quad+ \frac{2S}{N}\left(\frac{\sigma^2}{K}+L^2U_r \right) \\
            &\leq \left(1-\frac{S}{2N}\right)V_r+\frac{2N}{S}L^2\expect[\|x^r-x^{r-1}\|^2]+\frac{2S\sigma^2}{NK} + \frac{2S}{N}L^2U_r .
        \end{aligned}
    \end{equation*}
    Here we apply Lemma \ref{lem:bias-var} to obtain the second inequality. Combine this with Lemma \ref{lem:scaffold_client_drift},
    \begin{equation*}
        \begin{aligned}
            V_{r+1}
            &\leq \left(1-\frac{S}{2N}+16e\frac{S}{N}(\beta \eta KL)^2\right)V_r + 2\sigma^2\left(\frac{S}{NK}+\frac{2S}{N}(\beta \eta KL)^2(K^{-1}+2(\beta \eta KL)^2)\right)\\
            &\quad\quad + \left(\frac{4N}{S}(\gamma L)^2+\frac{32eS}{N}(\eta KL)^2\right)(\mathcal{E}_{r-1}+\expect[\|\nabla f(x^{r-1})\|^2]) \\
            &\leq \left(1-\frac{S}{3N}\right)V_r + \frac{4S}{NK}\sigma^2+\frac{8N}{S}(\gamma L)^2(\mathcal{E}_{r-1}+\expect[\|\nabla f(x^{r-1})\|^2]),
        \end{aligned}
    \end{equation*}
    where we apply the upper bound of $\eta$. Therefore, we finish the proof by summing up over $r$ from $0$ to $R-1$ and rearranging the inequality.
\end{proof}

\begin{thm}
    Under Assumption \ref{asp:smooth} and \ref{asp:sgd_var}, if we take $g^0=0$, $c_i^0=\frac{1}{B}\sum_{b=1}^B\nabla F(x^0;\xi_i^b)$ with $\{\xi^{b}_i\}_{b=1}^B\overset{iid}{\sim}\mathcal{D}_i$, $c^0=\frac{1}{N}\sum_{i=1}^Nc_i^0$ and set
    \begin{equation}\label{eqn:scaffold-m-para}
        \begin{aligned}
            &\gamma=\frac{\beta }{L},\quad \beta  = \min\left\{c, \frac{S}{N^{2/3}}, \sqrt{\frac{L\Delta  SK}{\sigma^2R}}, \sqrt{\frac{L\Delta  S^2}{G_0 N}}\right\},\\
            &\eta KL\lesssim \min\left\{\frac{1}{S^{1/2}}, \frac{1}{\beta  K^{1/4}}, \frac{S^{1/2}}{N}
    \right\},\quad B=\left\lceil \frac{NK}{SR}\right\rceil,
        \end{aligned}
    \end{equation}
    then \scaffoldm converges as 
    \begin{equation*}
        \frac{1}{R}\sum_{r=0}^{R-1}\expect[\|\nabla f(x^r)\|^2]\lesssim \sqrt{\frac{L\Delta  \sigma^2}{SKR}}+\frac{L\Delta  }{R}\left(1+\frac{N^{2/3}}{S}\right).
    \end{equation*}
\end{thm}

\begin{proof}
    By Lemma \ref{lem:scaffold_grad_err}, sum over $r$ from $0$ to $R-1$ and plug Lemma \ref{lem:scaffold_client_drift}, Lemma \ref{lem:scaffold_control} in,
    \begin{equation*}
        \begin{aligned}
            \sum_{r=0}^{R-1}\cE_r
            &\leq \left(1-\frac{8\beta }{9}\right)\sum_{r=-1}^{R-2}\cE_r+\frac{16}{\beta }(\gamma L)^2\sum_{r=0}^{R-2}\expect [\|\nabla f(x^{r})\|^2] \\
            &\quad + \frac{4\beta ^2\sigma^2}{SK}R + 10\beta  L^2\sum_{r=0}^{R-1}U_r +6\beta ^2\q \sum_{r=0}^{R-1}V_r \\
            &\leq \left(1-\frac{8\beta }{9}+80e\beta (\eta KL)^2\right)\sum_{r=-1}^{R-2}\cE_r+(\frac{16}{\beta }(\gamma L)^2+160e\beta (\eta KL)^2)\sum_{r=0}^{R-2}\expect [\|\nabla f(x^{r})\|^2] \\
            &\quad+\beta ^2\sigma^2R\left(\frac{4}{SK}+10(\eta KL)^2(K^{-1}+2(\beta \eta KL)^2)\right) + \\
            &\quad +\beta ^2\left(6\q + 80e\beta (\eta KL)^2\right)\sum_{r=0}^{R-1}V_r \\
            &\leq \left(1-\frac{7\beta }{9}\right)\sum_{r=-1}^{R-2}\cE_r + \left(\frac{16}{\beta }(\gamma L)^2+\frac{\beta }{9}\right)\sum_{r=0}^{R-2}\expect [\|\nabla f(x^{r})\|^2] + \frac{80\beta ^2\sigma^2}{SK}R+\frac{30\beta ^2N}{S^2}V_0.
        \end{aligned}
    \end{equation*}
    Here the coefficients in the last inequality are derived by the following bounds:
    \begin{equation*}
        \left\{
        \begin{array}{c}
            160e\beta  (\eta KL)^2+24(\frac{\beta \gamma LN}{S})^2\left(6\q+80e\beta  (\eta KL)^2\right) \leq \frac{\beta }{9}, \\
            \\
            10(\eta KL)^2(K^{-1}+2(\beta \eta KL)^2) + 960e\beta  K^{-1}(\eta KL)^2 \leq \frac{4}{SK}, \\
            \\
            80e\beta  (\eta KL)^2 \leq \frac{4}{S},
        \end{array}
        \right.
    \end{equation*}
    which can be guaranteed by
    \begin{equation*}
        \left\{
        \begin{array}{c}
            \gamma L\lesssim \frac{S^{3/2}}{\beta ^{1/2} N}, \\
            \\
            \eta KL \lesssim \frac{1}{S^{1/2}}.
        \end{array}
        \right.
    \end{equation*}
    Therefore, 
    \begin{equation*}
        \sum_{r=0}^{R-1}\cE_r\leq \frac{9}{7\beta }\mathcal{E}_{-1} + \frac{2}{7}\expect \left[\sum_{r=-1}^{R-2}\|\nabla f(x^{r})\|^2\right] + \frac{270\beta  N}{7S^2}V_0+ \frac{720\beta \sigma^2}{7SK}R.
    \end{equation*}
    Combining this inequality with Lemma \ref{lem:onestep}, we obtain
    \begin{equation*}
        \frac{1}{\gamma }\expect [f(x^R)-f(x^0)]\leq -\frac{1}{7}\sum_{r=0}^{R-1}\expect[\|\nabla f(x^r)\|^2] + \frac{39}{56\beta }\mathcal{E}_{-1}+ \frac{585\beta  N}{28S^2}V_0+ \frac{390\beta \sigma^2}{7SK}R.
    \end{equation*}
    Finally, noticing that $g^0=0$ implies $\mathcal{E}_{-1}\leq 2L\Delta$ and $c_i=\frac{1}{B}\sum_{b}\nabla F(x^0;\xi_i^b)$ implies $V_0\leq \frac{\sigma^2}{B}\leq \frac{SR\sigma^2}{NK}$, we reach
    \begin{equation*}
        \begin{aligned}
            \frac{1}{R}\sum_{r=0}^{R-1}\expect[\|\nabla f(x^r)\|^2]
            &\lesssim \frac{L\Delta  }{\gamma LR}+\frac{\mathcal{E}_{-1}}{\beta  R} + \frac{\beta  N}{S^2R}V_0+ \frac{\beta \sigma^2}{SK} \\
            &\lesssim \frac{L\Delta  }{\beta  R}+\frac{L\Delta  }{S^{3/2}R}N\beta ^{1/2}+\frac{\beta \sigma^2}{SK} \\
            &\lesssim \frac{L\Delta  }{R}\left(1+\frac{N^{2/3}}{S}\right)+\sqrt{\frac{L\Delta  \sigma^2}{SKR}}.
        \end{aligned}
    \end{equation*}
\end{proof}

\subsection{\scaffoldmvr}\label{app:scaffoldmvr}
In this subsection, we present the proofs for the \scaffoldmvr algorithm, shown as in Algorithm \ref{alg:scaffold_mom_vr}.
\begin{algorithm}[ht]
    \caption{\scaffoldmvr: {\sc SCAFFOLD} with variance-reduced momentum}
    \label{alg:scaffold_mom_vr}
    \begin{algorithmic}
        \REQUIRE{initial model $x^{-1}=x_0$, gradient estimator $g^0$, control variables $\{c_i^0\}_{i=1}^N$ and $c^0$, local learning rate $\eta$, global learning rate $\gamma $, momentum $\beta$}
        \FOR{$r=0, \cdots, R-1$}
            \STATE{Uniformly sample clients $\mathcal{S}_r\subseteq \{1, \cdots, N\}$ with $|\cS_r|=S$}
            \FOR{each client $i\in \cS_r$ in parallel}
                \STATE{Initialize local model $x^{r,0}_i=x^r$}
                \FOR{$k=0, \cdots, K-1$}
                    \STATE{
                            \colorbox{BPink}{Compute  $g^{r,k}_i=  \nabla F(x^{r,k}_i;\xi^{r,k}_i)-\beta (c_i^r-c^r) + (1-\beta )(g^r-\nabla F(x^{r-1};\xi_i^{r,k}))$}\\
                            \;Update local model $x^{r,k+1}_i=x^{r,k}_i-\eta g^{r,k}_i$
                    }
                \ENDFOR 
                \STATE{Update control variable $c^{r+1}_i:= \frac{1}{K}\sum_k \nabla F(x^{r,k}_i;\xi^{r,k}_i)$ (for $i\notin \cS_r$, $c_i^{r+1}=c_i^r$)}
            \ENDFOR
            \STATE Aggregate local updates $g^{r+1}= \frac{1}{\eta S K}\sum_{i\in\cS_r}\left(x^r - x^{r,K}_i\right)$\\
            \STATE Update global model $x^{r+1}= x^r-\gamma g^{r+1}$\\
            \STATE Update control variable $c^{r+1}= c^{r}+ \frac{1}{N}\sum_{i\in\mathcal{S}_r} (c^{r+1}_i-c^{r}_i)$
        \ENDFOR
    \end{algorithmic}
\end{algorithm}

\begin{lemma}\label{lem:scaffold_VR_grad_err}
    If $\gamma L \leq \sqrt{\frac{\beta  S}{126}}$, then the following holds for $r\geq 1$:
    \begin{equation*}
        \begin{aligned}
            \cE_r
            &\leq (1-\frac{8\beta }{9})\mathcal{E}_{r-1}+\frac{14(\gamma L)^2}{S}\expect [\|\nabla f(x^{r-1})\|^2]+\frac{8}{\beta }L^2U_r +\frac{7\beta ^2\sigma^2}{SK}+\frac{4(N-S)}{S(N-1)} \beta ^2V_r.
        \end{aligned}
    \end{equation*}
    In addition,
    \begin{equation*}
        \mathcal{E}_0\leq (1-\beta )\mathcal{E}_{-1}+\frac{8}{\beta }L^2U_0 +\frac{7\beta ^2\sigma^2}{SK}+\frac{4(N-S)}{S(N-1)} \beta ^2V_0.
    \end{equation*}
\end{lemma}

\begin{proof}
    By Lemma \ref{lem:bias-var}, we have
    \begin{equation*}
        \begin{aligned}
            \cE_r
            &\leq \underbrace{\expect \left[\left\|\nabla f(x^r)-\frac{1}{NK}\sum_{i,\,k}\left[\nabla F(x^{r,k}_i;\xi_i^{r,k})+(1-\beta )(g^r-\nabla F(x^{r-1}; \xi_i^{r,k}))\right] \right\|^2\right]}_{\Lambda_1} \\
            &\quad\quad + \q \underbrace{\frac{1}{N}\sum_{i=1}^N \expect\left[\left\|\frac{1}{K}\sum_{k}\left[\nabla F(x^{r,k}_i;\xi_i^{r,k})-(1-\beta )\nabla F(x^{r-1};\xi_i^{r,k})\right]-\beta  c_i^r\right\|^2\right]}_{\Lambda_2}.
        \end{aligned}
    \end{equation*}
    Applying the same derivation as  Lemma \ref{lem:Fedavg_VR_grad_err}, we can show that
    \begin{equation*}
        \Lambda_1\leq (1-\beta )\mathcal{E}_{r-1}+\frac{4}{\beta }L^2U_r +3\frac{\beta ^2\sigma^2}{NK}+3(1-\beta )^2\frac{L^2}{NK}\expect[\|x^r-x^{r-1}\|^2].
    \end{equation*}
    Additionally, by the AM-GM inequality,
    \begin{equation*}
        \begin{aligned}
            \Lambda_2
            &\leq \frac{1}{N}\sum_{i=1}^N 4\expect\left[ \left\|\frac{1}{K}\sum_{k}\nabla F(x^{r,k}_i;\xi_i^{r,k})-\nabla F(x^r;\xi_i^{r,k})\right\|^2 \right. \\
            &\quad +\beta ^2\left\|\frac{1}{K}\sum_{k}\nabla F(x^{r};\xi_i^{r,k})-\nabla f_i(x^r)\right\|^2 +\beta ^2\|\nabla f_i(x^{r-1})-c^r_i\|^2 \\
            &\left. \quad\quad +\left\|\beta  (\nabla f_i(x^r)-\nabla f_i(x^{r-1}))+\frac{1-\beta }{K}\sum_{k}\nabla F(x^{r};\xi_i^{r,k})-\nabla F(x^{r-1};\xi_i^{r,k})\right\|^2 \right] \\
            &\leq 4\left(L^2U_r +\frac{\beta ^2\sigma^2}{K}+\beta ^2V_r+L^2\expect[\|x^r-x^{r-1}\|^2]\right).
        \end{aligned}
    \end{equation*}
    Further notice that for $r\geq 1$, $\expect [\|x^r-x^{r-1}\|^2\leq 2\gamma ^2(\mathcal{E}_{r-1}+\expect[\|\nabla f(x^{r-1})\|^2])$ and 
    \begin{equation*}
        (\gamma L)^2(\frac{8(N-S)}{S(N-1)} + \frac{6(1-\beta )^2}{NK})\leq \frac{14(\gamma L)^2}{S} \leq \frac{\beta }{9}.
    \end{equation*}
    Hence we obtain
    \begin{equation*}
        \begin{aligned}
            \cE_r
            &\leq (1-\frac{8\beta }{9})\mathcal{E}_{r-1}+\frac{14(\gamma L)^2}{S}\expect [\|\nabla f(x^{r-1})\|^2+\frac{8}{\beta }L^2U_r +\frac{7\beta ^2\sigma^2}{SK}+\frac{4(N-S)}{S(N-1)} \beta ^2V_r.
        \end{aligned}
    \end{equation*}
    The case for $r=0$ can be established similarly.
\end{proof}

\begin{lemma}\label{lem:scaffold_VR_control}
    If $\eta KL\leq \frac{1}{4e}$, $\eta K\leq \frac{\gamma N}{10S}$, and $\gamma L\leq \frac{1}{24}$, then it holds that
    \begin{equation*}
        \sum_{r=0}^{R-1}V_{r}\leq \frac{3N}{S}\left(V_0+\frac{4SR}{NK}\sigma^2+\frac{6N}{S}(\gamma L)^2\sum_{r=-1}^{R-2}(\mathcal{E}_{r}+\expect[\|\nabla f(x^{r})\|^2])\right).
    \end{equation*}
\end{lemma}
\begin{proof}
    Note that $\zeta_i^{r,k}=-\eta(\nabla f_i(x_i^{r,k})+(1-\beta )(g^r-\nabla f_i(x^{r-1}))-\beta (c_i^r-c^r))$, with the same procedures in Lemma \ref{lem:Fedavg_VR_client_drift}, we have
    \begin{equation*}
        U_r \leq 4eK^2\Xi_r + 8(\eta K)^2(2(\eta KL)^2+K^{-1})\left(\beta ^2\sigma^2+2L^2\expect[\|x^r-x^{r-1}\|^2]\right).
    \end{equation*}
    Additionally, by the AM-GM inequality,
    \begin{equation*}
        \begin{aligned}
            \Xi_r
            &= \frac{\eta^2}{N}\sum_i \expect[\|\nabla f_i(x^r)+(1-\beta )(g^r-\nabla f_i(x^{r-1})-\beta (c_i^r-c^r)\|^2] \\
            &= \frac{\eta^2}{N}\sum_i \expect\left[\left\|(\nabla f_i(x^r)-\nabla f_i(x^{r-1}))+(1-\beta )(g^r-\nabla f(x^{r-1})) \right.\right.\\
            &\left. \left.\quad\quad -\beta \left(c_i^r-c^r-\nabla f_i(x^{r-1})+\nabla f(x^{r-1})\right)+\nabla f(x^{r-1})\right\|^2 \right] \\
            &\leq 4\eta^2 \expect \left[L^2\|x^r-x^{r-1}\|^2+(1-\beta )^2\mathcal{E}_{r-1}+\beta ^2 V_r+\|\nabla f(x^{r-1})\|^2\right] \\
            &\leq 8\eta^2(\mathcal{E}_{r-1}+\expect [\|\nabla f(x^{r-1})\|^2]+\beta ^2V_r).
        \end{aligned}
    \end{equation*}
    Hence, by applying $32(2(\eta KL)^2+K^{-1})(\gamma L)^2\leq 96(\gamma L)^2\leq 2$, we obtain
    \begin{equation}\label{eq:scaffold_VR_client_drift}
        \begin{aligned}
            U_r 
            &\leq 32e(\eta K)^2(\mathcal{E}_{r-1}+\expect [\|\nabla f(x^{r-1})\|^2]+\beta ^2V_r) \\
            &\quad\quad +8(\eta K)^2(2(\eta KL)^2+K^{-1})\left(\beta ^2\sigma^2+2L^2\expect[\|x^r-x^{r-1}\|^2]\right) \\
            &\leq 90(\eta K)^2(\mathcal{E}_{r-1}+\expect [\|\nabla f(x^{r-1})\|^2]+\beta ^2V_r) + 8(\beta \eta K)^2(2(\eta KL)^2+K^{-1})\sigma^2.
        \end{aligned}
    \end{equation}
    Also, similar to Lemma \ref{lem:scaffold_control}, it still holds that
    \begin{equation*}
        V_{r+1}\leq \left(1-\frac{S}{2N}\right)V_r+\frac{2N}{S}L^2\expect[\|x^r-x^{r-1}\|^2+\frac{2S\sigma^2}{NK} + \frac{2S}{N}L^2U_r .
    \end{equation*}
    Combine this with the upper bound of $U_r $,
    \begin{equation*}
        \begin{aligned}
            &V_{r+1}\\
            \leq &\left(1-\frac{S}{2N}+\frac{180(\beta \eta KL)^2S}{N}\right)V_r+\left(\frac{4N(\gamma L)^2}{S}+\frac{180(\eta KL)^2S}{N}\right)(\mathcal{E}_{r-1}+\expect [\|\nabla f(x^{r-1})\|^2]) \\
            &\quad + \sigma^2\left(\frac{2S}{NK}+8(\beta \eta KL)^2(2(\eta KL)^2+K^{-1})\right) \\
            \leq &\left(1-\frac{S}{3N}\right)V_r+\frac{6N(\gamma L)^2}{S}(\mathcal{E}_{r-1}+\expect [\|\nabla f(x^{r-1})\|^2])+\frac{4S\sigma^2}{NK},
        \end{aligned}
    \end{equation*}
    where we apply the upper bound of $\eta$ in the last inequality. Iterating the above inequality completes the proof.
\end{proof}

\begin{thm}\label{app:thm:scaffold_mvr}
       Under Assumption \ref{asp:sample_smooth} and \ref{asp:sgd_var}, if we take $c_i^0=\frac{1}{B}\sum_{b=1}^{B}\nabla F(x^0;\xi^{b}_i)$ with $\{\xi_i^b\}_{b=1}^B\overset{iid}{\sim}\mathcal{D}_i$, $g^0=c^0=\frac{1}{N}\sum_{i=1}^Nc_i^0$  and set 
       \begin{equation}\label{eqn:scaffold-mvr-para}
           \begin{aligned}
               &\gamma  = \min\left\{\frac{1}{L}, \frac{\sqrt{\beta  S}}{L}\right\},\quad \beta  = \min\left\{\frac{S}{N}, \left(\frac{KL\Delta  }{\sigma^2R}\right)^{2/3}S^{1/3}\right\},\\
               &\eta KL\lesssim \min\left\{\left(\frac{\beta }{S}\right)^{1/2}, \left(\frac{\beta }{SK}\right)^{1/4}\right\} ,\quad B =\left\lceil \max\left\{\frac{SK}{NR\beta ^2}, \frac{NK}{SR}\right\}\right\rceil,
           \end{aligned}
       \end{equation}
    \scaffoldmvr converges as
    \begin{equation*}
        \frac{1}{R}\sum_{r=0}^{R-1}\expect[\|\nabla f(x^r)\|^2]\lesssim \left(\frac{L\Delta  \sigma}{S\sqrt{K}R}\right)^{2/3}+\frac{L\Delta  }{R}\left(1+\frac{N^{1/2}}{S}\right).
    \end{equation*}
    {Alternatively, if $R\gtrsim \frac{N}{S}$ and $\beta  = \min\left\{\frac{1}{R}, \left(\frac{KL\Delta  }{\sigma^2R}\right)^{2/3}S^{1/3}\right\}$, $B=\Theta(\frac{SKR}{N})$, \scaffoldmvr converges as
    \begin{equation*}
        \frac{1}{R}\sum_{r=0}^{R-1}\expect[\|\nabla f(x^r)\|^2]\lesssim \left(\frac{L\Delta  \sigma}{S\sqrt{K}R}\right)^{2/3}+\frac{L\Delta  }{R}\left(1+\frac{N^{1/2}}{S} + \frac{\sigma^2}{SKR}\right).
    \end{equation*}}
\end{thm}

\begin{proof}
    By Lemma \ref{lem:scaffold_VR_grad_err}, sum over $r$ from $0$ to $R-1$ and plug \eqref{eq:scaffold_VR_client_drift}, Lemma \ref{lem:scaffold_VR_control} in,
    \begin{equation*}
        \begin{aligned}
            \sum_{r=0}^{R-1}\cE_r
            &\leq (1-\frac{8\beta }{9})\sum_{r=-1}^{R-2}\cE_r+\frac{14(\gamma L)^2}{S}\sum_{r=0}^{R-2}\expect [\|\nabla f(x^{r})\|^2] + \frac{7\beta ^2\sigma^2}{SK}R \\
            &\quad + \frac{8}{\beta } L^2\sum_{r=0}^{R-1}U_r +4\beta^2\q \sum_{r=0}^{R-1}V_r \\
            &\leq (1-\frac{8\beta }{9}+720\frac{(\eta KL)^2}{\beta }) \sum_{r=-1}^{R-2}\cE_r+(\frac{14(\gamma L)^2}{S}+720\frac{(\eta KL)^2}{\beta })\sum_{r=0}^{R-2}\expect [\|\nabla f(x^{r})\|^2] \\
            &\quad +\beta ^2\sigma^2R\left(\frac{7}{SK}+\frac{64(\eta KL)^2}{\beta }(K^{-1}+2(\eta KL)^2)\right) \\
            &\quad\quad+ \beta ^2\left(\frac{4(N-S)}{S(N-1)} + 720\frac{(\eta KL)^2}{\beta }\right)\sum_{r=0}^{R-1}V_r \\
            &\leq (1-\frac{7\beta }{9})\sum_{r=-1}^{R-2}\cE_r + (\frac{14(\gamma L)^2}{S}+\frac{\beta }{9})\sum_{r=0}^{R-2}\expect [\|\nabla f(x^{r})\|^2] + 60\frac{\beta ^2\sigma^2}{SK}R+15\frac{\beta ^2N}{S^2}V_0.
        \end{aligned}
    \end{equation*}
    Here the coefficients in the last inequality are derived by the following bounds:
    \begin{equation*}
        \left\{
        \begin{array}{c}
            720\frac{(\eta KL)^2}{\beta }+18(\frac{\beta \gamma LN}{S})^2\left(4\q + 720\frac{(\eta KL)^2}{\beta }\right) \leq \frac{\beta }{9}, \\
            \\
            64\frac{(\eta KL)^2}{\beta }(K^{-1}+2(\eta KL)^2) + 8640\frac{(\eta KL)^2}{\beta  K} \leq \frac{5}{SK}, \\
            \\
            720\frac{(\eta KL)^2}{\beta } \leq \frac{1}{S},
        \end{array}
        \right.
    \end{equation*}
    which can be guaranteed by
    \begin{equation*}
        \left\{
        \begin{array}{c}
            \gamma L\lesssim \frac{S^{3/2}}{\beta ^{1/2} N}, \\
            \\
            \eta KL \lesssim \min\{\sqrt{\frac{\beta }{S}}, (\frac{\beta }{SK})^{1/4}\}.
        \end{array}
        \right.
    \end{equation*}
    Therefore, it holds that
    \begin{equation*}
        \sum_{r=0}^{R-1}\cE_r\leq \frac{9}{7\beta }\mathcal{E}_{-1} + \frac{2}{7}\expect \left[\sum_{r=-1}^{R-2}\|\nabla f(x^{r})\|^2\right] + \frac{135\beta  N}{7S^2}V_0+ \frac{540\beta \sigma^2}{7SK}R.
    \end{equation*}
    Combine this inequality with Lemma \ref{lem:onestep} and we get
    \begin{equation*}
        \frac{1}{\gamma }\expect [f(x^R)-f(x^0)]\leq -\frac{1}{7}\sum_{r=0}^{R-1}\expect[\|\nabla f(x^r)\|^2] + \frac{39}{56\beta }\mathcal{E}_{-1}+ \frac{585\beta  N}{56S^2}V_0+ \frac{585\beta \sigma^2}{14SK}R.
    \end{equation*}
    Finally, for $B=\left\lceil \max\left\{\frac{SK}{NR\beta ^2}, \frac{NK}{SR}\right\}\right\rceil$, noticing that $g^0=\frac{1}{NB}\sum_{i,b}\nabla F(x^0;\xi_i^b)$ implies $\mathcal{E}_{-1}\leq \frac{\sigma^2}{NB}\leq \frac{\beta ^2\sigma^2R}{SK}$ and $c_i=\frac{1}{B}\sum_{b}\nabla F(x^0;\xi_i^b)$ implies $V_0\leq\frac{\sigma^2}{B}\leq \frac{SR\sigma^2}{NK}$, we reach
    \begin{equation*}
        \begin{aligned}
            \frac{1}{R}\sum_{r=0}^{R-1}\expect[\|\nabla f(x^r)\|^2]
            &\lesssim \frac{L\Delta  }{\gamma LR}+\frac{\mathcal{E}_{-1}}{\beta  R} + \frac{\beta  N}{S^2R}V_0+ \frac{\beta \sigma^2}{SK} \\
            &\lesssim \frac{L\Delta  }{R}+\frac{L\Delta  }{(\beta  S)^{1/2}R}+\frac{L\Delta  }{S^{3/2}R}N\beta ^{1/2}+\frac{\beta \sigma^2}{SK} \\
            &\lesssim \frac{L\Delta  }{R}\left(1+\frac{N^{1/2}}{S}\right)+\left(\frac{L\Delta  \sigma}{S\sqrt{K}R}\right)^{2/3}.
        \end{aligned}
    \end{equation*}
    Similarly, for $B=\frac{SKR}{N}$ and $R\gtrsim \frac{N}{S}$, $\mathcal{E}_{-1}\leq \frac{\sigma^2}{NB}\leq \frac{\sigma^2}{SKR}$, $V_0\leq\frac{\sigma^2}{B}\leq \frac{N\sigma^2}{SKR}$ and thus we have
    \begin{equation*}
        \begin{aligned}
            \frac{1}{R}\sum_{r=0}^{R-1}\expect[\|\nabla f(x^r)\|^2]
            &\lesssim \frac{L\Delta  }{\gamma LR}+\frac{\mathcal{E}_{-1}}{\beta  R} + \frac{\beta  N}{S^2R}V_0+ \frac{\beta \sigma^2}{SK} \\
            &\lesssim \frac{L\Delta  }{R}+\frac{L\Delta  }{(\beta  S)^{1/2}R}+\frac{L\Delta  }{S^{3/2}R}N\beta ^{1/2}+ \frac{\sigma^2}{\beta SKR^2}+ \frac{\beta \sigma^2}{SK} \\
            &\lesssim \frac{L\Delta  }{R}\left(1+\frac{N^{1/2}}{S}\right)+\left(\frac{L\Delta  \sigma}{S\sqrt{K}R}\right)^{2/3} + \frac{\sigma^2}{SKR}.
        \end{aligned}
    \end{equation*}
\end{proof}

\newpage
\section{Implementation details \& more experiments}\label{app_sec:exp}

\subsection{Training setup of MLP}
We generate non-iid data for the clients, we sample label ratios from the Dirichlet distribution \citep{hsu2019measuring} with a parameter of $0.5$ for the full participation setting and $0.2$ for the partial participation setting. Our experimental setup involves $N=10$ clients and $K=32$ local updates. The weight decay is set as $10^{-4}$. The global learning rate is fixed as $\gamma=\eta K$ for all the algorithms, and we perform a grid search for the local learning rate $\eta$ in values $\{0.005, 0.01, 0.05, 0.1, 0.5\}$. Similarly, we search for the momentum parameter $\beta$ in values $\{0.1,0.2,0.5,0.8\}$.

\subsection{Training setup of ResNet18}\label{app_subsec:resnet_high}

We generate non-iid data by setting the parameter of Dirichlet distribution as $0.1$, which implies higher heterogeneity. The experiment involves $N=10$ clients and $K=16$ local updates. We set $S=2$ in the partial participation setting. The local learning is fixed as $\hat{\eta}=0.001$ and global learning rate is $\hat{\gamma}=\hat{\eta}K$. The momentum parameter is $\beta=0.1$ and batchsize is $128$.

\paragraph{Reparameterizing momentum. }The update rule of \fedavgm in \eqref{eq:fedavg_m} is equivalent to, with a transformation of hyperparameters $\hat{g}_i^{r,k} := g_i^{r,k} / \beta,\ \hat{g}^r := g^r / \beta,\ \hat{\eta}:= \beta\eta,\ \hat{\gamma}:= \beta\gamma$,
\begin{align}
    \hat{g}_i^{r,k}=\nabla F(x^{r,k}_i;\xi^{r,k}_i) + (1-\beta)\hat{g}^r,
\end{align}
\begin{align}
    x_i^{r,k+1}=x_i^{r,k}-\hat{\eta}\ \hat{g}_i^{r,k},\quad x^{r+1}=x^r-\hat{\gamma}\hat{g}^{r+1}.
\end{align}
This is typically used in the current Pytorch implementation of momentum-based methods. When $\beta=1$, it still reduces to vanilla \fedavg.
Similarly, in \scaffoldm, the update rule \eqref{eq:scaffold_m} is equivalent to 
\begin{align}
    \hat{g}^{r,k}_i= (\nabla F(x^{r,k}_i;\xi^{r,k}_i)-c_i^r+c^r) + (1-\beta )\hat{g}^r
\end{align}
In all the experiments on ResNet18, we implement our proposed \fedavgm and \scaffoldm with this reparameterization.


\subsection{More experiments of ResNet18}
We conduct more algorithms under mild heterogeneity with the parameter of Dirichlet distribution being $0.5$.  We set $S=5$ in the partial participation setting. Other hyperparameters are the same as described in Section \ref{app_subsec:resnet_high}. We plot the evolution of test loss in Figures \ref{fig:full_low_loss} and \ref{fig:par_low_loss}, respectively.  Again, we observe that our proposed \fedavgm and \scaffoldm outperform the vanilla \fedavg and \scaffold with evident margins. We also evaluate VR methods in terms of test accuracy in the context of full client participation. The results are presented in Figure \ref{fig:vr_full_low}, which demonstrates the advantage of our proposed VR methods over the prior methods.

\begin{figure}[ht]
    \centering
    \hspace{-3mm}
    \subfigure[Full participation]{\includegraphics[clip,width=0.45\textwidth]{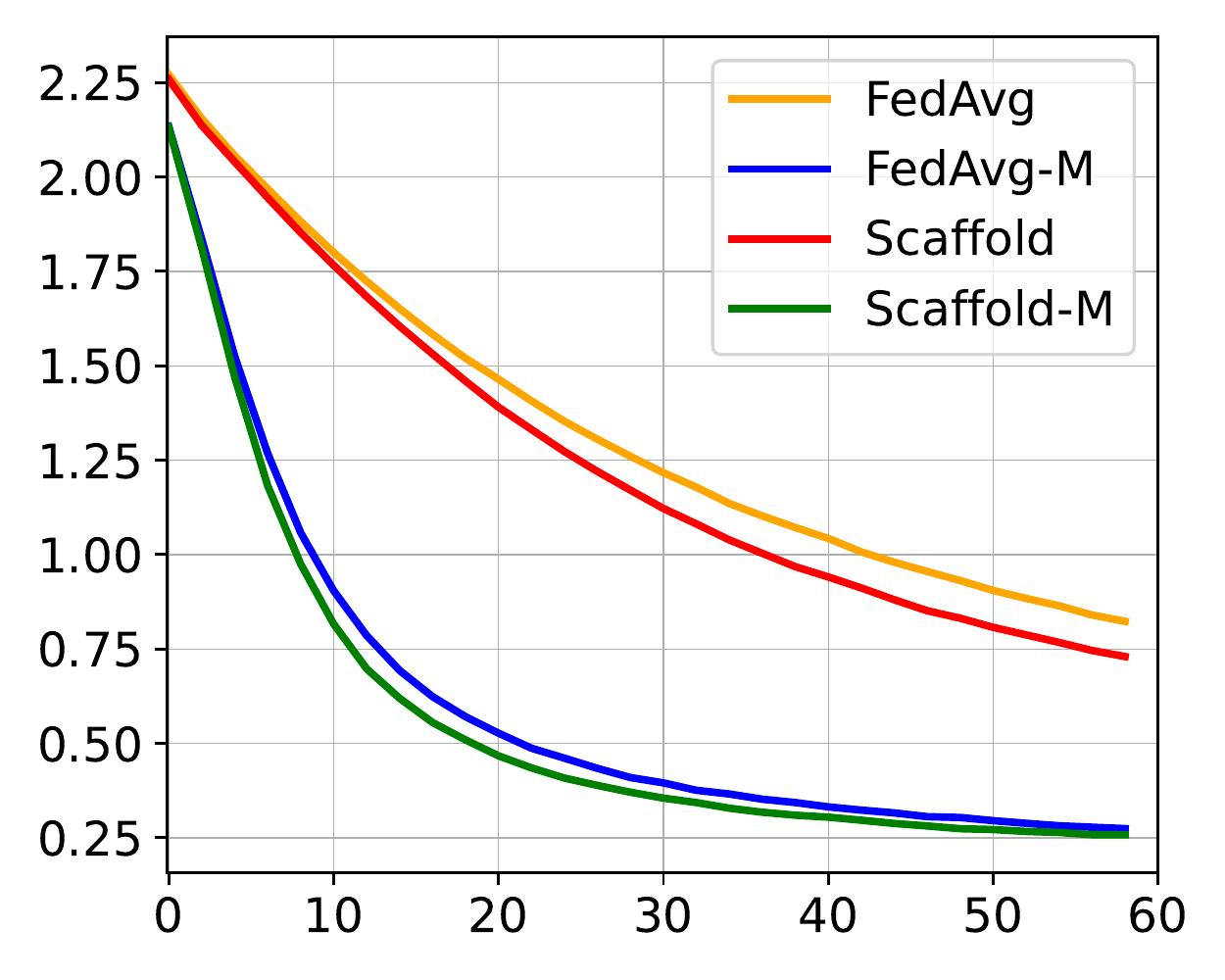} \label{fig:full_low_loss}} \hspace{-3mm}
    \subfigure[Partial participation]{\includegraphics[clip,width=0.45\textwidth]{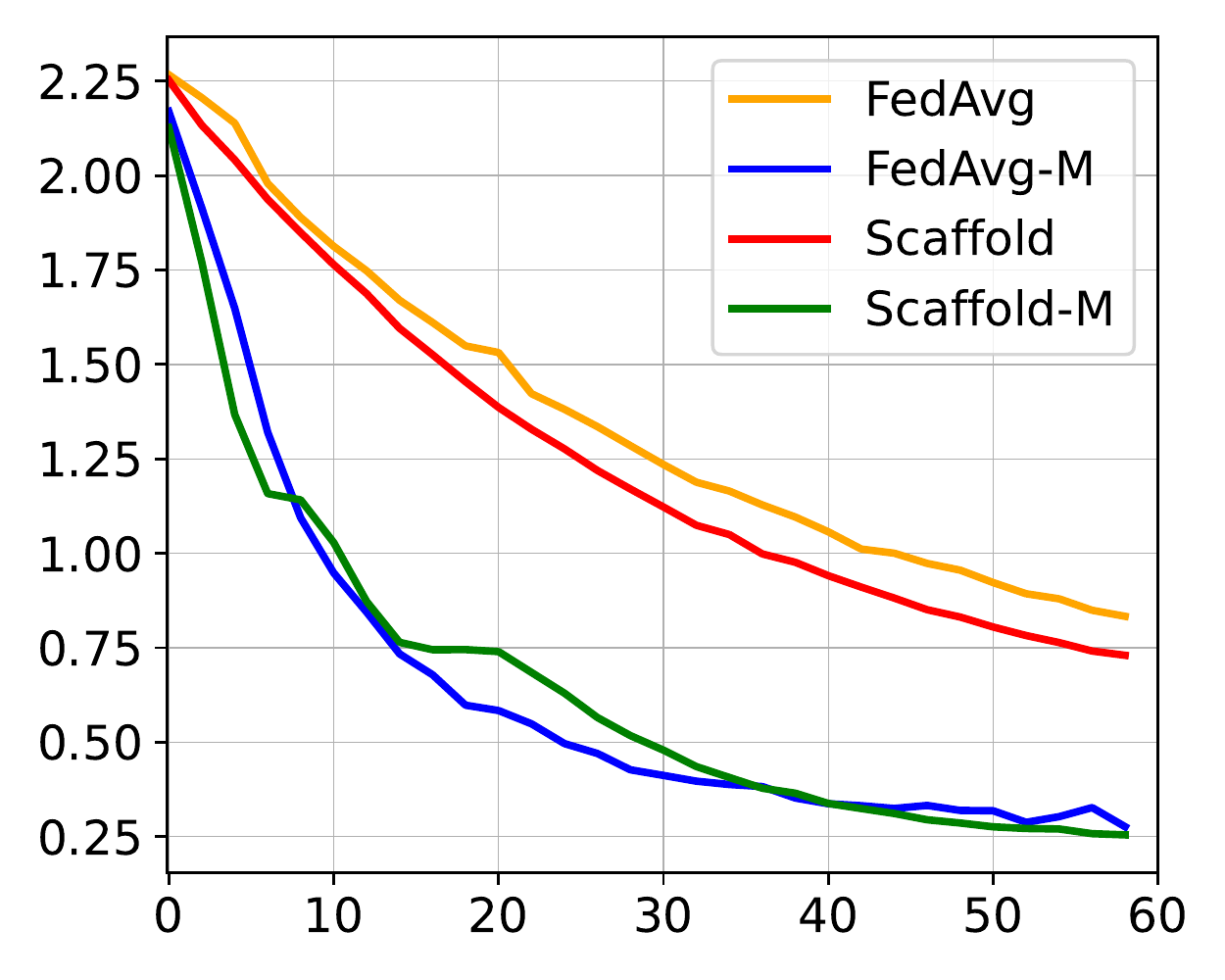} \label{fig:par_low_loss}}
    \hspace{-3mm}
    \caption{Test loss of ResNet18  versus the number of communication rounds}
\end{figure}

\begin{figure}
    \centering
    \hspace{-3mm}
    \includegraphics[clip,width=0.45\textwidth]{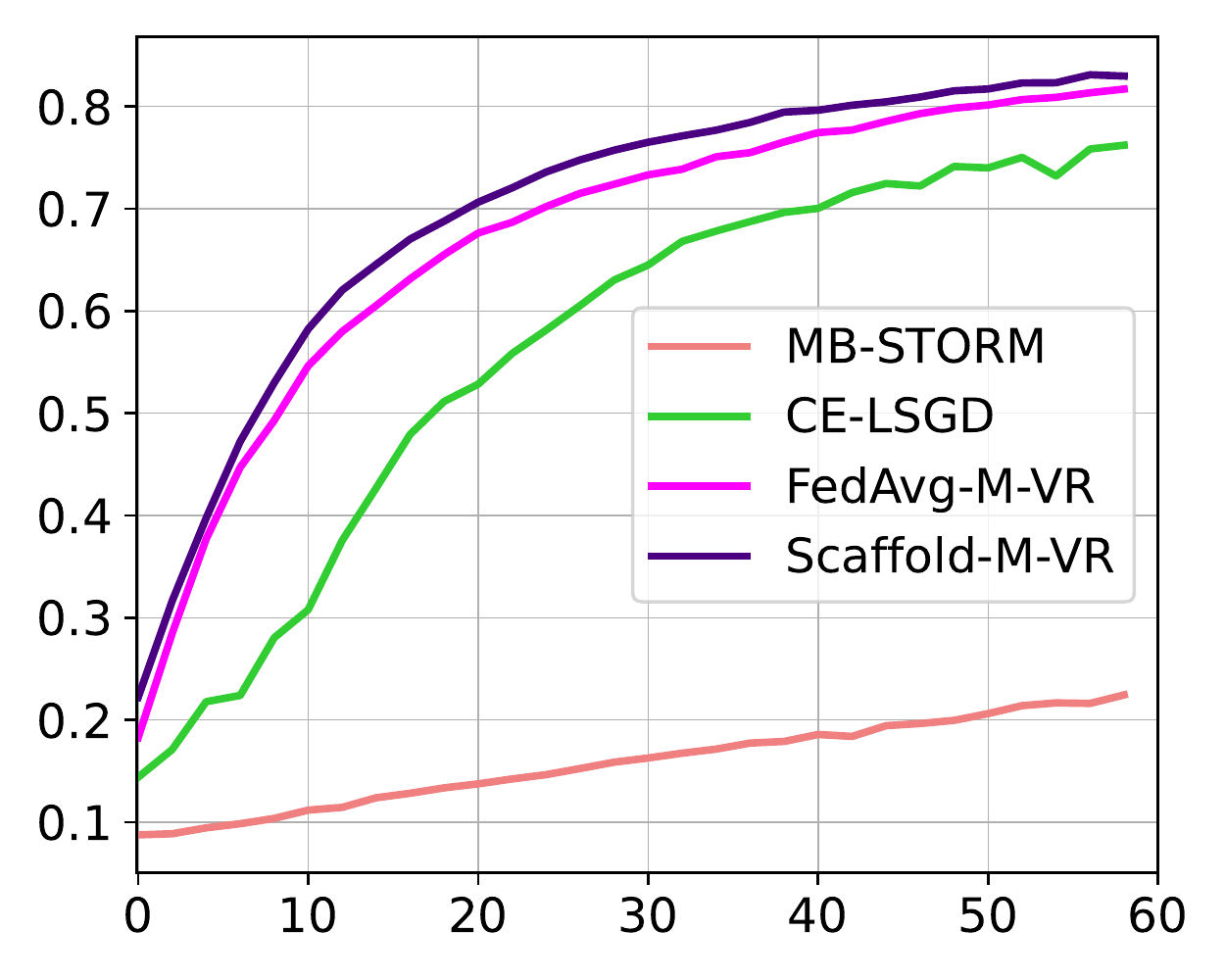}
    \caption{Comparing the test accuracy of VR methods with ResNet-18}
    \label{fig:vr_full_low}
    \hspace{-3mm}
\end{figure}


{\subsection{Experiments with more clients}}
\begin{figure}[h]
    \centering
    \hspace{-3mm}
    \subfigure[Full participation]{\includegraphics[clip,width=0.34\textwidth]{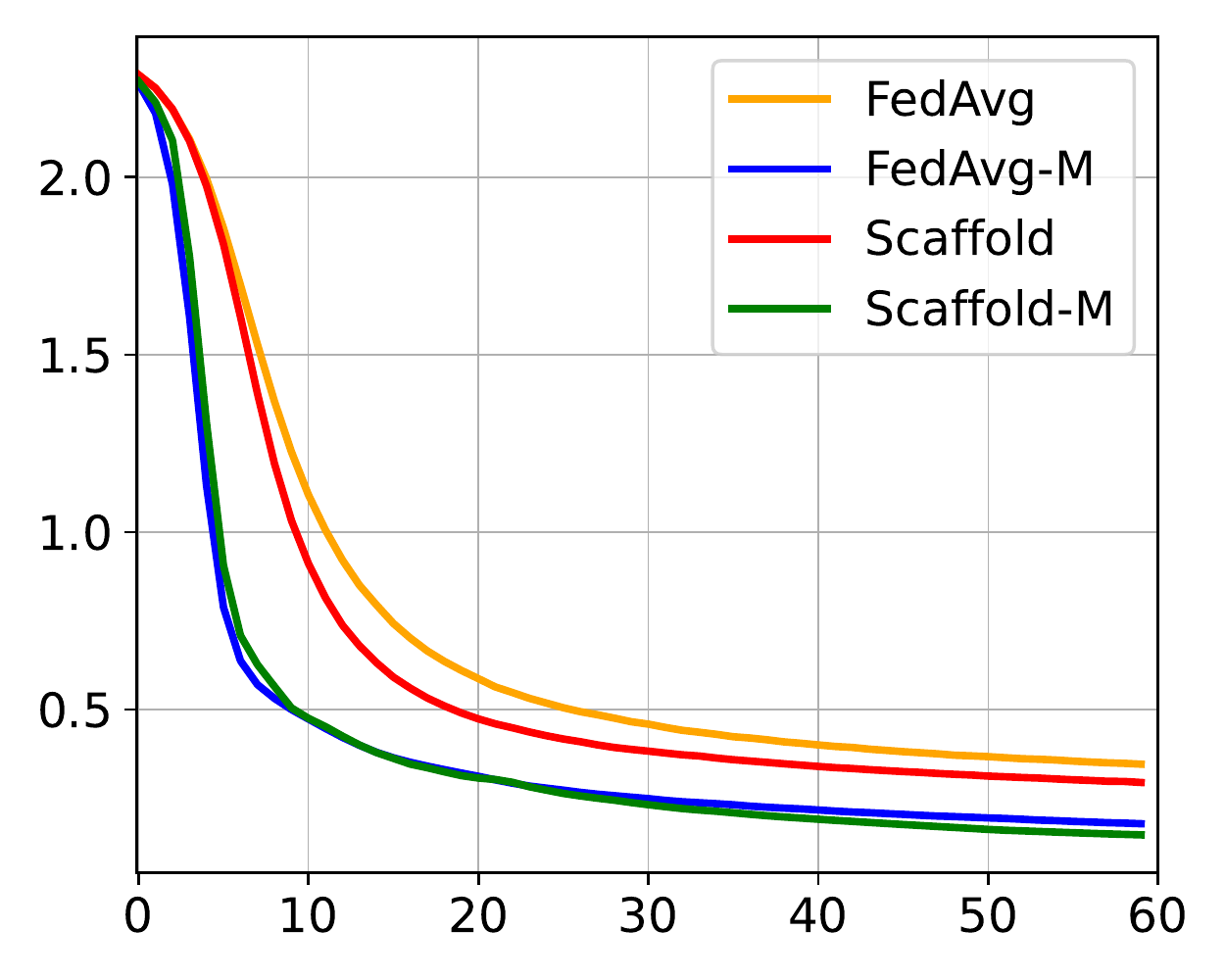} \label{fig:large_full_loss}} \hspace{-3mm}
    \subfigure[Partial participation, $S=10$]{\includegraphics[clip,width=0.34\textwidth]{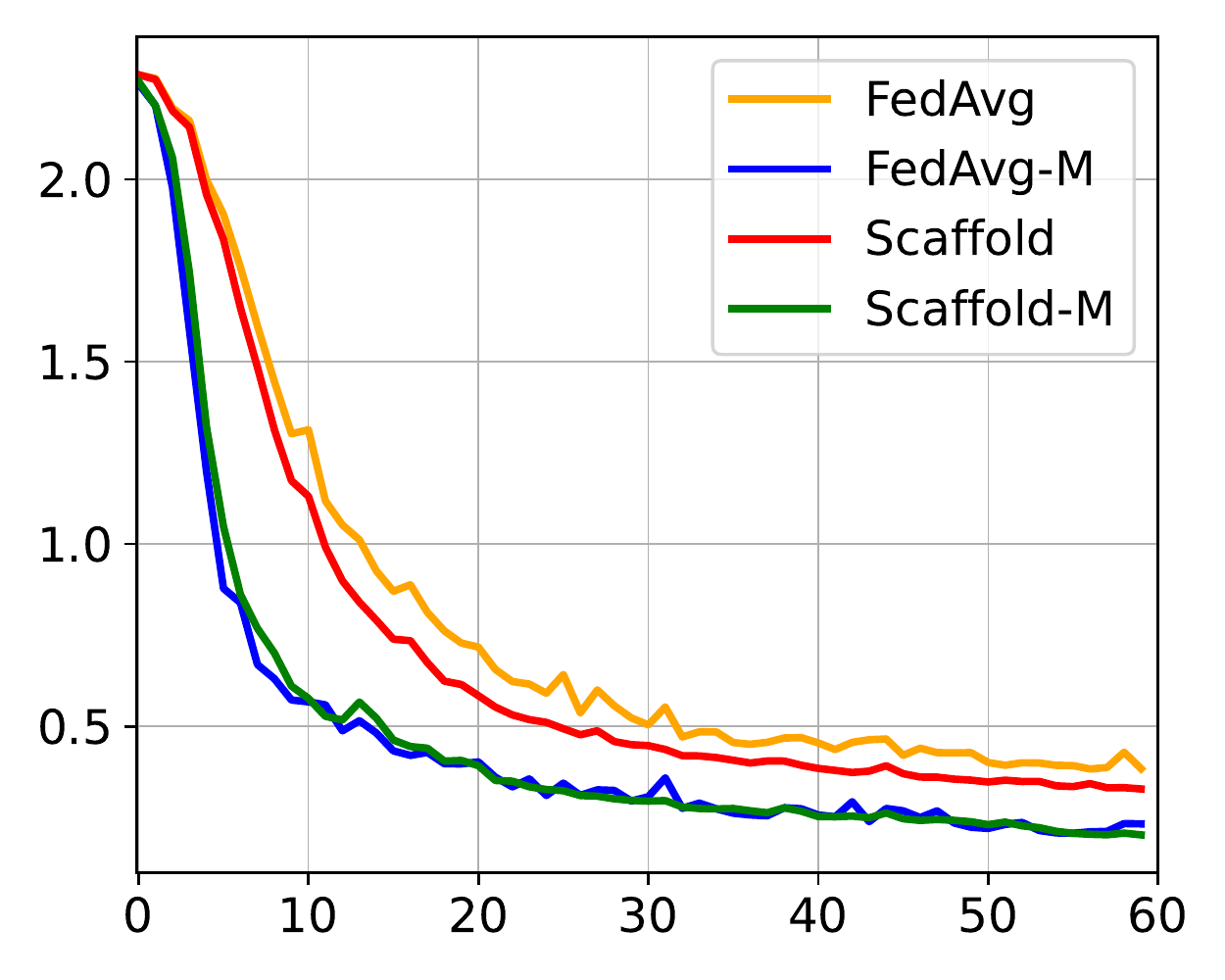} \label{fig:large_par_10_loss}}
    \hspace{-3mm}
    \subfigure[Partial participation, $S=5$]{\includegraphics[clip,width=0.34\textwidth]{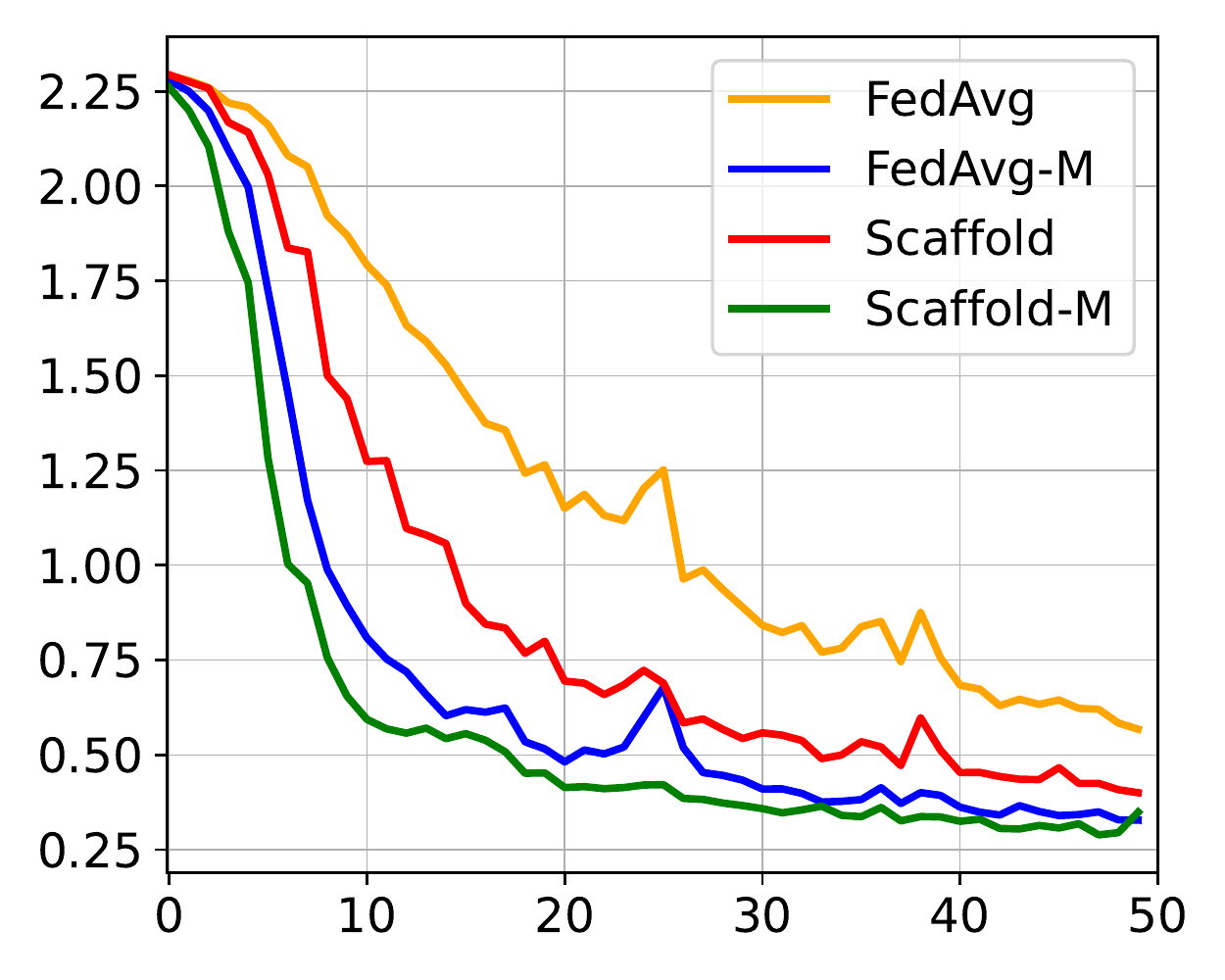} \label{fig:large_par_5_loss}}
    \hspace{-3mm}
    \caption{Test loss of MNIST versus the number of communication rounds}
    \label{fig:large_loss}
\end{figure}

\begin{figure}[h]
    \centering
    \hspace{-3mm}
    \subfigure[Full participation]{\includegraphics[clip,width=0.34\textwidth]{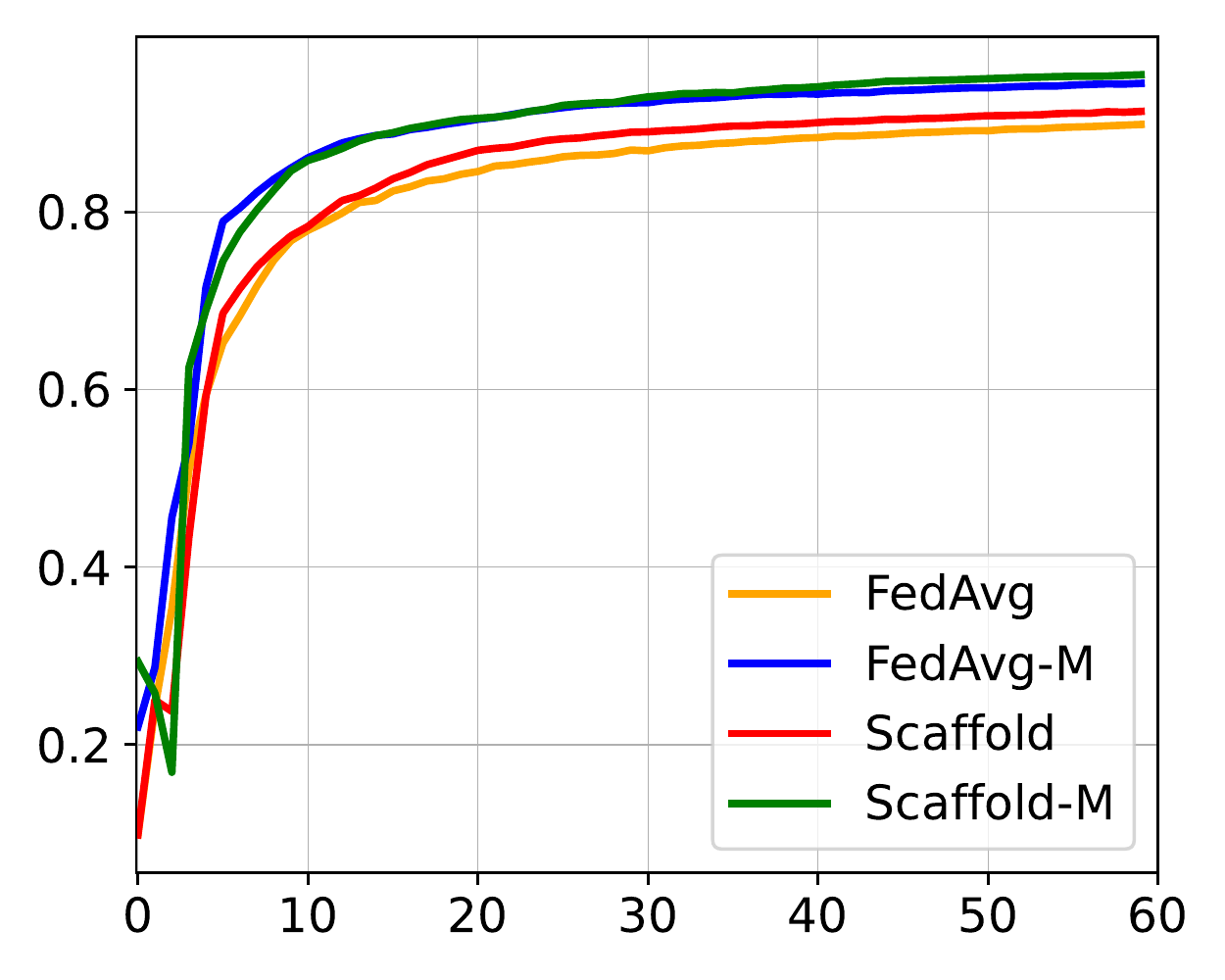} \label{fig:large_full_acc}} \hspace{-3mm}
    \subfigure[Partial participation, $S=10$]{\includegraphics[clip,width=0.34\textwidth]{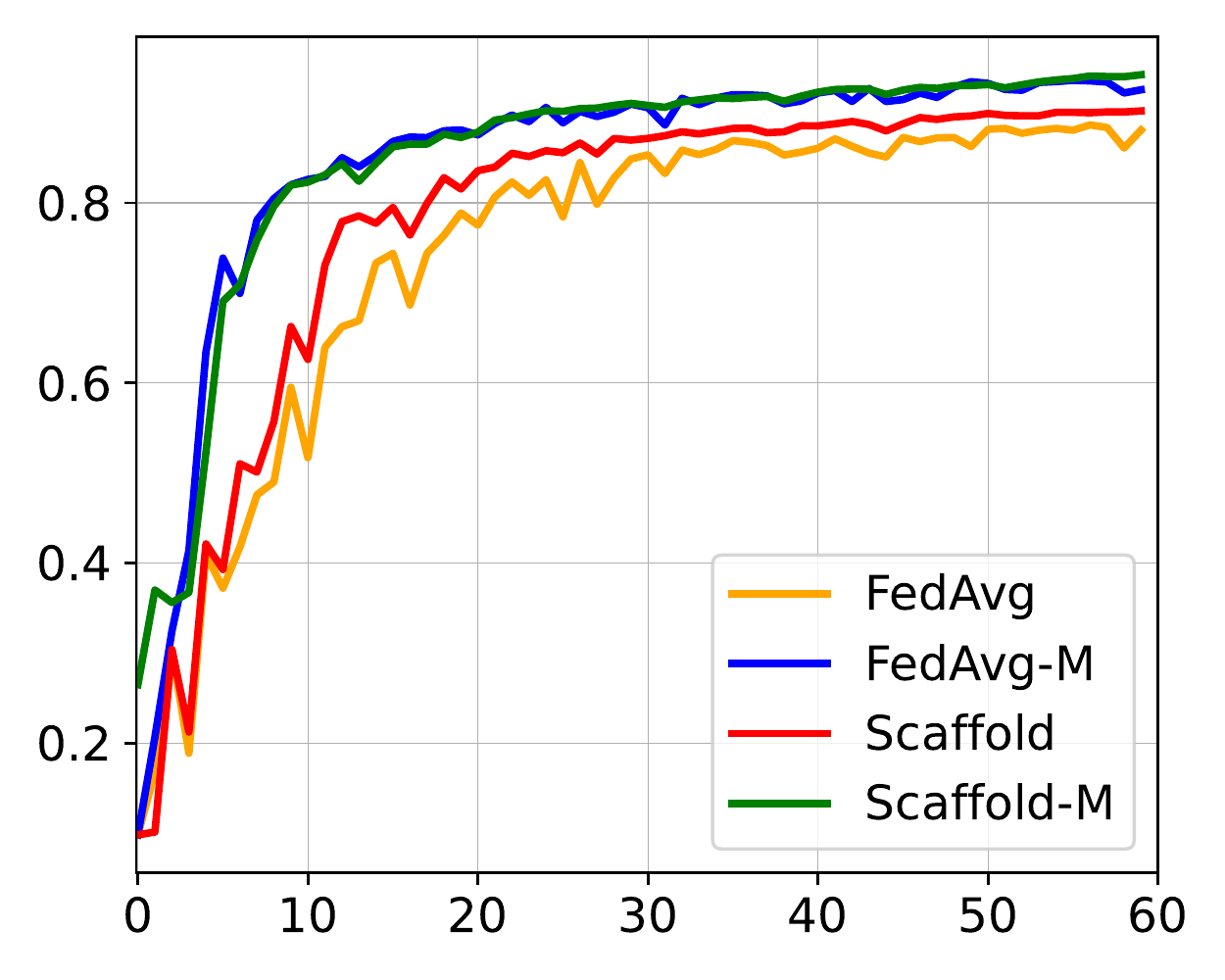} \label{fig:large_par_10_acc}}
    \hspace{-3mm}
    \subfigure[Partial participation, $S=5$]{\includegraphics[clip,width=0.34\textwidth]{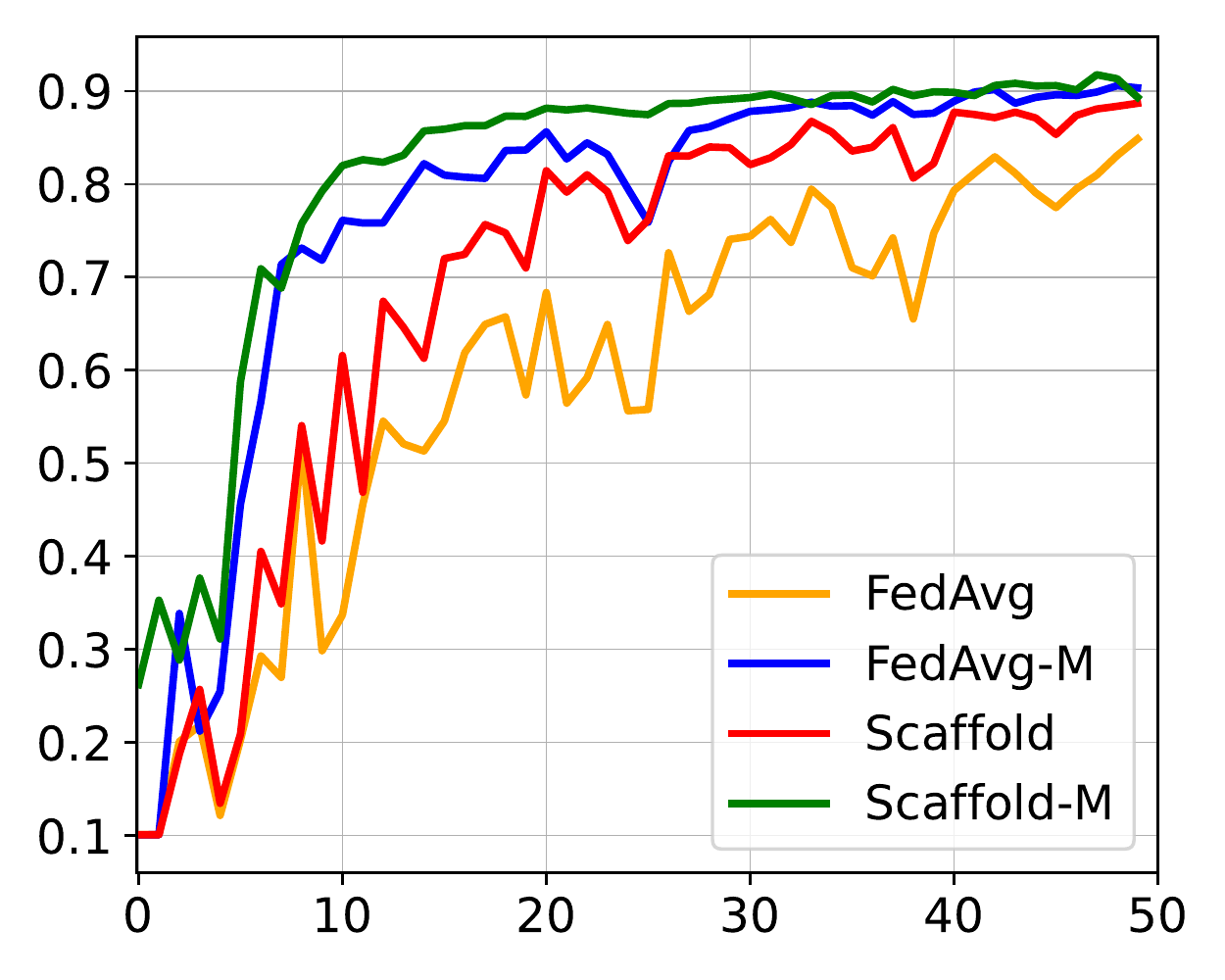} \label{fig:large_par_5_acc}}
    \hspace{-3mm}
    \caption{Test accuracy of MNIST versus the number of communication rounds}
    \label{fig:large_acc}
\end{figure}

{We further conduct experiments with $N=100$ on the MNIST dataset and two-layer fully connected ReLU neural network. 
The parameter of Dirichlet distribution is $0.2$ and the batchsize is $32$.
We set the number of local steps $K=16$. For \fedavgm and \scaffoldm, we set $\beta=0.2$. We plot the test loss and test accuracy of our proposed algorithms in the regime of full participation ($S=N$) and partial participation ($S=10$ and $S=5$).
The results are shown in Figure \ref{fig:large_loss} and \ref{fig:large_acc}.
Compared to former experiments where $N$, we observe that our proposed momentum-based algorithms scale well to FL setups with large $N$. Moreover, we observe that the advantage of our momentum-based variants over the vanilla \fedavg and \scaffold becomes more evident when fewer clients participate in training, suggesting a great utility of our algorithms in practical FL setups.}

{\subsection{Impact of momentum value $\beta$}
To further illustrate the effect of momentum, we examine different choices of $\beta$ in both \fedavgm and \scaffoldm under partial participation setting with $S=5$ and $N=100$. We again simulate with the MNIST dataset and two-layer fully connected ReLU neural networks. The results are shown in Figure \ref{fig:large_mom_loss} and \ref{fig:large_mom_acc}.
It is worth noting that when $\beta\to 1$, the momentum will anneal down to off, recovering the vanilla \fedavg and \scaffold. 
We observe that the stronger the momentum used, the better performance we eventually obtain. This directly demonstrates the benefit of momentum.}

\begin{figure}[ht]
    \centering
    \hspace{-3mm}
    \subfigure[\fedavgm]{\includegraphics[clip,width=0.45\textwidth]{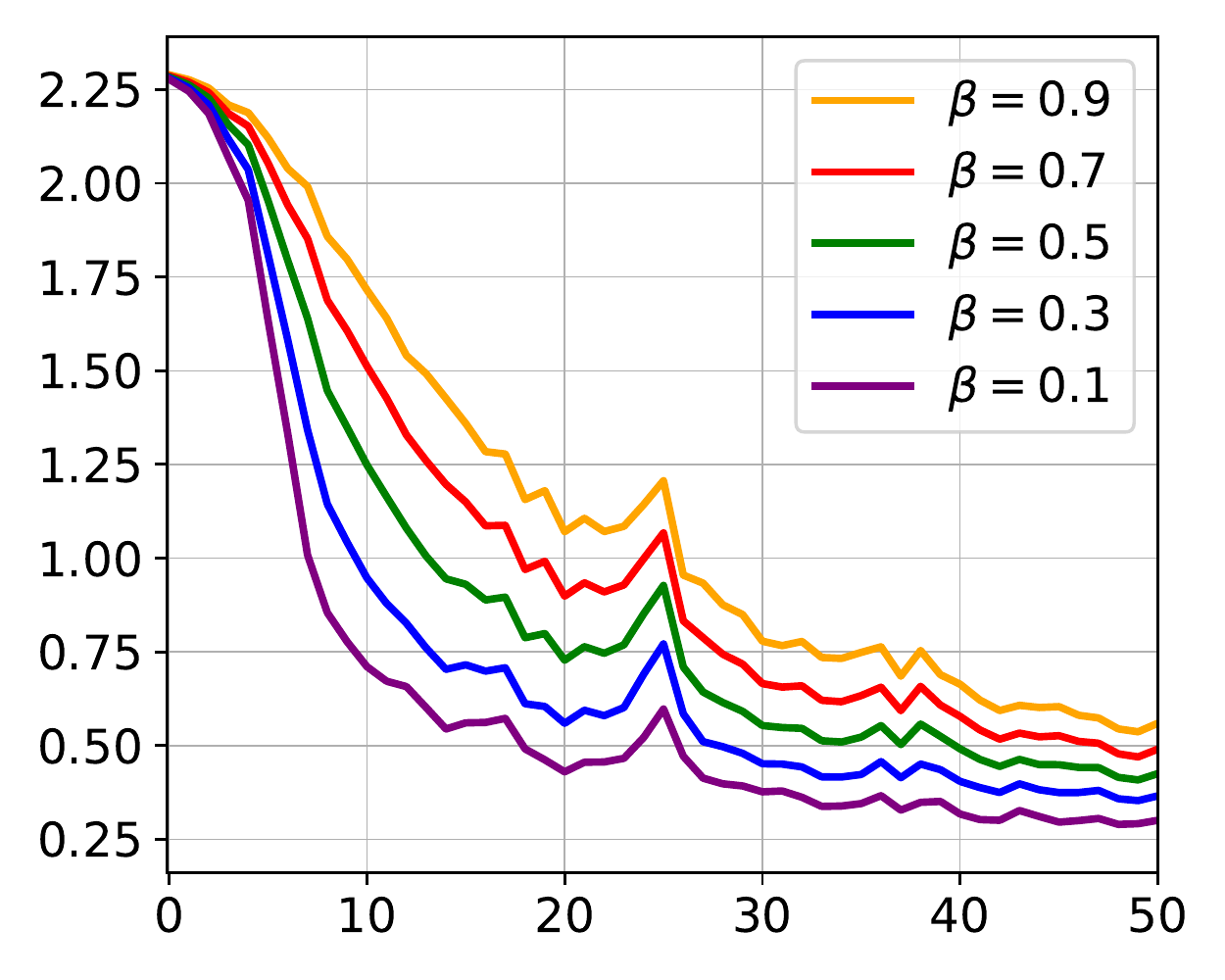} \label{fig:large_fedavg_mom_loss}} \hspace{-3mm}
    \subfigure[\scaffoldm]{\includegraphics[clip,width=0.45\textwidth]{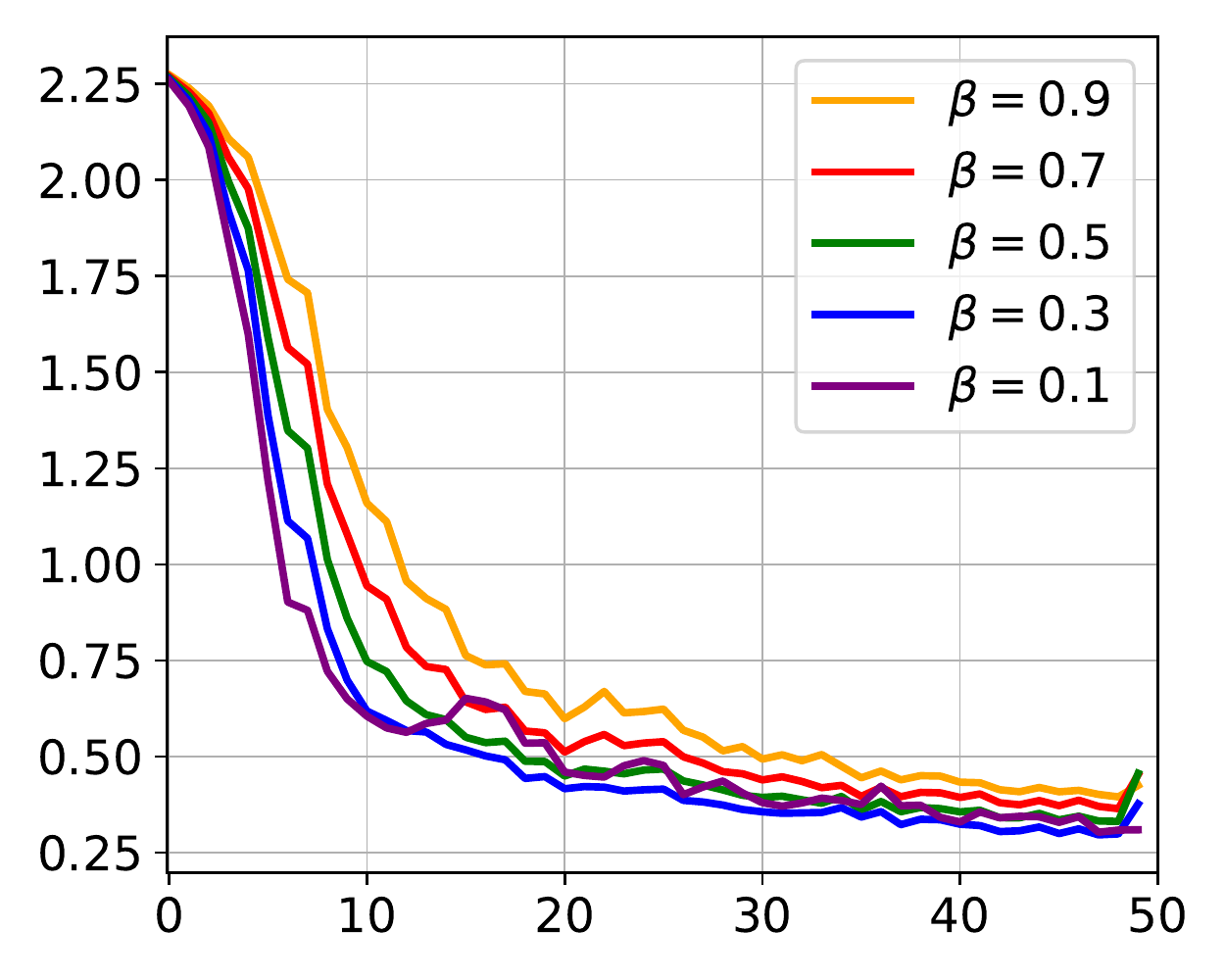} \label{fig:large_scaffold_mom_loss}}
    \hspace{-3mm}
    \caption{Test loss versus communication rounds with different momentum values}
    \label{fig:large_mom_loss}
\end{figure}

\begin{figure}[ht]
    \centering
    \hspace{-3mm}
    \subfigure[\fedavgm]{\includegraphics[clip,width=0.45\textwidth]{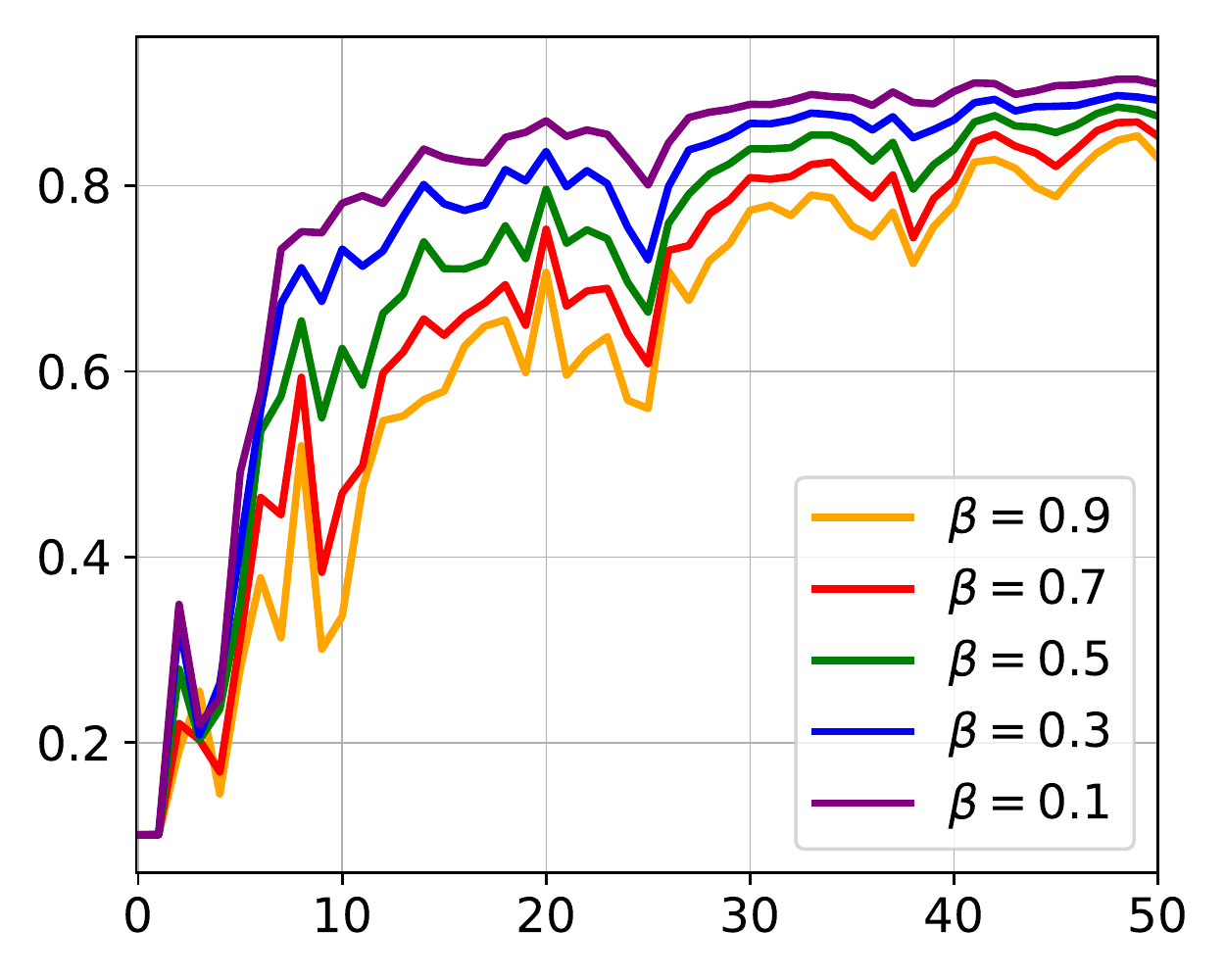} \label{fig:large_fedavg_mom_acc}} \hspace{-3mm}
    \subfigure[\scaffoldm]{\includegraphics[clip,width=0.45\textwidth]{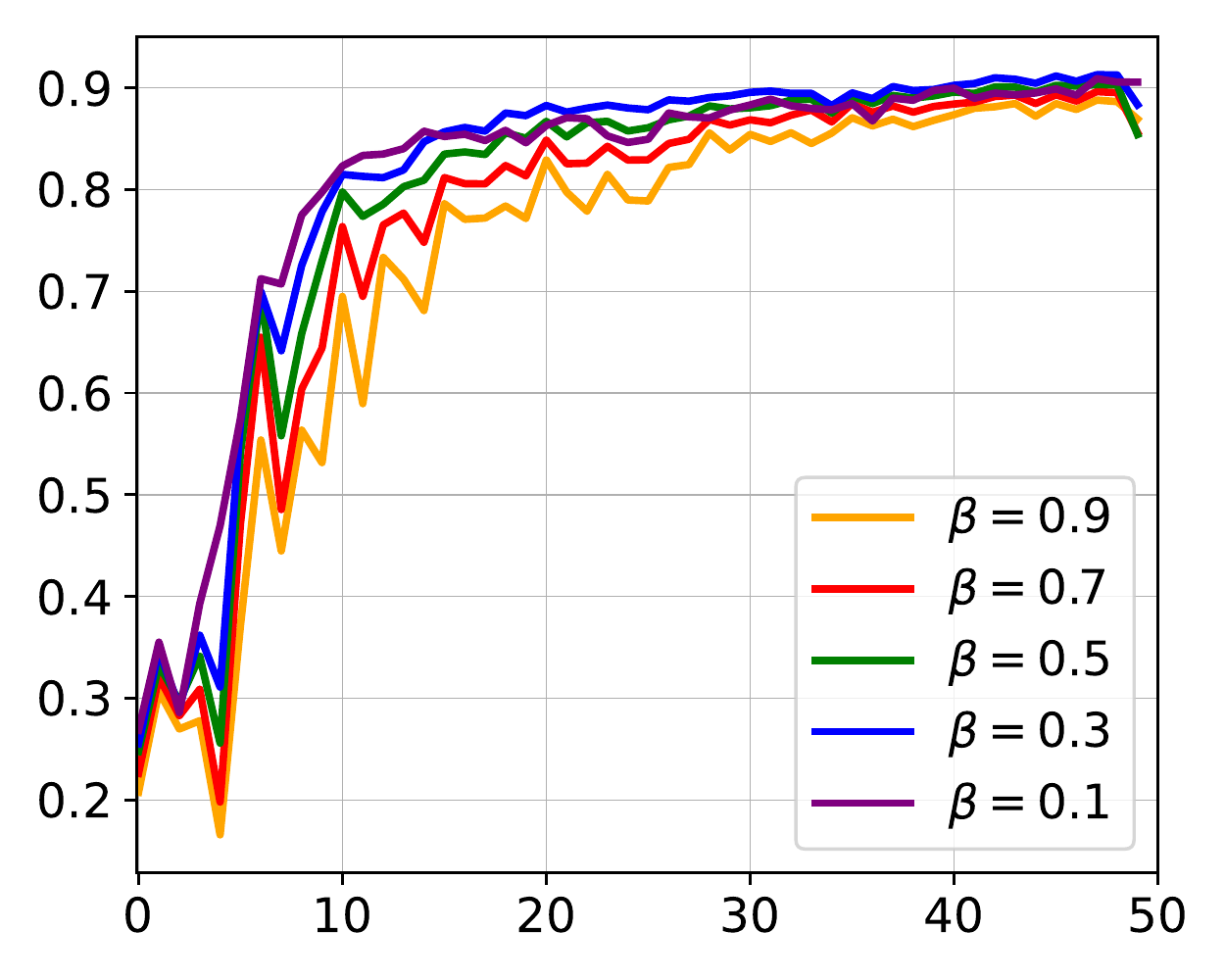} \label{fig:large_scaffold_mom_acc}}
    \hspace{-3mm}
    \caption{Test accuracy versus communication rounds with different momentum values}
    \label{fig:large_mom_acc}
\end{figure}

\end{document}